\documentclass[twoside,11pt]{article}

\usepackage{amsthm,amsmath} 
\usepackage[colorlinks=false,allbordercolors={1 1 1}]{hyperref}
\usepackage{cleveref}

\usepackage[nohyperref]{jmlr2e}

\usepackage[utf8]{inputenc} 
\usepackage[T1]{fontenc}    
\usepackage{hyperref}       
\usepackage{url}            
\usepackage{booktabs}       
\usepackage{amsfonts}       
\usepackage{nicefrac}       
\usepackage{microtype}      
\usepackage{placeins}       
\usepackage{floatpag}       

\usepackage{changes}        
\usepackage[normalem]{ulem}	
\usepackage{cancel}	        

\newcommand{\es}[1]{{\color{blue}#1}}
\newcommand{\ced}[1]{{\color{cyan}#1}}
\newcommand{\journal}{} 

\usepackage{amsmath,graphicx,amssymb, algorithm, bm,epstopdf,subcaption,multirow}
\usepackage{color} 
\usepackage{dsfont} 
\usepackage{tikz-cd}
\usepackage{xfrac} 
\usepackage{array}
\usepackage{hyperref}
\usepackage{cleveref}
\usepackage{pbox} 
\usepackage{minitoc} 
\usepackage{xfp} 

\usepackage{algpseudocode}

\usepackage{enumitem} 

\usepackage{comment} 
\usepackage{appendix}
\usepackage{rotating} 

\newcommand{\mt}[1]{\mathbf{#1}} 
\newcommand{\vt}[1]{\bm{\mathrm{#1}}} 

\definecolor{comment}{gray}{0.5}

\newcommand{\cfd}[1]{{\color{orange} #1}} 
\newcommand{\red}[1]{{\color{red} #1}} 

\ifdefined \journal
\else
    \newtheorem{theorem}{Theorem}
    \newtheorem{corollary}{Corollary}
    \newtheorem{lemma}{Lemma}
    \newtheorem{definition}{Definition}
    
    \newtheorem{remark}{Remark}
    \newtheorem{proposition}{Proposition}

\fi

\newcommand{\y}{\vt{y}} 
\newcommand{\x}{\vt{x}} 
\newcommand{\A}{\mt{A}} 
\newcommand{\cA}{\vt{a}} 
\newcommand{\Axi}{[\A\x]_i} 
\newcommand{\yi}{y_i} 
\newcommand{\xj}{x_j} 
\newcommand{\vtheta}{\vt{\theta}} 
\newcommand{\vthetai}{\theta_i} 
\newcommand{\ATthetaj}{[\A^\T\vtheta]_j} 

\newcommand{\Id}{\mt{I}} 

\newcommand{\R}{\mathbb{R}}

\newcommand{\s}[1]{\mathcal{#1}}
\newcommand{\T}{{\sf T}}        
\newcommand{\e}{{\rm e}}        
\newcommand{\h}{{\rm h}}        
\newcommand{\SO}{{\s{S}_0}}        
\newcommand{\dscale}{\vt{\Xi}}

\newcommand{\ind}{\mathds{1}}   

\renewcommand{\t}{{\sf T}} 

\newcommand{\diag}[1]{\mathrm{Diag}\left(#1\right)}   
\makeatletter

\DeclareMathOperator{\rank}{rank}  

\DeclareMathOperator{\sign}{sign}
\DeclareMathOperator{\dom}{dom}

\DeclareMathOperator{\supp}{supp} 
 
\DeclareMathOperator{\logit}{logit}
\DeclareMathOperator{\unc}{unc}

\DeclareMathOperator{\ST}{ST}

\DeclareMathOperator*{\interior}{int}

\DeclareMathOperator*{\argmin}{argmin} 
\DeclareMathOperator*{\argmax}{argmax} 

\DeclareMathOperator{\KL}{KL}
\DeclareMathOperator{\Gap}{Gap}
\DeclareMathOperator{\gap}{gap}

\newcommand*{\mydots}{.\kern-0.075em.\kern-0.075em.} 
\def\myvdots{\vbox{\baselineskip=3pt \lineskiplimit=0pt 
\kern6pt \hbox{.}\hbox{.}\hbox{.}}}

\usepackage{lastpage}
\jmlrheading{1}{2021}{1-\pageref{LastPage}}{02/21}{00/00}{dantas21}{Cassio F.~Dantas, Emmanuel Soubies and C\'{e}dric F\'{e}votte}

\ShortHeadings{Expanding Boundaries of Gap Safe Screening}{Dantas, Soubies and F\'{e}votte}
\firstpageno{1}

\title{Expanding Boundaries of Gap Safe Screening}

\author{%
  \name Cassio F.~Dantas \email cassio.fraga-dantas@irit.fr\\
  \name Emmanuel Soubies \email emmanuel.soubies@irit.fr \\
  \name C\'{e}dric F\'{e}votte \email cedric.fevotte@irit.fr \\
  \addr IRIT, Universit\'{e} de Toulouse, CNRS, Toulouse, France}

\editor{Silvia Villa}

\begin{document}
\sloppy
\maketitle

%
\doparttoc 
\faketableofcontents 

\begin{abstract}%
    Sparse optimization problems are ubiquitous in many fields such as statistics, signal/image processing and machine learning. This has led to the birth of many iterative algorithms to solve them. A powerful strategy to boost the performance of these algorithms is known as \textit{safe screening}: it allows the early identification of zero coordinates in the solution, which can then be eliminated to reduce the problem's size and accelerate convergence.
    %
	In this work, we extend the existing Gap Safe screening framework 
	by relaxing the global strong-concavity assumption on the dual cost function. 
	Instead, we exploit local regularity properties, 
	that is, strong concavity on well-chosen subsets of the  domain.
	The non-negativity constraint is also integrated to the existing framework. 
	%
	Besides making safe screening possible to a broader	class of functions 
	that includes $\beta$-divergences (e.g., the Kullback-Leibler divergence), the proposed approach also improves upon the existing Gap Safe screening rules on previously applicable cases (e.g., logistic regression).
	%
	The proposed general framework is exemplified by some notable particular cases: 
	logistic function, $\beta=1.5$ and  Kullback-Leibler divergences.
	Finally, we showcase the effectiveness of the proposed screening rules with different solvers (coordinate descent, multiplicative-update and proximal gradient algorithms) and different data sets (binary classification,  hyperspectral and count data). 
\end{abstract}


\section{Introduction}








Safe screening rules have proved to be very powerful tools in order to accelerate the resolution of large-scale sparse optimization problems that arise in statistics, machine learning, signal/image inverse problems, pattern recognition, among other fields. The very principle of safe screening is to identify the zero coordinates in the solution before and/or within the course of iterations of any solver. Once identified, these inactive coordinates can be screened out, thus reducing the size of the problem and consequently the computational load of the solver. Hence, should a screening rule allow to screen many coordinates with a low computational overhead, significant speedups can be observed in practice-- see for instance \citet{ElGhaoui2012,Bonnefoy2015,Ndiaye2017}, and Section~\ref{sec:experiments}.

\paragraph{A brief tour of existing screening strategies.} 
Safe screening rules were initially proposed for the Lasso problem \citep{ElGhaoui2012} 
and were later extended to some of its variants: 
group Lasso \citep{Wang-Wonka2015,Bonnefoy2015}, 
sparse group Lasso \citep{Ndiaye2016,Wang2019}, 
non-negative Lasso \citep{Wang2019}, 
fused Lasso \citep{Wang-Ye2015},  
and generalized Lasso \citep{Ren2018}.
%
As opposed to correlation-based feature selection techniques \citep{Fan2008,Tibshirani2011}, safe screening
strategies are guaranteed to remove only coordinates that do no belong to the solution support. 
%
Beyond Lasso, safe screening has also been used for
other machine learning problems, such as:
binary logistic regression \citep{ElGhaoui2012,Ndiaye2017}, 
metric learning \citep{Yoshida2018},
nuclear norm minimization \citep{Zhou2015} and
support vector machine \citep{Ogawa2013,Wang-Wonka2014,Zimmert2015}. 

Three main classes of screening rules can be distinguished:
1) Static rules \citep{ElGhaoui2012,Xiang2011,Xiang2012} perform variable elimination once and for all prior to the optimization process;
2) Dynamic rules \citep{Bonnefoy2015,Ndiaye2017} perform screening repeatedly over the iterations of an iterative solver, 
leveraging the improvement of the solution estimate to screen-out more coordinates;
3) Sequential rules \citep{Xiang2011,Wang-Wonka2015,Liu2014,Malti2016}
exploit information from previously-solved problems in a regularization path approach. 
%
See, for instance, \citet{Xiang2016} or \citet{Ndiaye2018} for a survey of the domain. 

Most of the mentioned screening techniques are problem-specific, as they exploit particular properties of the targeted loss function. 
%
For instance, rules in \citet{ElGhaoui2012, Bonnefoy2015,Wang-Wonka2015}
assume the dual problem to be a projection problem, 
which is no longer the case for non-quadratic loss functions.
The Gap Safe rule \citep{Ndiaye2017}, however, relies primarily on the duality gap which is defined for any primal-dual pair of problems, regardless of the specific cost functions.
The authors were therefore able to deploy this screening rule for a fairly generic class of functions. 
Additionally, this particular rule leads to state-of-the-art performances in a wide range of scenarios \citep{Ndiaye2017}.
%

\paragraph{Problem definition and working assumptions.} In this work, we consider the following generic primal problem 
\begin{equation} \label{eq:GLM_reg_problem}
\x^\star \in  
\argmin_{\x \in \s{C}}  P_\lambda(\x) := F(\A\x) + \lambda \Omega(\x)
\end{equation}
where $\A \in \R^{m \times n}$,  $F: \R^m \to \R$, $\Omega: \R^n \to \R_{+}$, $\s{C} \subseteq \R^n$, and $\lambda>0$. Moreover, we make the following assumptions:
\begin{itemize}
    \item $F$ is coordinate-wise separable, i.e.,  $ F(\vt{z}) = \sum_{i=1}^m f_{i}(z_i)$ where each scalar function  $f_i : \R \to \R$ is proper, lower semi-continuous, convex, and differentiable.
    \item $\Omega$ is a group-decomposable norm, i.e., given a partition $\s{G}$ of $\{1,\ldots,n\}$, $\Omega(\x) = \sum_{g\in \s{G}} \Omega_g (\x_g)$ where each $\Omega_g$ is a norm on $\R^{n_g}$ ($n_g$ denoting the cardinality of $g\in \s{G}$).\footnote{The $\ell_1$-norm is a trivial example of group-decomposable norm, where each group $g$ corresponds to a singleton (i.e., $\s{G} = \{\{1\}, ~ \dots\, \{n\} \}$) and $\Omega_g = |\cdot |$.}
    \item $\s{C}$ is a constraint set. Here, we study the cases $\s{C} = \R^{n}$ (unconstrained) and $\s{C} = \R_{+}^{n}$ (non-negativity constraint).
\end{itemize}
Finally, we assume  that $ P_\lambda $ admits at least one minimizer $\x^\star \in \s{C}$.

\paragraph{Contributions and roadmap.} The present paper extends the Gap Safe rules proposed by~\citet{Ndiaye2017} (and recalled in Section~\ref{sec:screening_GLM}) to a broader class of problems of the form~\eqref{eq:GLM_reg_problem} in two aspects. First, we allow the use of a non-negativity constraint ($\s{C} = \R^n_+$). Second,  we relax the requirement of global strong concavity of the dual objective function (see Section~\ref{sec:local_approach}). Indeed, we prove in Theorem~\ref{prop:local_concavity} that a Gap Safe sphere can be constructed from the only requirement that the dual objective function is locally strongly concave on a subset that contains the dual solution.
This result is exploited in Section~\ref{ssec:local_fixed} to revisit the Gap Safe dynamic screening algorithm~\citep{Ndiaye2017}. It allows to tackle problems such as common $\ell_1$-regularized Kullback-Leibler regression. In Section~\ref{ssec:local_adaptive}, we further exploit Theorem~\ref{prop:local_concavity} to propose a new Gap Safe dynamic screening algorithm where, at each iteration, the Gap Safe sphere is iteratively refined, leading to an increase of the number of screened variables.
Finally, these two generic approaches are  applied to a set of concrete problems in Section~\ref{sec:Particular_cases}, and are experimentally evaluated in Section~\ref{sec:experiments}.

\section{Safe Screening for Generalized Linear Models} \label{sec:screening_GLM} 

\subsection{Notations and Definitions} 

Scalar operations (such as division, logarithm, exponential and comparisons), whenever applied to vectors, are implicitly assumed as entry-wise operations.
%
%
%
We denote by $[n] = \{1, \dots, n\}$ the set of integers ranging from $1$ to $n \in \mathbb{N}$.
For a vector $\vt{z} \in \R^n$, 
%
 $z_i$ (or sometimes $[\vt{z}]_i$ to avoid ambiguities) stands for its $i$-th entry.
We use the notation $[\vt{z}]^{+}$ to refer to the positive part operation defined as $\max(0, z_i)$ for all $i \in [n]$.
Given a subset of indices $g \subseteq [n]$ with cardinality $|g|= n_g$,  $\vt{z}_g \in \R^{n_g}$ (in bold case) denotes the restriction of $\vt{z}$ to its entries indexed by the elements of $g$. 
For a matrix $\A$, we denote by $\cA_j$ its $j$-th column and $\A_g$ the matrix formed out of the columns of $\A$ indexed by the set $g \subseteq [n]$. 
%
%
We denote $\s{I}^\complement$ the complement of a set $\s{I}$.


We consider functions taking values over the extended real line where the \emph{domain} of a function $f: \R^n \to (-\infty, +\infty]$ is defined as the set:
$$\dom(f) = \{\vt{x} \in \R^n ~|~ f(\vt{x}) < +\infty\}.$$

%

For a function $f : \R^n \to (-\infty, +\infty]$, 
we denote $f^* : \R^n \to (-\infty, +\infty]$ its \emph{Fenchel-Legendre transform} (or conjugate function), defined as follows:
\begin{equation} \label{eq:def:Fenchel_conjugate}
f^*(\vt{u}) := \sup_{\vt{z} \in \R^n} \langle \vt{z}, \vt{u} \rangle - f(\vt{z}).  
\end{equation}
For a norm $\Omega$ over $\R^n$, we denote $\overline{\Omega}$ its associated dual norm such that:
\begin{equation} \label{eq:def:dual_norm}
    \overline{\Omega}(\vt{u}) := \sup_{\Omega(\vt{z})\leq 1} \langle \vt{z}, \vt{u} \rangle.
\end{equation}
The dual norm $\overline{\Omega}$ is not to be confused with the Fenchel conjugate $\Omega^*$. Actually, the conjugate of a norm is the indicator function of the unit ball of the dual norm, i.e., $\Omega^*(\vt{u}) = \ind_{\overline{\Omega}(\vt{u}) \leq 1}$.

\subsection{Dual Problem}

In Theorem~\ref{thm:dual_GLM_and_optimality} below, we derive the dual problem of~\eqref{eq:GLM_reg_problem} together with the associated primal-dual optimality conditions. In particular, we provide a generic expression for the two considered constraint sets $\s{C}=\R^n$ and $\s{C}=\R^n_+$ which extends~\citet[Theorem 2]{Ndiaye2017}. The proof is given in Appendix~\ref{sec:finding_dual}.

\begin{theorem}
\label{thm:dual_GLM_and_optimality}
The dual formulation of the optimization problem defined in \eqref{eq:GLM_reg_problem} is given by
\begin{align} \label{eq:GLM_reg_dual_problem}
& \vtheta^\star =
\argmax_{\vtheta \in \Delta_{\A}} ~  D_\lambda(\vtheta) := - \sum_{i=1}^{m} f_i^*(-\lambda \vthetai)\\
%
&\text{with } ~  \Delta_{\A} = \{\vtheta \in \R^m ~|~ \forall g \in \s{G}, \overline{\Omega}_g(\phi(\A_{g}^\T \vtheta)) \leq 1\} \cap \dom(D_\lambda) \label{eq:GLM_dual_feasible_set}  
\end{align}
where $\vtheta \in \R^m$ is the dual variable. 
The function $\phi: \R \rightarrow \R$,
applied component-wisely in~\eqref{eq:GLM_dual_feasible_set},
is either the identity (when $\s{C} = \R^n$) or the positive part (when $\s{C} = \R_{+}^n$)
%
\begin{align}\label{eq:phi_GLM_dual}
\phi  = 
\left\{
\begin{array}{ll}
 \operatorname{Id} & \text{if } ~\s{C} = \R^n  \\
 \left[\cdot \right]^{+} & \text{if } ~\s{C} = \R_{+}^n
\end{array} \right.. 
\end{align}
%
Moreover, the first-order optimality conditions for a  primal-dual solution pair $(\x^\star,\vtheta^\star) \in (\dom (P_\lambda) \cap \s{C} ) \times \Delta_{\A}$, are given by
\begin{align}
&\forall i \in [m],~~ \lambda \vthetai^\star = - f_i'([\A\x^\star]_i)    & \text{(primal-dual link)} \label{eq:GLM_optimality_condition1} \\ 
%
%
&\forall g \in \s{G}, 
\left\{
\begin{array}{lrl}
\overline{\Omega}_g(\phi(\A_g^\T \vtheta^\star)) \leq 1, & \text{if } ~\x_g^\star = \vt{0} &\\
\overline{\Omega}_g(\phi(\A_g^\T \vtheta^\star)) = 1,  ~~ (\A_g^\T \vtheta^\star)^\T\x_g^\star = \Omega_g(\x_g^\star), & \text{otherwise} & 
\end{array} \right. \mkern-25mu
&  \begin{array}{r}
 \text{(sub-differential} \\
\text{inclusion)}
\end{array}\label{eq:GLM_optimality_condition2_explicit}
\end{align}
%
\end{theorem}




\subsection{Safe Screening Rules}

%

A direct consequence of Theorem~\ref{thm:dual_GLM_and_optimality} is that, given the dual solution $\vtheta^\star$,
\begin{equation}\label{eq:screening_implication}
    \overline{\Omega}_g(\phi(\A_g^\T \vtheta^\star)) < 1 \implies \x_g^\star = \vt{0} 
\end{equation}
for any primal solution~$\x^\star$.
Hence, every group of coordinates $g\in \s{G}$  for which $\overline{\Omega}_g(\phi(\A_g^\T \vtheta^\star)) < 1$ is surely inactive (i.e., $\x^\star_g = \vt{0}$). They can thus be safely screened out in order to reduce the size of the primal problem and accelerate its resolution.
In practice, however, the dual solution $\vtheta^\star$ is unknown (in advance) and~\eqref{eq:screening_implication} cannot be evaluated. Fortunately, it is possible to define a more restrictive---yet practical---sufficient condition that relies on the concept of \emph{safe region}.

\begin{definition}[Safe Region] A compact subset $\s{R} \subset \R^m$ is said to be a \emph{safe region} if it contains the dual solution (i.e., $\vtheta^\star \in \s{R}$).
\end{definition}

\begin{proposition}[Safe Screening Rule \citep{ElGhaoui2012}]   \label{propo:Safe_screen}
Let  $\s{R}$ be a safe region and $g \in \s{G}$. Then, 
\begin{align}
\max_{\vtheta \in \s{R}} \; \overline{\Omega}_g(\phi(\A_g^\T \vtheta)) < 1 \implies
\overline{\Omega}_g(\phi(\A_g^\T \vtheta^\star)) < 1 \implies
\x_g^\star = \vt{0}.
\end{align}
\end{proposition}
The inequality $\max_{\vtheta \in \s{R}} \overline{\Omega}_g(\phi(\A_g^\T \vtheta)) < 1$ is referred to as \emph{screening test} as it allows to test whether a group of coordinates is guaranteed to be zero in the optimal solution. 

With Proposition~\ref{propo:Safe_screen}, numerous screening rules can be defined from the construction of different safe regions. Although any region $\s{R}$ such that $\Delta_\A \subset  \s{R}$ is safe as $\vtheta^\star \in \Delta_\A$, these trivial choices would lead to poor screening performance. Instead,  to maximise the number of screened groups while limiting the computational overhead of testing, one needs to construct safe regions $\s{R}$  that are as small as possible, and for which the quantity $\max_{\vtheta \in \s{R}} \overline{\Omega}_g(\phi(\A_g^\T \vtheta))$ (screening test) can be computed efficiently.
It is thus standard practice to consider simple regions such as balls~\citep{ElGhaoui2012, Bonnefoy2015,Ndiaye2017} or domes~\citep{Fercoq2015,Xiang2012} 
as they are more likely to lead to closed-form expressions of the screening test (see Section~\ref{sec:Particular_cases}).




%

\paragraph{Gap Safe Sphere.} \label{ssec:GAP_Safe_sphere} 
%
A notable safe region is the Gap Safe sphere as it leads to state-of-the-art screening performances in a wide range of scenarios \citep{Ndiaye2017}. It relies on the duality gap for the primal-dual problems \eqref{eq:GLM_reg_problem}-\eqref{eq:GLM_reg_dual_problem} defined for a feasible pair $(\x,\vtheta)$ as
\begin{equation}
    \Gap_\lambda(\x,\vtheta) := P_\lambda (\x) - D_\lambda(\vtheta).
\end{equation}

\begin{theorem}[Gap Safe Sphere \citep{Ndiaye2017}]\label{thm:GAP_Safe}
Let the dual function $D_\lambda$ be $\alpha$-strongly concave.
%
Then, for any feasible primal-dual  pair $(\x,\vtheta) ~\in~ (\dom(P_\lambda) \cap \s{C}) \times \Delta_\A$:
\begin{align}
\s{B}(\vtheta,r),  \text{ with } r = \sqrt{\frac{2\Gap_\lambda(\x,\vtheta)}{\alpha}}
\label{eq:GAP_Safe_sphere}
\end{align}
is a safe region, i.e., $\vtheta^\star \in \s{B}(\vtheta, r)$. 
\end{theorem}

The Gap Safe sphere improves over previously proposed safe regions~\citep{ElGhaoui2012, Bonnefoy2015,Wang2015} in two ways. First, it is not restricted to the Lasso problem and applies to a broad class of problems of the form~\eqref{eq:GLM_reg_problem}, under the assumption that the associated dual function $D_\lambda$ is strongly concave. Second, in case of strong duality, its radius vanishes when a converging sequence of  primal-dual variables is provided (with the duality gap tending to zero).


%

\subsection{Screening with Existing Solvers} \label{sec:standard_GAPSafeScreen}

The previously presented screening tools can be integrated to most existing solvers in order to reduce the size of the primal problem~\eqref{eq:GLM_reg_problem} and accelerate its resolution. As previously mentioned, 
screening rules can be exploited in many ways~\citep[see][]{Xiang2016} that include \textit{static screening}~\citep{ElGhaoui2012}, \textit{sequential screening}~\citep{Xiang2011}, or \textit{dynamic screening}~\citep{Bonnefoy2015}. In this work, we focus on the dynamic screening approach that fully exploits the structure of iterative optimization algorithms by screening out groups of coordinates in the course of iterations. As the algorithm converges, smaller safe regions can be defined, leading to an increasing number of screened groups.
%
More precisely, we use the Gap Safe dynamic screening scheme proposed by \citet{Ndiaye2017} as our baseline.
This scheme is presented in Algorithm \ref{alg:solver_screening} where
\begin{equation}
    \{\x,\vt{\eta}\} \gets \mathtt{PrimalUpdate}(\x,\A, \lambda, \vt{\eta})
\end{equation}
represents the update step of any iterative primal solver for~\eqref{eq:GLM_reg_problem}. 
There, 
 $\x$ denotes the primal variable and $\vt{\eta}$ is a vector formed out of the auxiliary variables of the solver (e.g., gradient step-size, previous primal estimates). 
 To keep the presentation concise, screening is performed after every iteration of the primal solver in Algorithm~\ref{alg:solver_screening}. However, it is noteworthy to mention that this is not a requirement. Screening can actually be performed at any chosen moment. For instance, on regular intervals between a certain number of iterations of the solver. Finally, let us emphasise the nested update of the preserved set (line \ref{alg:line:nested_preserved_set}) showing that the screened groups are no longer tested in the ensuing iterations.
 
 To construct a Gap Safe sphere, a dual feasible point $\vtheta  \in \Delta_\A$ is required (Theorem~\ref{thm:GAP_Safe}). 
 Although a dual point may be provided by primal-dual solvers~\citep{Chambolle2011,Yanez2017}, it
 is not always guaranteed to be feasible. Moreover, such a dual point
 needs to be computed from $\x$ when the solver only provides a primal solution estimate at each iteration. Needless to say that the latter is the case of many popular solvers for~\eqref{eq:GLM_reg_problem} such as \citet{Beck2009,Harmany2012,Hsieh2011}.
 At line~\ref{alg:line:dual_update} of Algorithm~\ref{alg:solver_screening}, this computation of a dual feasible point (referred to as dual update)  is defined through the function $\vt{\Theta} : \R^n \rightarrow \Delta_\A$. Always with the aim of maximizing the number of screened groups (i.e., reducing the Gap Safe sphere) while limiting the computational overhead, a rule of thumb for $\vt{\Theta}$ is that, for a primal estimate $\x$, the evaluation of $\vt{\Theta}(\x)$ only requires ``simple'' operations and $\|\vt{\Theta}(\x) - \vtheta^\star \|$ is as small as possible. By exploiting the optimality condition~\eqref{eq:GLM_optimality_condition1}, it is customary to define $\vt{\Theta}$ as a simple rescaling of $-\nabla F(\A\x)$ (see Section~\ref{ssec:dual_update}). Not only this choice is computationally cheap but it enjoys the appealing property that $\vt{\Theta}(\x)  \rightarrow \vtheta^\star$ as $\x \rightarrow \x^\star$.

 The purpose of the next section is to extend and refine \Cref{alg:solver_screening}.


\begin{algorithm}[t]
\caption{Dynamic Gap Safe Screening (DGS)~\citep{Ndiaye2017}: \\
$~\hat{\x} \!=\! \mathtt{GAPSolver}(\A,\lambda, \varepsilon_{\gap})$ 
} \label{alg:solver_screening}
\begin{algorithmic} [1]
\State \textbf{Initialize} $\mathcal{A}= \s{G}$, $\x \in \mathcal{C}$ 
\State \textbf{Set}  $\vt{\eta}$ according to the solver
\State \textbf{Compute}  $\alpha$ a strong concavity bound of $D_\lambda$ on $\R^m$
	\Repeat
		\State \textit{--- Solver update restricted  to preserved set  ---}
		\State $\{\x_{\s{A}},\vt{\eta}\} \gets \mathtt{PrimalUpdate}(\x_{\s{A}},\A_{\s{A}}, \lambda, \vt{\eta})$ 
		\State \textit{--- Dynamic Screening ---}
		\State $\vtheta \gets \vt{\Theta}\left(\x \right) \in \Delta_\A$ \Comment{Dual update}\label{alg:line:dual_update}
		\State $r \gets \sqrt{\frac{2\Gap_\lambda(\x,\vtheta)}{\alpha}}$ \Comment{Safe radius}
		\State $\s{A} \gets \{g \!\in\! \s{A}  ~|~ \max_{\vtheta \in \s{B}(\vtheta,r)} \overline{\Omega}_g(\phi(\A_g^\T \vtheta)) \geq 1 \}$ \Comment{Screening test}\label{alg:line:nested_preserved_set}
		\State $\x_{\s{A}^c} \gets \vt{0}$
		\Until {$\Gap_\lambda(\x,\vtheta) < \varepsilon_{\gap}$} 
\end{algorithmic}
\end{algorithm}

 \section{Exploiting Local Regularity Properties of the Dual Function}
\label{sec:local_approach}

The Gap Safe sphere given in Theorem~\ref{thm:GAP_Safe} requires the dual function to be globally $\alpha$-strongly concave. 
This precludes its application to an important class of problems of practical interest such as problems involving the Kullback-Leibler divergence and other $\beta$-divergences with $\beta \in (1,2)$ (see \Cref{sec:Particular_cases}).
In this section, we relax this hypothesis by leveraging only \emph{local} properties of the dual function. More precisely, we derive in Theorem~\ref{prop:local_concavity} a Gap Safe sphere with the only requirement that $D_\lambda$ is strongly concave on a well-chosen subset of its domain. The proof is provided in \Cref{app:local_GAP}.
%



\begin{theorem}
\label{prop:local_concavity}
Let $D_\lambda$ be $\alpha_\s{S}$-strongly concave on a subset $\s{S} \subset \R^m$ 
such that $\vtheta^\star \in \s{S}$. Then, for any feasible primal-dual pair $(\x,\vtheta) \in (\dom(P_\lambda) \cap \s{C}) \times (\Delta_{\A} \cap \s{S})$:
\begin{equation}
\s{B}(\vtheta,r), \text{ with } r = \sqrt{\frac{2\Gap_\lambda(\x,\vtheta)}{\alpha_\s{S}}}
\label{eq:GAP_Safe_sphere_local}
\end{equation}
is a Gap Safe sphere, i.e., $\vtheta^\star \in \s{B}(\vtheta, r)$. 
\end{theorem}

Theorem~\ref{prop:local_concavity} allows us to extend the application of Gap Safe rules to problems for which the corresponding dual function $D_\lambda$ is not globally strongly concave.  When it comes to the primal objective function, it allows us to tackle some data-fidelity functions which do not have a  Lipschitz-continuous gradient.
%

Moreover, this result can also be used to improve upon the performance of standard Gap Safe rules by providing better local bounds $\alpha_\s{S}$ for the strong concavity of the dual function. 
Indeed, a global strong concavity bound $\alpha$ (if any) cannot ever be larger than its local counterpart $\alpha_S$ for a given valid set $\s{S}$.
%
Hence, not only Theorem~\ref{prop:local_concavity} extends Gap Safe rules to a broader class of problems, but it can boost their performances when the known global strong concavity bound is poor (too small).

In the two following sections, we exploit Theorem~\ref{prop:local_concavity} to revisit (in Section~\ref{ssec:local_fixed}) and improve (in Section~\ref{ssec:local_adaptive}) the Gap Safe screening approach proposed by~\citet{Ndiaye2017}.


\subsection{Generalized Gap Safe Screening}\label{ssec:local_fixed}

A natural choice would be to set $\s{S}=\Delta_\A$ in Theorem~\ref{prop:local_concavity} as it contains all possible feasible dual points, including the dual solution $\vtheta^\star$. This choice is fine when $D_\lambda$ is strongly concave on $\Delta_\A$, and when a strong concavity bound on this set can be derived. However, it may be necessary to further restrict $\Delta_\A$ in order to get the local strong concavity property, or to use a simpler shape for $\s{S}$ in order to derive a strong concavity bound in closed-form. Hence, without loss of generality, we consider hereafter the set $\s{S} = \Delta_\A \cap \SO$, where $\SO \subseteq \R^m$ is such that $\vtheta^\star \in \SO$. A careful choice of $\s{S}_0$ can turn out to be crucial in order to obtain the required local strong concavity property. For instance, this is the case for the $\beta=1.5$ and Kullback-Leibler divergences, as discussed in \Cref{ssec:particular_discussion}.

Given a local strong concavity bound $\alpha_{\Delta_\A \cap \SO}$ of $D_\lambda$ over $\Delta_\A \cap \SO$ (see \Cref{ssec:Strong_Concavity_Bounds}), we  revisit the Gap Safe dynamic screening approach (Algorithm~\ref{alg:solver_screening}) in the way it is presented in Algorithm~\ref{alg:solver_screening_bis}. 
A notable difference is that the dual update (line~\ref{alg:line:dual_update_bis}) requires to output a point $\vtheta$ in $\Delta_\A \cap \SO$, i.e., $\vt{\Theta}: \R^n \rightarrow \Delta_\A \cap \SO$ (instead of just $\Delta_\A$ in Algorithm~\ref{alg:solver_screening}). The reason is that, from Theorem~\ref{prop:local_concavity}, the ball $\s{B}(\vtheta,r)$ with $r$ given at line~\ref{alg:line:safe_radius_bis} is ensured to be safe only if the center $\vtheta$ belongs to the set on which the strong concavity bound has been computed, that is $\Delta_\A \cap \SO$. In contrast, the intersection of $\s{B}(\vtheta,r)$ with $\SO$ in the screening test (line~\ref{alg:line:nested_preserved_set_bis}) is not mandatory 
as $\s{B}(\vtheta,r)$ is itself a safe region. However, taking the intersection may lead to an even smaller set and, consequently, increase the number of screened variables.



\def\b#1{#1} 
\begin{algorithm}[t]
\caption{Generalized Dynamic Gap Safe Screening (G-DGS): \\ 
 $\hat{\x} = \mathtt{GapSolver}(\A,\lambda, \b{\SO}, \varepsilon_{\gap})$ 
} \label{alg:solver_screening_bis}
\begin{algorithmic} [1]
\State \textbf{Initialize} $\mathcal{A}= \s{G}$, $\x\in \mathcal{C}$   
\State \textbf{Set}  $\vt{\eta}$ according to the solver
\State \textbf{Compute}  \b{$\alpha_{\Delta_\A \cap \SO}$} a strong concavity bound of $D_\lambda$ on \b{$\Delta_\A \cap \SO$}
	\Repeat
		\State \textit{--- Solver update restricted  to preserved set  ---}
		\State $\{\x_{\s{A}},\vt{\eta}\} \gets \mathtt{PrimalUpdate}(\x_{\s{A}},\A_{\s{A}}, \lambda, \vt{\eta})$ 
		\State \textit{--- Dynamic Screening ---}
		\State $\vtheta \gets \vt{\Theta}\left(\x \right) \in \Delta_\A \b{\cap \SO}$. \Comment{Dual update}\label{alg:line:dual_update_bis}
		\State $r \gets \sqrt{\frac{2\Gap_\lambda(\x,\vtheta)}{\b{\alpha_{\Delta_\A \cap \SO}}}}$ \Comment{Safe radius}\label{alg:line:safe_radius_bis}
		\State $\s{A} \gets \{g \!\in\! \s{A}  ~|~ \max_{\vtheta \in \s{B}(\vtheta,r) \b{\cap \SO}} \overline{\Omega}_g(\phi(\A_g^\T \vtheta)) \geq 1 \}$ \Comment{Screening test}\label{alg:line:nested_preserved_set_bis}
		\State $\x_{\s{A}^c} \gets \vt{0}$
		\Until {$\Gap_\lambda(\x,\vtheta) < \varepsilon_{\gap}$} 
\end{algorithmic}
\end{algorithm}

\subsection{Gap Safe Screening with Sphere Refinement}\label{ssec:local_adaptive}

We now go one step further by observing that a Gap Safe sphere is a valid subset to invoke Theorem~\ref{prop:local_concavity}.

\begin{corollary}\label{coro:localGAP}
Let $\SO$ be a safe region and $\s{B}(\vtheta,r)$ be a Gap Safe sphere for the primal-dual pair $(\x,\vtheta) \in (\dom(P_\lambda) \cap \s{C}) \times (\Delta_{\A} \cap \SO)$. If $D_\lambda$ is $\alpha_{\s{B}}$-strongly concave on $\s{B}(\vtheta,r) \cap \SO$, then $\s{B}(\vtheta,\sqrt{{2\Gap_\lambda(\x,\vtheta)}/{\alpha_\s{B}}})$  is a  Gap Safe sphere.
\end{corollary}

\begin{proof}
Application of \Cref{prop:local_concavity} with $\s{S} = \s{B}(\vtheta,r) \cap \SO$.
\end{proof}

Corollary~\ref{coro:localGAP} suggests that the radius $r$ computed at line~\ref{alg:line:safe_radius_bis} of Algorithm~\ref{alg:solver_screening_bis} may be further reduced by computing a new strong concavity bound over $\s{B}(\vtheta,r) \cap \s{S}_0$. We thus propose to replace the radius update at line~\ref{alg:line:safe_radius_bis} by an iterative refinement procedure, as depicted in Algorithm~\ref{alg:solver_screening_local} (lines~\ref{alg:line:inner_loop_begin}--\ref{alg:line:inner_loop_end}). The convergence of this refinement loop is guaranteed as stated in Proposition~\ref{propo:CV_refinement}.

\begin{algorithm}[t]
\caption{Refined Dynamic Gap Safe Screening (R-DGS) : \\
$~\hat{\x} = \mathtt{GAPSolver}(\A,\lambda, \SO, \varepsilon_{\gap},\varepsilon_r)$ 
} \label{alg:solver_screening_local}
\begin{algorithmic} [1]
\State \textbf{Initialize} $\mathcal{A}= \s{G}$, $\x\in \mathcal{C}$   
\State \textbf{Set}  $\vt{\eta}$ according to the solver
\State \textbf{Compute} $\alpha_{\Delta_\A \cap \SO}$ a strong concavity bound of $D_\lambda$ on $\Delta_\A \cap \SO$
\State \textit{--- Construction of an initial safe sphere $\s{B}(\vtheta,r)$  ---}
\State $\vtheta \gets  \vt{\Theta}\left(\x \right) \in (\Delta_\A \cap \s{S}_0)$
\State $r \gets \sqrt{\frac{2\Gap_\lambda(\x,\vtheta)}{\alpha_{\Delta_\A \cap \SO}}}$
\State \textit{--- Main loop  ---}
	\Repeat
		\State \textit{--- Solver update restricted  to preserved set  ---}
		\State $\{\x_{\s{A}},\vt{\eta}\} \gets \mathtt{PrimalUpdate}(\x_{\s{A}},\A_{\s{A}}, \lambda, \vt{\eta})$ 
		\State \textit{--- Dynamic Screening (adaptive local variant) ---}
		\State $\vtheta^\mathrm{old} \gets \vtheta$
		\State $\vtheta \gets \vt{\Theta}\left(\x \right) \in (\Delta_\A \cap \SO)$ \Comment{Dual update} \label{alg:line:dual_update_3}
		\State $r \gets \max(r,\|\vtheta - \vtheta^\mathrm{old}\|)$ \Comment{Initialize safe radius} \label{alg:line:inner_loop_begin}
		\State $r \gets \sqrt{\frac{2\Gap_\lambda(\x,\vtheta)}{\alpha_{\s{B}(\vtheta^\mathrm{old}, r) \cap \SO}}}$ \label{alg:line:init_radius2}
		\Repeat \Comment{Refine safe radius} \label{alg:line:begin_loop}
			\State $r \gets \min\left(r,\sqrt{\frac{2\Gap_\lambda(\x,\vtheta)}{\alpha_{\s{B}(\vtheta, r) \cap \SO}}}\right)$ \label{alg:line:alpha_update_3} 
		\Until {$\Delta r < \varepsilon_r$}\label{alg:line:inner_loop_end}
		\State $\s{A} \gets \{g \!\in\! \s{A}  ~|~ \max_{\vtheta \in \s{B}(\vtheta,r) \cap \SO} \overline{\Omega}_g(\phi(\A_g^\T \vtheta)) \geq 1 \}$ \Comment{Screening test}
		\State $\x_{\s{A}^c} \gets \vt{0}$
		\Until {$\Gap_\lambda(\x,\vtheta) < \varepsilon_{\gap}$} 
\end{algorithmic}
\end{algorithm}

\begin{proposition}\label{propo:CV_refinement}
The sequence of safe radius generated by the refinement loop (lines~\ref{alg:line:inner_loop_begin}--\ref{alg:line:inner_loop_end}) in Algorithm~\ref{alg:solver_screening_local} converges.
\end{proposition}
\begin{proof}
Let $\s{B}(\vtheta^\mathrm{old},r^\mathrm{old})$ be the current safe sphere at the end of a given iteration of the main loop. Let $(\x,\vtheta) \in (\mathrm{dom}(P_\lambda) \cap \s{C}) \times (\Delta_\A \cap \SO)$ be the updated primal-dual pair at the next iteration. Clearly, we have $\s{B}(\vtheta^\mathrm{old},r^\mathrm{old}) \subseteq \s{B}(\vtheta^\mathrm{old},r)$ with $r=\max(r^\mathrm{old},\|\vtheta - \vtheta^\mathrm{old}\|)$. Hence, $\s{B}(\vtheta^\mathrm{old},r)$ is also a safe sphere and, by construction, $\vtheta \in \s{B}(\vtheta^\mathrm{old},r)$. Then, it follows from Theorem~\ref{prop:local_concavity} (with $\s{S} = \s{B}(\vtheta^\mathrm{old},r)\cap \SO$) that $\s{B}(\vtheta,r')$ with $r' = \sqrt{{2\Gap_\lambda(\x,\vtheta)}/{\alpha_{\s{B}(\vtheta^\mathrm{old},r)\cap \SO}}}$ is a safe sphere. This shows that, at line~\ref{alg:line:init_radius2} of Algorithm~\ref{alg:solver_screening_local}, the computed radius $r$ is such that  $\s{B}(\vtheta,r)$ is a safe sphere. Then  Corollary~\ref{coro:localGAP} combined with the ``$\min$'' at line~\ref{alg:line:alpha_update_3} ensures that the refinement loop builds a sequence of nested Gap Safe spheres (i.e., with decreasing radius), all centered in $\vtheta$. As the radius is bounded below by $0$, the proof is completed.
\end{proof}

The proposed radius refinement procedure  can be interpreted as a two-step process. First, a Gap Safe sphere centered in $\vtheta$ is computed from the previous one centered  in $\vtheta^\mathrm{old}$ (lines~\ref{alg:line:inner_loop_begin}--\ref{alg:line:init_radius2}). This step is required as Theorem~\ref{prop:local_concavity} cannot be directly applied since the dual update (line~\ref{alg:line:dual_update_3}) does not necessarily ensures  $\vtheta$ to belong to $\s{B}(\vtheta^\mathrm{old},r^\mathrm{old})$. Then, the radius of this new Gap Safe sphere centered in $\vtheta$ is iteratively reduced as long as the strong concavity bound improves (i.e., increases) when computed on the successively generated Gap Safe spheres (lines \ref{alg:line:begin_loop}--\ref{alg:line:inner_loop_end}).


The computational overhead of this refinement procedure depends essentially on how efficiently the local strong concavity constant $\alpha_{\s{B}(\vtheta,r) \cap \SO}$ can be computed  which, in turn, depends on the dual function  and $\SO$. We shall show in Section \ref{sec:Particular_cases}  that this task can be done  efficiently (in constant time) for several objective functions of practical interest. 
\section{Notable Particular Cases} \label{sec:Particular_cases}

In this section, we apply the proposed generic framework to $\ell_1$-regularized problems
%
with some pertinent data-fidelity functions, namely:  the quadratic distance, $\beta$-divergences with $\beta=1.5$ and $\beta=1$ (Kullback-Leibler divergence) and the logistic regression objective.
%
The $\beta$-divergence \citep{Basu1998} is a family of cost functions parametrized by the a scalar $\beta$ which is largely used in the context of non-negative matrix factorization (NMF) with prominent applications in audio \citep{Fevotte2018} and hyperspectral image processing \citep{Fevotte2015}. It covers as special cases the quadratic distance ($\beta=2$) and the Kullback-Leibler divergence ($\beta=1$).
These examples are particularly interesting as they encompass the three following scenarios.
\begin{enumerate}
    \item The dual cost function is only locally (\emph{not} globally) strongly concave. The standard Gap Safe screening approach is not applicable while the proposed extension is. This is the case for $\beta$-divergences with $\beta \in [1,2)$.
    \item The dual cost function is globally strongly concave, but improved local strong-concavity bounds can be derived. We thus expect the proposed approach to improve over the standard Gap Safe screening. This is the case for the logistic regression.
    \item The dual cost function is globally strongly concave and the global constant $\alpha$ cannot be improved locally. Here the proposed approach reduces to the standard Gap Safe screening. This is the case for the quadratic distance.
\end{enumerate}
To highlight the specificities of each of these problems (e.g., the set $\s{C}$, assumptions on $\A$ and the input data $\y$) 
we formalize them below. 
\begin{itemize}
    \item Quadratic distance: For $\y \in \R^m$, $\A \in \R^{m \times n}$ and $\lambda>0$,
    \begin{equation}
        \x^\star \in \argmin_{\x \in \R^n} \, \frac12 \sum_{i=1}^{m} (\yi - \Axi)^2 + \lambda \|\x\|_1.
    \end{equation}
    \item  $\beta$-divergence with $\beta=1.5$ (hereafter denoted $\beta_{1.5}$-divergence): For $\y \in \R^m_+$, $\A \in \R^{m \times n}_+$, $\lambda>0$, and a smoothing constant $\epsilon >0$ that allows to avoid singularities around zero,
    \begin{equation} \label{prob:beta15_l1}
        %
        \x^\star \in \argmin_{\x \in \R^n_+} \, \frac{4}{3} \sum_{i=1}^{m} \left( \yi^{3/2} + \frac{1}{2}(\Axi+\epsilon)^{3/2} - \frac{3}{2} \yi (\Axi+\epsilon)^{1/2} \right) + \lambda \|\x\|_1.
    \end{equation}
    \item Kullback-Leibler divergence:  For $\y \in \R^m_+$, $\A \in \R^{m \times n}_+$, $\lambda>0$, and $\epsilon >0$,
    \begin{equation} \label{prob:KL_l1}
        %
        \x^\star \in \argmin_{\x \in \R^n_+} \,  \sum_{i=1}^{m}  \left( \yi \log \left(\frac{\yi}{\Axi + \epsilon}\right)  + \Axi + \epsilon - \yi \right) + \lambda \|\x\|_1.
    \end{equation}
    \item Logistic regression:  For $\y \in \R^m$, $\A \in \R^{m \times n}$, and $\lambda>0$, 
    \begin{equation} \label{prob:LogReg_l1}
        %
        \x^\star \in \argmin_{\x \in \R^n} \, \sum_{i=1}^{m} \left(\log\left(1 + \e^{\Axi}\right) - \yi\Axi \right)  + \lambda \|\x\|_1.
    \end{equation}
\end{itemize}

We also assume in all cases that $\A$ has no all-zeros row, which is a natural assumption since otherwise the $i$-th entry $\yi$ becomes irrelevant to the optimisation problem and can simply be removed (along with the corresponding row in $\A$).

All the quantities required to deploy Algorithm~\ref{alg:solver_screening_bis}  and~\ref{alg:solver_screening_local} on these problems are reported in Table~\ref{tab:data-fidelity_cases}. Although calculation details are deferred to Appendix~\ref{apdx:Particular_cases}, we supply the main (generic) ingredients of these derivations in Section~\ref{ssec:main_ingredients}.

\begin{remark}
Since the focus of this paper is more on the constraint set and the data-fidelity term, 
we limit our examples to $\ell_1$-regularized problems. 
Yet, the proposed framework is more general and  can cope with  other regularizations terms such as those explored  by \citet{Ndiaye2017}. 
They can be easily combined with the data-fidelity terms explored in this paper thanks to the useful quantities reported in \citet[Table 1]{Ndiaye2017}. 
The provided dual norms can be directly plugged into equation \eqref{eq:GLM_dual_feasible_set} and Proposition~\ref{propo:Safe_screen} to derive the associated dual feasible set and screening test respectively. 
Other quantities, like the set $\SO$ and the local strong concavity bound $\alpha_{\Delta_\A \cap \SO}$ provided in Table~\ref{tab:data-fidelity_cases}, may need to be redefined accordingly (less straightforward). 
\end{remark}

\subsection{Useful Quantities} \label{ssec:main_ingredients}


\subsubsection{Regularization and dual norm}

Let us instantiate some of the generic expressions in  \Cref{sec:screening_GLM} for the case considered in this paper: $\Omega(\x) = \|\x\|_1$. Here, each group $g$ is defined as an individual coordinate with $\s{G} = \{\{1\}, \dots, \{n\}\}$.
\begin{equation}
\begin{array}{llll}
&\Omega(\x) = \|\x\|_1
& 
&\overline{\Omega}(\x) = \|\x\|_\infty = \max_{j \in [n]} |\xj|  \\
&\Omega_g(\x_g) = |x_g|
&
&\overline{\Omega}_g(\x_g) = |x_g|
\end{array}\label{eq:DualL1}
\end{equation}
\subsubsection{Maximum Regularization Parameter}

We call \textit{maximum regularization parameter} the parameter $\lambda_{\max}>0$  such that for all $\lambda \geq \lambda_{\max}$ the zero vector is a solution of \eqref{eq:GLM_reg_problem}. 
The following result generalizes \citet[Proposition 4]{Ndiaye2017} to the framework considered in this paper (see \Cref{app:prop:max_regularization} for a proof).

\begin{proposition} \label{prop:max_regularization}
$ \lambda_{\max} = \overline{\Omega}\left(\phi\left(-\A^\T\nabla F(\vt{0})\right)\right).$
\end{proposition}

\paragraph{$\ell_1$-norm case.}
For $\Omega(\x) = \|\x\|_1$, we have 
\begin{align} \label{eq:max_regularization}
    \lambda_{\max} := \|\phi\left(-\A^\T\nabla F(\vt{0})\right) \|_\infty  =
    \left\lbrace
    \begin{array}{ll}
         \|\A^\T\nabla F(\vt{0}) \|_\infty &  \text{if}~ \s{C} = \R^n\\
         \max \left( -\A^\T\nabla F(\vt{0}) \right) & \text{if}~ \s{C} = \R^n_+
    \end{array} \right.
\end{align}
The instantiation of this expression for each case under consideration is provided in \Cref{tab:data-fidelity_cases}.





%

\subsubsection{Dual Feasible Set}

As in our examples $\Omega = \| \cdot  \|_1$, we get from Theorem~\ref{thm:dual_GLM_and_optimality} and~\eqref{eq:DualL1} that
\begin{align}
    \Delta_\A  & =  \{\vtheta \in \R^m ~|~ \lvert \phi(\ATthetaj)  \rvert  \leq 1, ~ \forall j \in [n]  \} \cap \dom(D_\lambda) \\
     & = \{\vtheta \in \R^m ~|~ \|\phi(\A^\T\vtheta)\|_\infty  \leq 1 \} \cap \dom(D_\lambda),
\end{align}
where we recall that $\phi(x) = x$ for problems where $\s{C} = \R^n$ and $\phi(x) = [x]^+$ when $\s{C} = \R^n_+$. Note that, because $|[z]^+ |  = [z]^+$ we have the simplification 
$\{ \vtheta \in \R^m ~|~ \|\phi(\A^\T\vtheta)\|_\infty \leq 1\}=  \{ \vtheta \in \R^m ~|~ \max_{j \in [n]}([\A^\T\vtheta]_j) \leq 1 \}$ 
when $\s{C} = \R^n_+$. Finally, the definition of $\dom(D_\lambda)$ is provided in Table~\ref{tab:data-fidelity_cases} for each considered problem and details can be found in Appendix~\ref{apdx:Particular_cases}.

\subsubsection{Dual Update} \label{ssec:dual_update}

At line~\ref{alg:line:dual_update_bis} (resp., line~\ref{alg:line:dual_update_3}) of Algorithm~\ref{alg:solver_screening_bis} (resp., Algorithm~\ref{alg:solver_screening_local}), one needs to compute a dual point $\vt{\Theta}(\x)  \in (\Delta_\A \cap \SO)$ from the current primal estimate $\x$. This point will define the center of the computed Gap Safe sphere. As discussed in Section~\ref{sec:standard_GAPSafeScreen}, the function $\vt{\Theta} : \R^n \rightarrow \Delta_\A \cap \SO$ should be such that
\begin{enumerate}
\item it can be evaluated efficiently (to limit the computational overhead);
\item $\|\vt{\Theta}(\x) - \vtheta^\star\|$ is as small as possible (to reduce the sphere radius).
\end{enumerate} 
Given a primal point $\x$, the standard practice to compute a dual feasible point is to rescale the negative gradient $-\nabla F(\A\x)$~\citep{Bonnefoy2015,Ndiaye2017}. The rationale behind this choice is that the dual solution is a scaled version of $-\nabla F(\A\x^\star)$ (see  optimality condition~\eqref{eq:GLM_optimality_condition1}). 
We formalize this scaling procedure in Lemma~\ref{lem:dual_scaling} whose proof is given in \Cref{app:lem:dual_scaling}.

\begin{lemma}\label{lem:dual_scaling}
Assume that $\dom(D_\lambda)$ is stable by contraction and let $\dscale : \R^m \rightarrow \R^m$ be the scaling operator defined by
\begin{equation}\label{eq:dual_scaling}
    \dscale(\vt{z}) := \frac{\vt{z}}{\max\left(\overline{\Omega}(\phi(\A^\T \vt{z})),1\right)}.
\end{equation}
Then, for any point $\vt{z} \in \dom(D_\lambda)$, we have $\dscale(\vt{z}) \in \Delta_\A$.
Moreover, for any primal point $\x \in \dom(P_\lambda)$, we have that $\vt{z} = (-\nabla F(\A\x)/\lambda) \in \dom(D_\lambda)$ and therefore 
\begin{align}\label{eq:dual_scaling_proposed}
    \dscale(-\nabla F(\A\x)/\lambda) \in \Delta_\A.
\end{align}
%
Finally, if $F \in C^1$, then $\dscale(-\nabla F(\A\x)/\lambda) \rightarrow \vtheta^\star$ as $\x \rightarrow \x^\star$.
\end{lemma}
From Lemma~\ref{lem:dual_scaling}, we obtain a simple and cheap scaling procedure (the gradient $\nabla F(\A\x)$ being often computed by the primal solver) that allows to obtain a dual feasible point $\dscale(\x) \in \Delta_\A$ from any primal point $\x \in \dom (P_\lambda)$. 
Hence, when $\SO = \R^m$ one can simply set $\vt{\Theta}(\x) = \dscale (-\nabla F(\A\x)/\lambda)$. For other choices of $\SO$, $\dscale$ can be a starting point to define $\vt{\Theta}$ (see \Cref{tab:data-fidelity_cases}).


\paragraph{$\ell_1$-norm case.}
For $\Omega(\x) = \|\x\|_1$, we have 
\begin{align} \label{eq:dual_scaling_L1}
    \dscale(\vt{z}) := \frac{\vt{z}}{\max\left( \| \phi(\A^\T \vt{z}) \|_\infty,1\right)}.
\end{align}

\subsubsection{Sphere Test}

The safe screening rule in Proposition~\ref{propo:Safe_screen} takes the following form when the safe region is a $\ell_2$-ball, $\s{B}(\vtheta,r)$, with center $\vtheta$  and radius $r$:
\begin{align}
\max_{\vtheta' \in \s{B}(\vtheta,r)} \overline{\Omega}_g(\phi(\A_g^\T \vtheta'))
&= \max_{\vt{u} \in \s{B}(\vt{0},1)} \overline{\Omega}_g\left(\phi\left(\A_g^\T(\vtheta + r \vt{u})\right)\right) \nonumber \\
&\leq \max_{\vt{u} \in \s{B}(\vt{0},1)} \overline{\Omega}_g\left(\phi(\A_g^\T\vtheta) + r \phi( \A_g^\T \vt{u})\right) \nonumber \\
&\leq \overline{\Omega}_g\left(\phi(\A_g^\T\vtheta)\right) + r \max_{\vt{u} \in \s{B}(\vt{0},1)} \overline{\Omega}_g\left( \phi( \A_g^\T \vt{u})\right)
\end{align}
where we used the subadditivity and homogeneity of the operator $\phi$ and the norm $\overline{\Omega}_g$.
The resulting screening rule  for the $g$-th group reads:
\begin{align}\label{eq:sphere_test}
 \overline{\Omega}_g\left(\phi(\A_g^\T\vtheta)\right) + r \max_{\vt{u} \in \s{B}(\vt{0},1)} \overline{\Omega}_g\left( \phi( \A_g^\T \vt{u})\right) < 1
 \implies \x_g^\star = \vt{0}.
\end{align}


\paragraph{$\ell_1$-norm case.}
For  $\Omega(\x) = \|\x\|_1$, 
%
%
the screening rule in \eqref{eq:sphere_test} simplifies as: 
\begin{align}\label{eq:sphere_test_L1}
 \left\lvert \phi(\cA_j^\T\vtheta)\right\rvert + r\|\cA_j\| < 1 \implies x_j^{\star} = {0}.
\end{align}
Finally, let us emphasise that in \Cref{alg:solver_screening_bis,alg:solver_screening_local}, the safe region is $\s{B}(\vtheta,r) \cap \SO$. Although the test~\eqref{eq:sphere_test_L1} remains valid, a specific test that accounts for $\SO$ (when $\SO \neq \R^m$) can increase the number of screened variables (see for instance \Cref{prop:KL_screening_test} in \Cref{apdx:KL_screening_test} for the KL-divergence).

\subsubsection{Strong Concavity Bounds}\label{ssec:Strong_Concavity_Bounds}



The strong concavity parameters can be determined by upper-bounding the eigenvalues of the Hessian matrix $\nabla^2 D_\lambda(\vtheta)$.
%
%
%
%
Indeed, a twice-differentiable function is strongly concave with constant $\alpha>0$ if and only if its Hessian's eigenvalues are majorized by $-\alpha$, i.e., all eigenvalues are strictly negative with modulus greater or equal to $\alpha$ \citep[Chapter IV, Theorem 4.3.1]{Hiriart-Urruty1993}.
In all considered particular cases, the resulting dual function is twice differentiable (as detailed in~\Cref{apdx:Particular_cases}). In \Cref{prop:strong_concavity}, we derive the local strong concavity of $D_\lambda$ on a convex set $\s{S}\in\R^m$.


\begin{proposition}\label{prop:strong_concavity}
Assume that $D_\lambda$ given in Theorem~\ref{thm:dual_GLM_and_optimality} is twice differentiable. Let $\s{S} \in \R^m$ be a convex  set  and $\s{I} = \{ i \in [m]: \forall (\vtheta,\vtheta') \in \s{S}^2, \theta_i = \theta_i'\}$ 
a (potentially empty) set of coordinates in which $\s{S}$ reduces to a singleton.
Then, $D_\lambda$ is $\alpha_{\s{S}}$-strongly concave on $\s{S}$ if and only if
\begin{align}
0 < \alpha_{\s{S}} \leq 
\min_{i \in \s{I}^\complement} \;- \sup_{\vtheta \in \s{S}} \;  \sigma_i(\theta_i),
\end{align}
where $\sigma_i(\theta_i) = -\lambda^2 (f_i^*)''(-\lambda\vthetai)$ is the (negative) i-th eigenvalue of the Hessian matrix $\nabla^2  D_\lambda(\vtheta)$.
\end{proposition}
%
\begin{proof}
See \Cref{app:prop:strong_concavity}.
\end{proof}

From \Cref{prop:strong_concavity}, we define a general recipe to derive a strong concavity bound $\alpha_{\s{S}}$ of $D_\lambda$ on $\s{S} \subset \R^m$:
\begin{enumerate}
    \item Compute the eigenvalues $\sigma_i(\vthetai)$ of the Hessian of the dual function $\nabla^2 D_\lambda(\vtheta)$.
    \item Upper-bound the $i$-th eigenvalue $\sigma_i(\vthetai)$ over the set $\s{S}$.
    \item Take the minimum over $i \in \s{I}^\complement$ of the opposite (negative) upper-bounds.
\end{enumerate}
The bound $\alpha_{\s{S}}$ obtained via the above procedure is valid if it is strictly positive. For each considered loss, we  provide in \Cref{tab:data-fidelity_cases} strong concavity bounds for three different choices of set $\s{S}$:
1)  $\s{S}=\R^m$ (global bound), 2) $\s{S} = \Delta_\A \cap \SO$ and 3) $\s{S} = \s{B}(\vtheta,r) \cap \SO$ (local bounds). 
%
The choice of the set $\SO$ is discussed in \Cref{ssec:particular_discussion} and details concerning the derivation of these bounds are provided in~\Cref{apdx:Particular_cases}.

\subsection{Discussion} \label{ssec:particular_discussion}

The quadratic case is quite straightforward, as one can easily verify that
the dual function is globally strongly~concave with constant $\alpha = \lambda^2$ and we can simply take $\SO = \R^m$  (see \Cref{ssec:example_Euc}).
Moreover, because all eigenvalues of the Hessian $\nabla^2 D_\lambda (\vtheta)$ are constant and equal to $-\lambda^2$ for all~$\vtheta$, every local strong-concavity bound coincides with the global one.

In the following sections, we discuss the particularities of the three remaining cases: $\beta_{1.5}$-divergence, Kullback-Leibler divergence, and  logistic regression.

The following set will be useful to define $\SO$ for both $\beta_{1.5}$ and KL divergences.
\begin{definition}
We denote $\s{I}_0 \in [m]$ the set of coordinates for which the input data $\y$ equals zero: 
$$\s{I}_0 := \{i \in [m] ~|~ \yi = 0\}$$
\end{definition}

\subsubsection{\texorpdfstring{$\beta_{1.5}$}{beta=1.5} Divergence Case} \label{ssec:beta15_discussion_SO}

In this case, the $i$-th eigenvalue of the Hessian $\nabla^2 D_\lambda(\vtheta)$ is given by (see~\Cref{ssec:example_beta}):
\begin{align}
    \sigma_i\left( \vthetai \right)  = - \lambda^2 \left( \frac{(\lambda\vthetai)^2 + 2\yi}{\sqrt{ (\lambda\vthetai)^2 + 4\yi}} - \lambda\vthetai \right).
\end{align}
One can see that the eigenvalues are all non-positive, but may
%
vanish in two cases: 
1) when $\vthetai \rightarrow +\infty$,
2) when $\yi=0$ and $\vthetai\geq 0$.

This means that the corresponding dual function is not globally strongly concave. 
Moreover, one can show that the dual function is still not strongly concave when restricted to the dual feasible set $\Delta_{\A}$ (see \Cref{fig:beta15_S0a}).  Therefore, the application of the proposed local approach requires the definition of some additional constraint set $\SO$. 
%
Below, we show that these two problems can be fixed by setting \begin{align}\label{S0_bet1.5}
    \SO = \{\vtheta \in \R^m ~|~ \lambda \vtheta \leq \vt{b}  \},
\end{align}
where $\vt{b}$ is such that $\vt{b}_{\s{I}_0} <0$ and $\lambda \vtheta^\star \leq \vt{b}$.
%
\begin{itemize}
\item 
The first problem is that the eigenvalues $\{\sigma_i\}_i$ are increasing functions of $\{\vthetai\}_i$ (see \Cref{prop:beta15_eigenvalues_increasing}) and tend to zero as $\vthetai \rightarrow +\infty$. This can be prevented by restricting  ourselves to a subset that upper bounds $\vtheta$. 
Unfortunately, it is not sufficient to restrict $\vtheta$ to  the feasible set $\Delta_\A = \{\vtheta ~|~ \A^\T \vtheta \leq 1 \}$
as illustrated in \Cref{fig:beta15_S0a} (one may indefinitely increase the value of a given coordinate by reducing the value of the others).
However, the introduction of $\SO$ defined in~\eqref{S0_bet1.5} fixes this issue, as illustrated in \Cref{fig:beta15_S0b}.
%
\item 
The second problem is that the eigenvalues $\{\sigma_i\}_i$ are equal to $0$ when $i \in \s{I}_0$ and $\theta_i\geq 0$. Again, one can easily see that the introduction of $\SO$ defined in~\eqref{S0_bet1.5} fixes this issue (thanks to $\vt{b}_{\s{I}_0} < \vt{0}$).
\end{itemize}

The bound $\vt{b} = ({\y-\epsilon})/{\sqrt{\epsilon}}$ is simple and valid choice. Indeed, we get from the primal-dual link in optimality condition \eqref{eq:GLM_optimality_condition1} that $\lambda \vtheta^\star \leq ({\y - \epsilon})/{\sqrt{\epsilon}}$ (see equation \eqref{eq:beta15_optimality_condition1} in~\Cref{ssec:example_beta}). Moreover, we have, for all $i \in \s{I}_0$, $b_i= ({y_i - \epsilon})/{\sqrt{\epsilon}} = -\sqrt{\epsilon} <0$. Yet, in \Cref{tab:data-fidelity_cases} and \Cref{prop:beta15_S0}, we provide the expression of a more complex but finer bound $\vt{b}$ that leads to improved strong concavity constants (which are particularly relevant for \Cref{alg:solver_screening_bis}). Finally, note that  $\vtheta^\star \in \SO$, as required.


\begin{figure}
    \centering
    \begin{subfigure}{.5\textwidth}
        \centering
        \includegraphics[width=.8\linewidth]{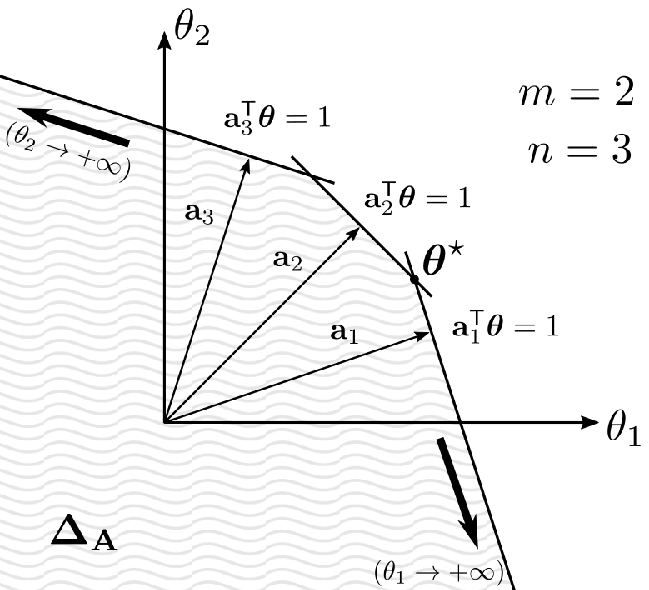} 
        \caption{} \label{fig:beta15_S0a}
    \end{subfigure}%
    \begin{subfigure}{.5\textwidth}
        \centering
        \includegraphics[width=.8\linewidth]{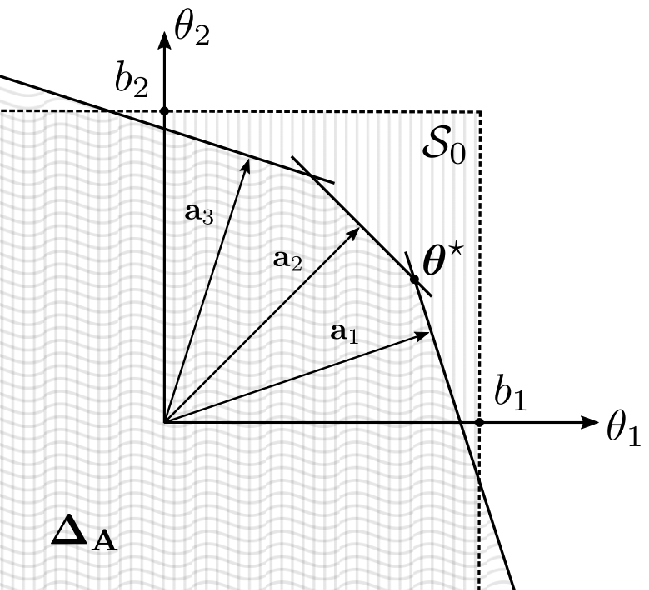} 
        \caption{}\label{fig:beta15_S0b}
    \end{subfigure}%
    \caption{Motivation for the choice of $\SO$ set in the $\beta_{1.5}$-divergence case with $\s{I}_0 = \emptyset$. a) The thick arrows show directions along which $\vthetai \to +\infty$ and, as a consequence, the singular values of the Hessian $\nabla^2 D_\lambda(\vtheta)$ tend to zero. b) Upper-bounding $\vthetai$ solves the problem.}\label{fig:beta15_S0}
\end{figure}

The resulting local strong concavity bounds $\alpha_{\Delta_\A \cap \SO}$ (\Cref{prop:beta15_alpha_fixed}) and $\alpha_{\s{B}(\vtheta,r) \cap \SO}$ (\Cref{prop:beta15_alpha_adaptive}) are presented in  \Cref{tab:data-fidelity_cases}.
The dual update discussed in \Cref{ssec:main_ingredients} is also adapted in view of $\SO$ (\Cref{prop:beta15_dual_update}).

\subsubsection{\texorpdfstring{$\beta =1$}{beta=1} Kullback-Leibler Divergence Case}
\label{ssec:KL_discussion_SO}

For the Kullback-Leibler divergence case, the $i$-th eigenvalue of the Hessian $\nabla^2 D_\lambda(\vtheta)$ is given by (see \Cref{ssec:example_KL}) :
\begin{align}
    \sigma_i (\vthetai) = -\lambda^2\frac{\yi}{(1+\lambda\vthetai)^2}.
\end{align}
One can see that the eigenvalues are all non-positive, but may
%
vanish in two cases:
1) when $|\vthetai| \rightarrow +\infty$.
2) when $\yi=0$.
This means that, like before, the dual function is not globally strongly concave. 
%
The first problem is prevented when we restrict ourselves to any bounded set, which is the case for both $\s{B}(\vtheta, r)$ and $\Delta_\A$.
We recall that in this particular case $\Delta_\A = \{\vtheta \in \R^m  ~|~ \A^\T \vtheta \leq \vt{1} ,~ \vtheta \geq -\vt{1}/\lambda\}$ 
is indeed bounded under the assumptions made on $\A$ (see equation \eqref{eq:bound_theta_KL} in~\Cref{ssec:example_KL}).
However, the case $\yi=0$ remains a problem and,
once again, the application of the proposed local approach requires the definition of some additional constraint set $\SO$. 

Fortunately, for coordinates $i\in \s{I}_0$, the dual solution is trivially determined. Indeed, the primal-dual link in optimality condition \eqref{eq:GLM_optimality_condition1} gives that $\vthetai^\star = -1/\lambda$ (see equation \eqref{eq:KL_optimality_condition1} in~\Cref{ssec:example_KL}).
This leads us to the following constraint set (\Cref{prop:KL_S0}):
\begin{align}
\SO = \{\vtheta \in \R^m ~|~ \vtheta_{\s{I}_0} = - \vt{1}/\lambda\}.
\end{align}
Note that $\SO$ reduces to a singleton on coordinates $i\in \s{I}_0$ and, from \Cref{prop:strong_concavity}, the corresponding eigenvalues of $\nabla^2 D_\lambda$ can be neglected. Moreover, $\vtheta^\star \in \SO$, as required.
This set $\SO$ allows us to compute the local bounds $\alpha_{\Delta_\A \cap \SO}$ (\Cref{prop:KL_alpha_fixed}) and $\alpha_{\s{B}(\vtheta,r) \cap \SO}$ (\Cref{prop:KL_alpha_adaptive}).
%
Additionally, the dual update is adapted considering $\SO$ in \Cref{prop:KL_dual_update_S0} and an improved screening test is defined  in \Cref{prop:KL_screening_test}. All these quantities are reported in \Cref{tab:data-fidelity_cases}.

In a previous paper \citep{Dantas2021}, we addressed the KL particular case and proposed a local bound $\alpha_{\Delta_\A \cap \SO}$ for a screening approach similar to \Cref{alg:solver_screening_bis}. However, the bound proposed here is far superior than the previous one, 
as this will be demonstrated in the experimental section.
Moreover, the iterative refinement approach in \Cref{alg:solver_screening_local} was not a part of  \citet{Dantas2021} and no bound $\alpha_{\s{B}(\vtheta,r) \cap \SO}$ was provided. 

\subsubsection{Logistic Case}\label{ssec:LogReg_discussion}

In this case, the $i$-th eigenvalue of the Hessian $\nabla^2 D_\lambda(\vtheta)$ is given by (see~\Cref{ssec:example_LogReg}):
\begin{align}
    \sigma_i\left( \vthetai \right)  = - \frac{\lambda^2}{(\yi -\lambda \vthetai)(1 -\yi + \lambda \vthetai)}.
\end{align}

Similarly to the quadratic case, the dual function is globally strongly concave as the eigenvalues never vanish.
For this reason, we can simply take $\SO = \R^m$.
On the other hand, differently from the quadratic case, the eigenvalues are not constant and the proposed local approach may lead to better strong-concavity bounds.

%
Indeed, the local strong-concavity bound $\alpha_{\Delta_\A}$ (\Cref{prop:LogReg_alpha_fixed}) 
does improve upon the global bound depending on the regularization parameter $\lambda$. More precisely, it 
improves upon the global constant $\alpha = 4\lambda^2$ (see \Cref{tab:data-fidelity_cases}) for all $$\lambda < \frac{1}{2\|\A^\dagger\|_1}.$$
where $\A^\dagger$ denotes the right pseudo-inverse of $\A$. 
It is important to mention that the proposed $\alpha_{\Delta_\A}$ uses the supplementary assumption that $\A$ is full rank with $\rank(\A) = \min(n,m)$.
If this hypothesis is not fullfiled, one can always initialize \Cref{alg:solver_screening_bis,alg:solver_screening_local} with the global bound $\alpha = 4\lambda^2$. When deploying~\Cref{alg:solver_screening_local}, this bound will be automatically refined during the iterations thanks to the local bound $\alpha_{\s{B}(\vtheta,r)}$  (\Cref{prop:LogReg_alpha_adaptive}).

%
\begin{sidewaystable}
\thisfloatpagestyle{plain}
\renewcommand{\arraystretch}{1.5}
\fontsize{9.5}{11.5}\selectfont 
\begin{tabular}{ |l|c|c|c|c| } 
 \cline{2-5}
 \multicolumn{1}{c|}{} & $\beta = 2$ (quadratic) & $\beta = 1.5$ & $\beta = 1$ (Kullback-Leibler) & Logistic \\ \hline
 $F(\A\x)$
 & $\frac{1}{2}\|\y -  \A\x\|_2^2$ 
 & $\begin{array}{r} \frac{4}{3} \|\y\|_{1.5}^{1.5} + \frac{2}{3}\|\A\x+\epsilon\|_{1.5}^{1.5} \\ - 2 \y^\T (\A\x+\epsilon)^{0.5} \end{array}$
 & $\y^\T \log(\frac{\y}{\A\x + \epsilon})  + \vt{1}^\T (\A\x + \epsilon - \y)$
 & $\vt{1}^\T\log\left(1 + \e^{\A\x}\right) - \y^\T\A\x$  \\ \hline
 %
 %
 $\s{C}$
 & $\R^n$
 & $\R_+^n$
 & $\R_+^n$
 & $\R^n$ \\ \hline
 $D_\lambda(\vtheta)$
 & $\frac{1}{2}\left(\|\y\|_2^2 - \|\y  - \lambda \vtheta\|_2^2  \right)$
 & $\begin{array}{l} \frac{1}{6}\|\lambda\vtheta\|_3^3 - \frac{1}{6} \|(\lambda\vtheta)^2 + 4\y\|_{1.5}^{1.5} \\ \qquad + \lambda \vtheta^\T \y + \frac{4}{3}\|\y\|_{1.5}^{1.5}   - \epsilon \lambda \vtheta^\t \vt{1}  \end{array}$
 & $\y^\T \log(1+\lambda\vtheta) - \lambda\epsilon\vtheta^\T\vt{1}$
 & $\begin{array}{l}  (\y \!-\! \lambda \vtheta \!-\! 1)^\T \log(1\!-\!\y \!+\! \lambda \vtheta) \\ \qquad -(\y\!-\!\lambda \vtheta)^\T \log(\y \!-\!\lambda \vtheta)\end{array}$ 
 \\ \hline
 %
 %
 $\dom(D_\lambda)$
 & $\R^m$
 & $\R^m$
 & $\{\vtheta \in \R^m ~|~ \vtheta \geq -\vt{1}/\lambda\}$
 & $\{\vtheta \in \R^m ~|~ \y-1 \leq \lambda\vtheta \leq \y\}$ \\ \hline
 $\Delta_\A$
 & $\{\vtheta \!\in\! \R^m |~ \|\A^\T\vtheta\|_\infty  \!\leq\! 1 \}$
 & $\{\vtheta \in \R^m ~|~ \A^\T\vtheta  \leq \vt{1} \}$
 & $\{\vtheta \in \R^m ~|~ \A^\T\vtheta  \leq \vt{1}, \vtheta \geq -\vt{1}/\lambda \}$
 & $\begin{array}{l} \{\vtheta \in \R^m ~|~ \|\A^\T\vtheta\|_\infty  \leq 1, \\ \qquad \qquad \;\; \y-1 \leq \lambda\vtheta \leq \y \} \end{array}$ \\ \hline
 $\SO$
 & $\R^m$
 & $\begin{array}{ll}
      &\{\vtheta \in \R^m ~|~ \lambda \vtheta \leq \min( \vt{b},\frac{\y-\epsilon}{\sqrt{\epsilon}} ) \} \\ 
      %
      &b_i = \lambda \min_j \left( \frac{1 - c \|\cA_j\|_1 }{ a_{ij}}  \right) + \lambda c \\ 
      %
      &c = -\frac{1}{\lambda} \sqrt[3]{\frac{4 \|\y\|_{1.5}^{1.5} + 2(m-1)\epsilon^{1.5} + 3\epsilon}{1-3\epsilon} } \nonumber  
    \end{array}$
 & $\begin{array}{ll}
    \{ \vtheta \in \R^m  ~|~ \vtheta_{\s{I}_0} = -\vt{1}/\lambda \} \\
    \text{with } \s{I}_0 = \{ i\in [m] ~|~  \yi\!=\!0 \}
   \end{array}$
 & $\R^m$ \\ \hline
 Primal-dual link
 & $\lambda \vtheta^\star = \y -  \A\x^\star$
 & $\lambda \vtheta^\star = \frac{\y}{\sqrt{\A\x^\star+ \epsilon}} -  \sqrt{\A\x^\star + \epsilon}$
 & $\lambda \vtheta^\star = \frac{\y}{\A\x^\star + \epsilon} - \vt{1} $
 & $\lambda \vtheta^\star = \y - \frac{\e^{\A\x^\star}}{1+\e^{\A\x^\star}}$
 \\ \hline
 $\lambda_{\max}$
 & $\|\A^\T \y\|_\infty$
 & $\max\left(\A^\T (\y-\epsilon) \right)/\sqrt{\epsilon}$
 & $\max\left(\A^\T (\y-\epsilon) \right)/\epsilon$
 & $\|\A^\T (\y-\frac{1}{2})\|_\infty$ \\ \hline
 \pbox{10cm}{Dual update \\$\vt{\Theta}(\x) \!\in\! \Delta_\A \!\cap \SO$}
 & $\dscale\left((\y - \A\x)/\lambda \right)$
 & $\min\left(\left[\dscale\left(\frac{\y-\A\x - \epsilon}{\lambda\sqrt{\A\x + \epsilon}} \right)\right]_i,~ \frac{b_i}{\lambda},~ \frac{\yi - \epsilon}{\lambda\sqrt{\epsilon}}\right)$
 & $\left\lbrace 
    \begin{array}{ll}
       \left[ \dscale\left( \frac{1}{\lambda} \left(\frac{\y}{\A\x+\epsilon} - \vt{1}\right) \right)\right]_i  & \text{if } i \in \s{I}_0^\complement \\
       -\frac{1}{\lambda}  & \text{if } i \in \s{I}_0
    \end{array}\right.$
 & $\dscale\left(\frac{1}{\lambda} \left( \y - \frac{\e^{\A\x}}{1+\e^{\A\x}} \right) \right)$ \\ \hline
 $\nabla^2 D_\lambda(\vtheta)$
 & $-\lambda^2 \diag{1, \dots, 1}$
 & $-\lambda^2 \diag{\left[ \frac{(\lambda\vthetai)^2 + 2\yi}{\sqrt{ (\lambda\vthetai)^2 + 4\yi}} - \lambda\vthetai  \right]_i}$ 
 & $-\lambda^2 \diag{\left[\frac{\yi}{(1+\lambda\vthetai)^2}\right]_i}$
 & $-\lambda^2 \diag{ \left[ \frac{4}{1 -4(\lambda\vthetai - \yi + \frac{1}{2})^2} \right]_i }$ 
 \\ \hline
 $\alpha$ (global)
 & $\lambda^2$
 & --
 & --
 & $4\lambda^2$\\ \hline
 $\alpha_{\Delta_{\A} \cap \SO}$
 & $\lambda^2$
 & $\begin{array}{l}
      \min_{i \in [m]} - \sigma_i 
      \left(\frac{1}{\lambda}\min\left( b_i, \frac{\yi - \epsilon}{\sqrt{\epsilon}}\right)\right) \\ 
      \sigma_i(\frac{z_i}{\lambda}) = -\lambda^2 \left( \frac{ z_i^2 + 2\yi}{\sqrt{ z_i^2 + 4\yi}} - z_i \right)
    \end{array}$
 & $\lambda^2 \min_{i \in \s{I}_0^\complement} \frac{\yi}{ \left(\min_{j}\left( \frac{\lambda + \|\cA_j\|_1}{a_{ij}} \right) \right)^2}$ 
 &  $ \frac{4\lambda^2}{1 - 4 \left(\min(\lambda\|\A^\dagger\|_1,\frac{1}{2}) - \frac{1}{2} \right)^2} $
 \\ \hline
 $\alpha_{\s{B}(\vtheta,r)\cap \SO}$
 & $\lambda^2$
 & $\begin{array}{l}
      \min_{i \in [m]} -\sigma_i(d_i/\lambda) \\ 
      d_i = \min \left( \lambda(\vthetai+r), \; b_i, \; \frac{\yi -\epsilon}{\sqrt{\epsilon}} \right)
    \end{array}$
 & $\lambda^2 \min_{i \in \s{I}_0^\complement} \frac{\yi}{(1+\lambda(\vthetai+r))^2}$
 & $\frac{4 \lambda^2}{1 - 4 ([\min_i(|\lambda \vthetai - \yi + \frac{1}{2}|) - \lambda r]^+)^2 }$ \\ \hline  
 \pbox{10cm}{Screening test \\ ($\Omega(\x) = \|\x\|_1$)} 
 & $|\cA_j^\T \vtheta| + r \|\cA_j\|_2 < 1$ 
 & $\cA_j^\T \vtheta + r \|\cA_j\|_2 < 1$
 & $\cA_j^\T \vtheta + r \|[\cA_j]_{\s{I}_0^{\complement}}\|_2 < 1$
 & $|\cA_j^\T \vtheta| + r \|\cA_j\|_2 < 1$ \\ \hline 
\end{tabular}
\caption{Screening with three instances of $\beta$-divergence and logistic regression.} \label{tab:data-fidelity_cases}
\end{sidewaystable}
%

\section{Experiments} \label{sec:experiments}

In this section, our objective is to 
evaluate the proposed techniques in a diverse range of scenarios and
provide essential insights for efficiently applying such techniques to other problems and settings.
To that end,
a wide range of experiments have been performed
but only a representative subset of the results are presented in this section for the sake of conciseness and readability. 
%
The complete Matlab code is made available by the authors\footnote{Code available at: \url{https://github.com/cassiofragadantas/KL_screening}} for the sake of reproducibility. 
Different simulation scenarios (not reported here) are also available for the interested reader.

\paragraph{Solvers.}
For each of the three 
problems described in \Cref{sec:Particular_cases}, namely
logistic regression, KL divergence, and $\beta_{1.5}$-divergence, we deploy some popular solvers from the literature in their standard form (i.e., without screening) as well as within the three mentioned screening approaches: the existing dynamic Gap Safe approach in \Cref{alg:solver_screening} (when applicable) and the proposed screening methods with local strong concavity bounds in \Cref{alg:solver_screening_bis,alg:solver_screening_local}.
%
Screening is performed at every iteration after the solver update.
Three different categories of solvers are used:  coordinate descent (CoD) \citep{Friedman2010,Hsieh2011,Yuan2010}, multiplicative update (MU) yielding from majorization-minimization \citep{Fevotte2011} and proximal gradient algorithms \citep{Harmany2012}.
%
The chosen solvers and the corresponding variations for each particular case are listed in \Cref{tab:simulation_cases}.

\paragraph{Data Sets.}
Well-suited data sets have been chosen according to each considered problem (see \Cref{tab:simulation_cases}). 
The Leukemia binary classification data set \citep{Leukemia}%
\footnote{Data set available at LIBSVM: \url{https://www.csie.ntu.edu.tw/~cjlin/libsvmtools/datasets/}}
is used for the logistic regression case.
For the KL case, the NIPS papers word count data set \citep{NIPSpapers1-17}%
\footnote{Data set available at: \url{http://ai.stanford.edu/~gal/data.html}}
is considered.
For the $\beta$-divergence case, we use the Urban hyperspectral image data set \citep{Jia2007}
\footnote{Data set available at: \url{https://rslab.ut.ac.ir/data}}, 
since the $\beta$-divergence (especially with $\beta=1.5$) was reported to be well-suited for this particular type of data \citep{Fevotte2015}. 
Additional details are given in the following respective sections.

\begin{table}[h]
\centering
\begin{tabular}{|l|l|l|l|} \hline
Problem     & Solvers  & Variations & Data \\ \hline
Logistic    & CoD 
    & \pbox[c][1cm][c]{10cm}{No screening, 
        Alg.~\ref{alg:solver_screening} (DGS), \\ Alg.~\ref{alg:solver_screening_bis} (G-DGS), Alg.~\ref{alg:solver_screening_local} (R-DGS)
      }
    & \pbox{10cm}{Binary classification \\ (Leukemia data set)} \\ \hline
KL          & \pbox{10cm}{MU, CoD,\\ Prox. Grad.}
    & \pbox[c][1cm][c]{10cm}{No screening,
        \\ Alg.~\ref{alg:solver_screening_bis} (G-DGS), Alg.~\ref{alg:solver_screening_local} (R-DGS)
      }
    & \pbox{10cm}{Count data \\ (NIPS papers data set)} \\ \hline
$\beta=1.5$ & MU              
    & \pbox[c][1cm][c]{10cm}{No screening, 
        \\ Alg.~\ref{alg:solver_screening_bis} (G-DGS), Alg.~\ref{alg:solver_screening_local} (R-DGS)
      }
    & \pbox{10cm}{Hyperspectral data \\ (Urban image)}   \\ \hline
\end{tabular}
\caption{Experimental scenarios}\label{tab:simulation_cases}
\end{table}

\paragraph{Evaluation.}
To compare the tested approaches, two main performance measures are used

\begin{enumerate}[noitemsep]
    \item  Screening rate: how many (and how quickly) inactive coordinates are identified and eliminated by the compared screening strategies.

    \item Execution time: the impact of screening in accelerating the solver's convergence.
    

\end{enumerate}

\Cref{tab:simulation_parameters} specifies the explored values of two parameters with decisive impact in the performance measures.
%
We set the ``smoothing'' parameter of $\beta_{1.5}$ and KL divergence to  $\epsilon = 10^{-6}$.
Remaining problem parameters are fixed by the choice of the data set: 
problem dimensions ($m,n$) and data distribution (both the input vector $\y$ and matrix $\A$).

\begin{table}
\centering
\renewcommand{\arraystretch}{1.2}
\begin{tabular}{|l|l|} \hline
Parameter     & Range \\ \hline
Regularization ($\lambda/\lambda_{\max}$) & $[10^{-3}, 1]$ \\ \hline
Stopping criterion ($\varepsilon_{\gap}$) & $\{10^{-7},10^{-5}\}$ \\ \hline
\end{tabular}
\caption{Simulation parameters and explored values.}\label{tab:simulation_parameters}
\end{table}

This section is organized as follows:
the particular cases of logistic regression, $\beta_{1.5}$-divergence and Kullback-Leibler divergence are treated  
respectively in \Cref{ssec:LogReg_experiments,ssec:beta15_experiments,ssec:KL_experiments}.
Other worth-mentioning properties of the proposed approaches are discussed in \Cref{ssec:further_experiments},
notably the robustness of \Cref{alg:solver_screening_local} to the initialization of the strong concavity bound.

\subsection{Logistic Regression}\label{ssec:LogReg_experiments}

This example is particularly interesting as it allows to compare all screening approaches (\Cref{alg:solver_screening,alg:solver_screening_bis,alg:solver_screening_local}). 
%
%
The classic Leukemia binary classification data set is used in the experiments,%
\footnote{Sample number 17 was removed from the original Leukemia data set in order to improve the conditioning of matrix $\A$, which would be nearly singular otherwise 
and the proposed bound $\alpha_{\Delta_{\A} \cap \SO}$ in \Cref{prop:beta15_alpha_fixed} would reduce to the global $\alpha_{\R^m}$, leading to uninteresting results (although still technically correct). 
}
leading to a matrix $\A$ with dimensions $(m \times n) = (71 \times 7129)$, whose columns are re-normalized to unit-norm. 
Vector $\y$ contains the binary labels $\{0,1\}$ of each sample.
%
A coordinate descent algorithm \citep{Tseng2009a} 
is used to optimize problem \eqref{prob:LogReg_l1}---same as used by \citet{Ndiaye2017}.

\Cref{fig:LogReg_screening_rate_a,fig:LogReg_screening_rate_b} show the screening ratio (number of screened coordinates divided by the total number of coordinates) as a function of the iteration number,
\Cref{fig:LogReg_convergence_vs_time_a,fig:LogReg_convergence_vs_time_b} show the convergence rate (duality gap) as a function of the execution time and
 \Cref{fig:LogReg_Reltime} depicts relative execution times (where the solver without screening is taken as the reference) as a function of the regularization parameter.
\Cref{fig:LogReg_screening_rate_a,fig:LogReg_screening_rate_b} show a clear hierarchy between \Cref{alg:solver_screening,,alg:solver_screening_bis,alg:solver_screening_local} in terms of screening performance, from worst to best. 
The difference is particularly pronounced at lower regularizations---cf. plot with $\lambda/\lambda_{\max} = 10^{-3}$. As regularization grows, \Cref{alg:solver_screening,alg:solver_screening_bis} become equivalent ($\lambda/\lambda_{\max} = 10^{-2}$) and later, around $\lambda/\lambda_{\max} = 10^{-1}$, all approaches become equivalent.
The first mentioned transition point, where \Cref{alg:solver_screening,alg:solver_screening_bis} 
become equivalent, can be theoretically predicted at $\lambda = (2\|\A^\dagger\|_1)^{-1}$, as discussed in \Cref{ssec:LogReg_discussion}. 
In the reported example this corresponds to $\lambda = 1.2 \! \times\! 10^{-2} \lambda_{\max}$. 
This threshold is depicted as a vertical dotted line in \Cref{fig:LogReg_Reltime} and it accurately matches the experimental results. 

The discussed screening performances translate quite directly in terms of execution times, as shown in \Cref{fig:LogReg_convergence_vs_time_a,fig:LogReg_convergence_vs_time_b}.
%
\Cref{fig:LogReg_Reltime} shows the relative execution times, where the basic solver (without screening) is taken as the reference, for a range of regularisation values at given convergence threshold ($\varepsilon_{\gap} = \{10^{-5}, 10^{-7}\}$).
%
A typical behavior of screening techniques is observed: the smaller the convergence tolerance, the more advantageous screening becomes.
Indeed, a smaller convergence tolerance leads to a larger number of iterations in the final optimization stage, where most coordinates are already screened out.
%
A considerable speedup is obtained with \Cref{alg:solver_screening_local} 
in comparison to the other approaches in a wide range of regularization values and, in worst-case scenario, it is equivalent to other screening approaches (i.e., no significant overhead is observed).
These results indicate that the proposed \Cref{alg:solver_screening_local} (R-DGS) should be preferred over the other approaches in this particular problem.

Similar results were obtained with other binary classification data sets from LIBSVM like the colon-cancer data set and the Reuters Corpus Volume I (rcv1.binary) data set.\footnote{For similar reasons as for the Leukemia data set, some samples were removed.}

\begin{figure}
\begin{subfigure}{\columnwidth}
  \includegraphics[height=4cm, trim={0 0cm 1.5cm 0cm},clip]{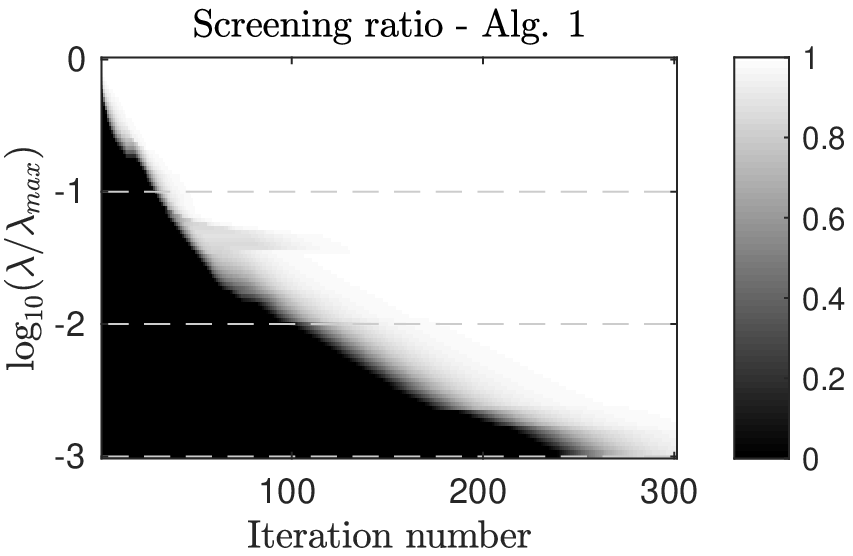}%
  \hfill
  \includegraphics[height=4cm, trim={1cm 0cm 1.5cm 0cm},clip]{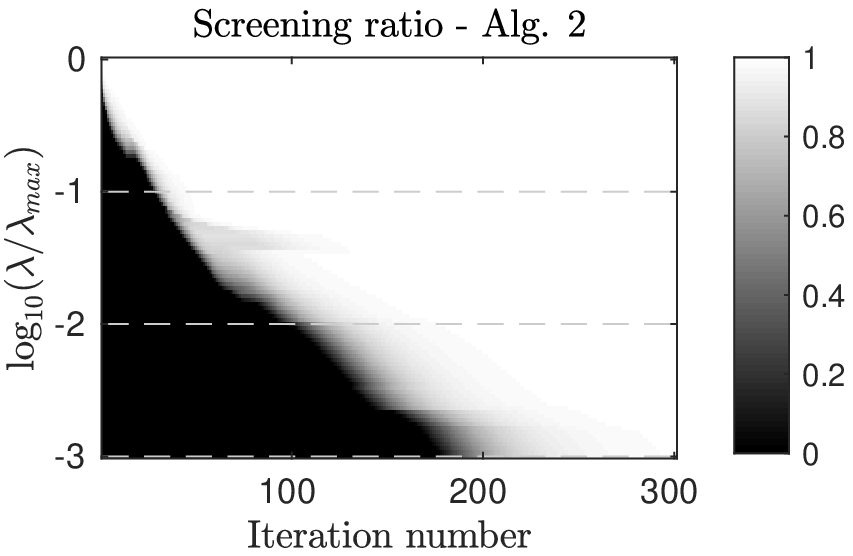}%
  \hfill
  \includegraphics[height=4cm, trim={1cm 0cm 0.1cm 0cm}, clip]{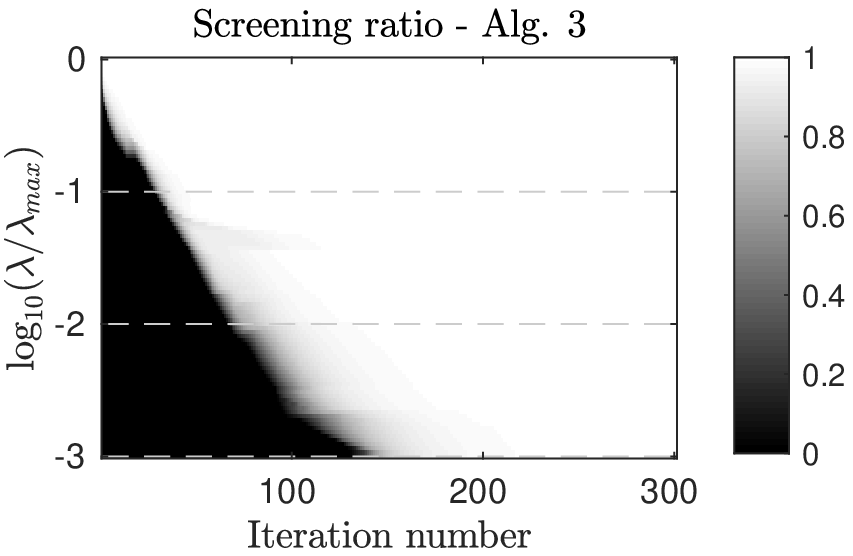}      
  \caption{Screening ratio against iterations for $\lambda/\lambda_{\max} \in [ 10^{-1}, 10^{-3}]$ (the lighter, the more screened coordinates), for Algorithms 1 to 3. \label{fig:LogReg_screening_rate_a}}
\end{subfigure}
\begin{subfigure}{\columnwidth}
  \includegraphics[height=4.2cm, trim={0cm 0cm 0cm -0.2cm}]{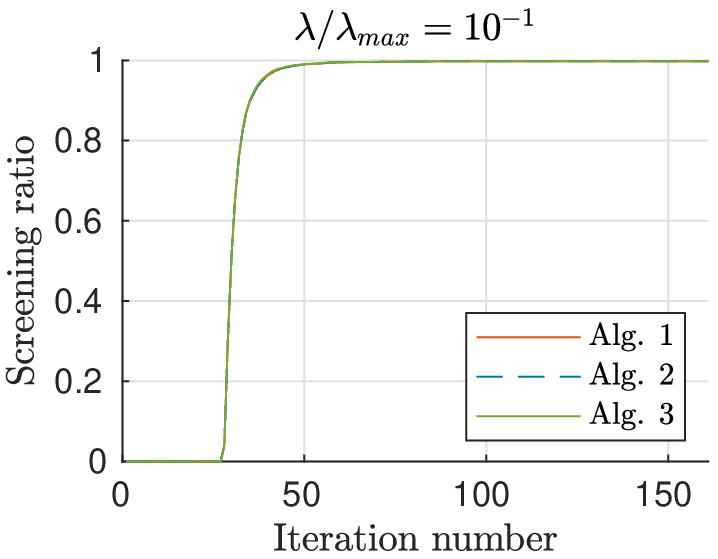}%
  \hfill
  \includegraphics[height=4cm, trim={0.5cm 0cm 0cm 0cm}, clip]{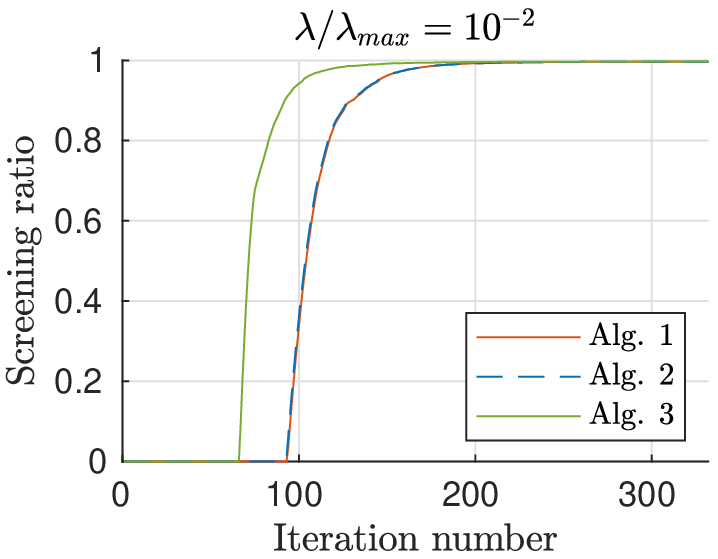}%
  \hfill
  \includegraphics[height=4cm, trim={0.5cm 0cm 0cm 0cm}, clip]{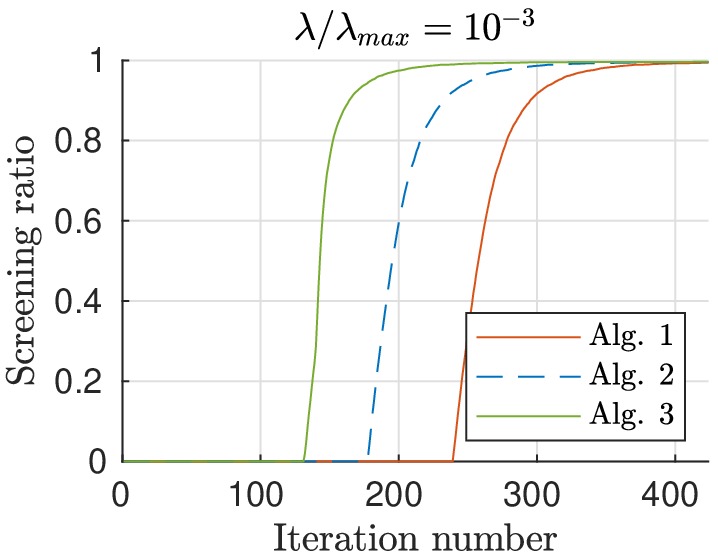}%
  \caption{Screening ratio against iterations for  fixed regularization $\lambda/\lambda_{\max} = \{ 10^{-1}, 10^{-2}, 10^{-3}\}$, corresponding to the dotted slices in \Cref{fig:LogReg_screening_rate_a}.\label{fig:LogReg_screening_rate_b}} 
\end{subfigure}
%
%
\begin{subfigure}{\columnwidth}
  \includegraphics[height=4.2cm, trim={0 0cm 2.1cm -0.2cm},clip]{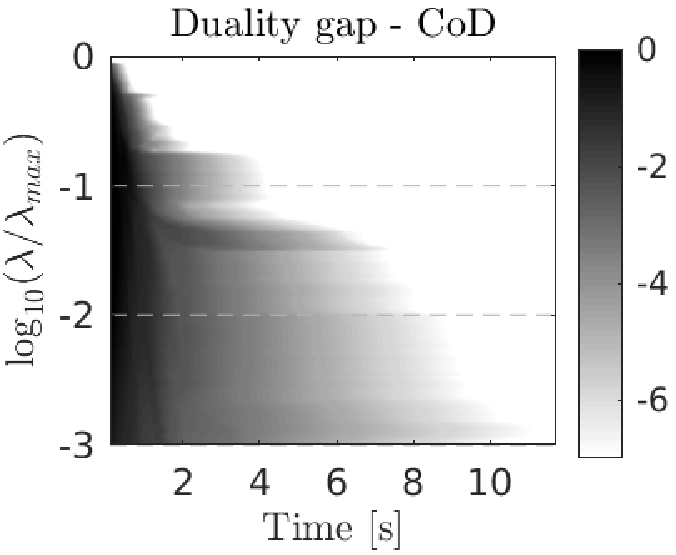}%
  \hfill
  \includegraphics[height=4.2cm, trim={1.1cm 0cm 2.1cm -0.2cm},clip]{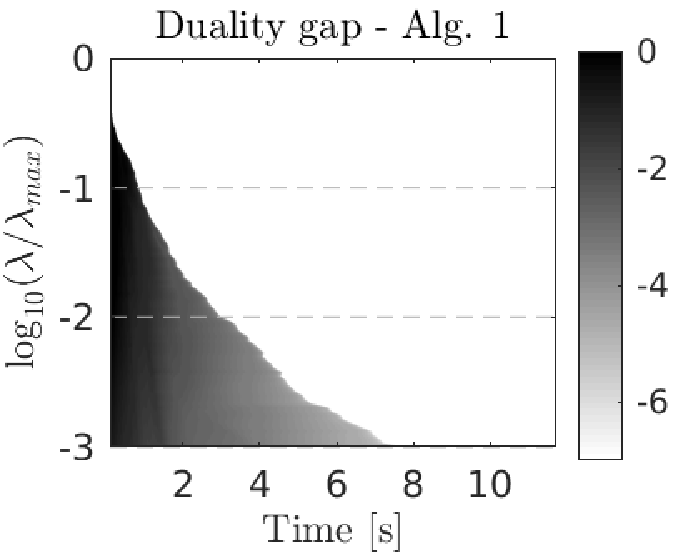}%
  \hfill
  \includegraphics[height=4.2cm, trim={1.1cm 0cm 2.1cm -0.2cm},clip]{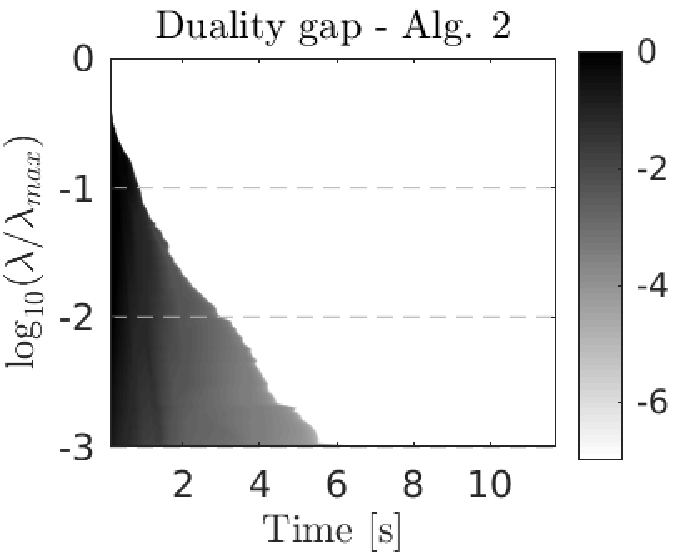}%
  \hfill
  \includegraphics[height=4.2cm, trim={1.1cm 0cm 1cm -0.2cm},clip]{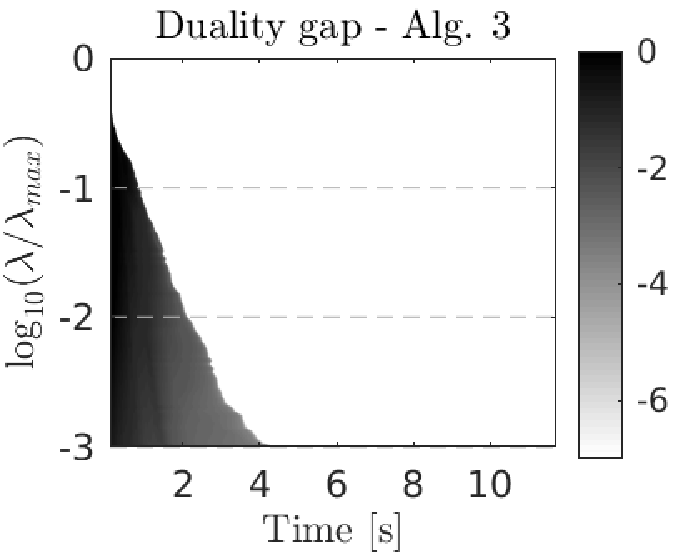}%
  \caption{Convergence rate (duality gap) against execution time for $\lambda/\lambda_{\max} \in [ 10^{-1}, 10^{-3}]$ (the lighter, the closer to convergence). Left to right: Coordinate Descent solver alone and in Algorithms 1 to 3.} \label{fig:LogReg_convergence_vs_time_a}  
\end{subfigure}
%
\begin{subfigure}{\columnwidth}
  \includegraphics[height=4.2cm, trim={0cm 0cm 0cm -0.2cm}, clip]{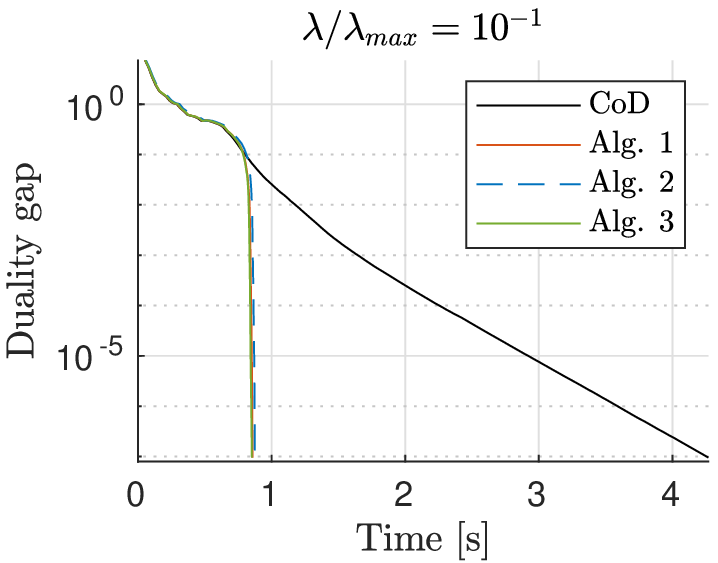}%
  \includegraphics[height=4cm, trim={0.5cm 0cm 0cm 0cm}, clip]{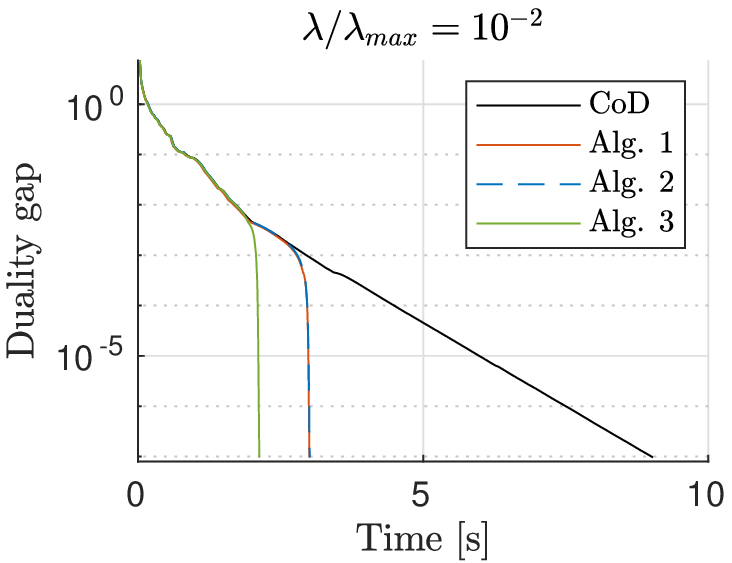}%
  \includegraphics[height=4cm, trim={0.5cm 0cm 0cm 0cm}, clip]{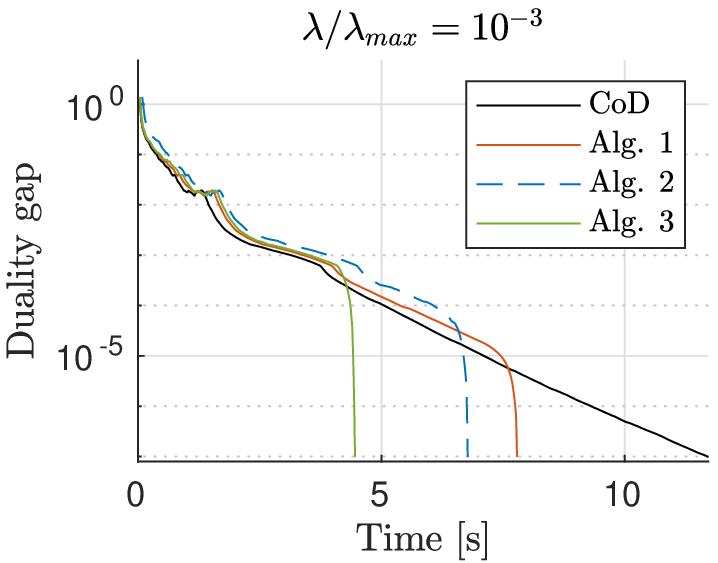}%
  \caption{Duality gap against time for fixed regularization $\lambda/\lambda_{\max} = \{ 10^{-1}, 10^{-2}, 10^{-3}\}$, corresponding to dotted slices in \Cref{fig:LogReg_convergence_vs_time_a}.} \label{fig:LogReg_convergence_vs_time_b}  
\end{subfigure}
\caption{Sparse logistic regression of Leukemia data set using Coordinate Descent and screening.} \label{fig:LogReg_fig_panel}
\end{figure}

\begin{figure}
\begin{subfigure}{.49\columnwidth}
  \centering
  \includegraphics[width=\linewidth, trim={0 0cm 0cm 0cm}]{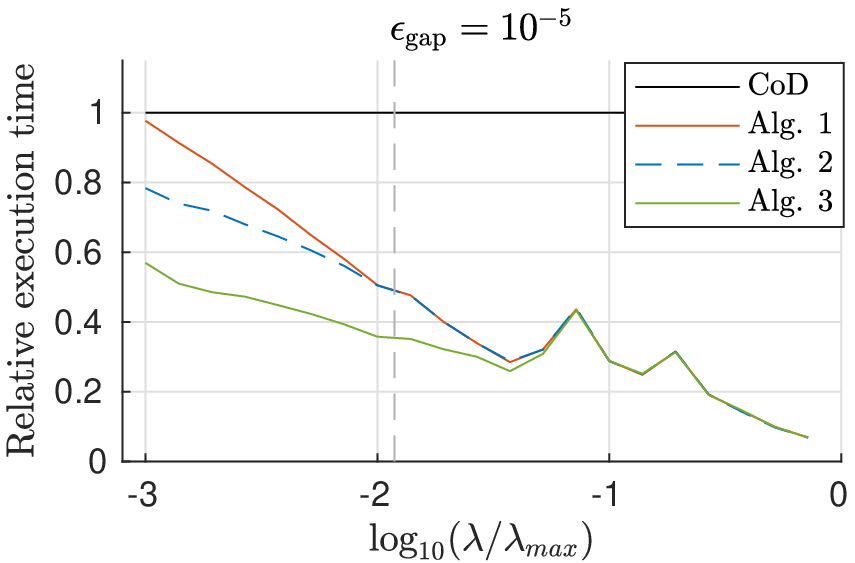}  
\end{subfigure}
\hfill
\begin{subfigure}{.49\columnwidth}
  \centering
  \includegraphics[width=\linewidth, trim={0 0cm 0cm 0cm}]{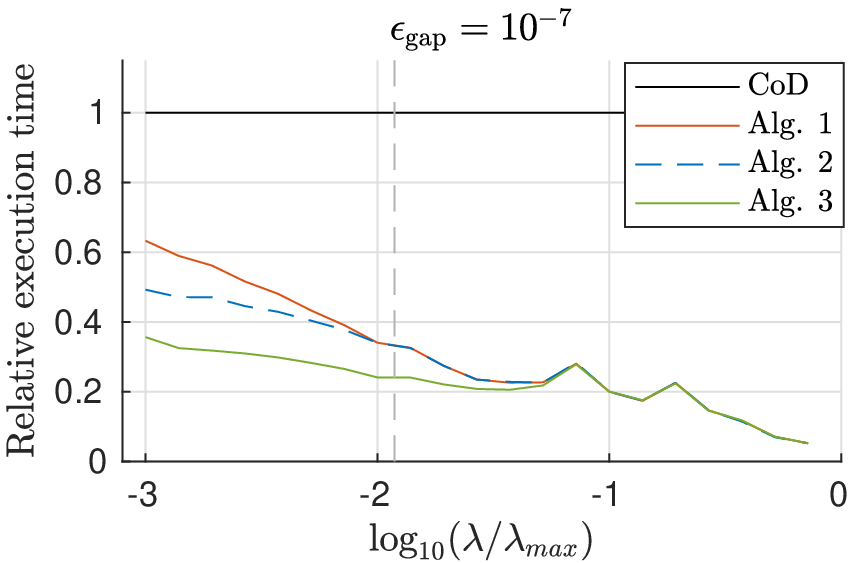}  
\end{subfigure} 
\caption{Sparse logistic regression on Leukemia data set using Coordinate Descent and screening. Relative execution times for $\lambda/\lambda_{\max} \in [ 10^{-3}, 1]$ (the smaller value, the faster) with convergence criterion $\varepsilon_{\gap} =10^{-5}$ (left) and $\varepsilon_{\gap} = 10^{-7}$ (right).} \label{fig:LogReg_Reltime}
\end{figure}

\subsection{\texorpdfstring{$\beta_{1.5}$}{beta=1.5} Divergence} \label{ssec:beta15_experiments}

In this section, we use the Urban data set which is a ($307 \times 307$)-pixel hyperspectral image with $162$ spectral bands per pixel.
For the experiments, 
the input signal $\y \in \R^{162}$ is given by a randomly-selected pixel from the image and the dictionary matrix $\A \in \R^{162 \times 5000}$ is made of a uniformly-distributed random subset of $5000$ of the remaining pixels. 
Therefore, the goal is to  reconstruct a given pixel (in $\y$) as a sparse combination of other pixels from the same image, akin to archetypal analysis \citep{Cutler1994}.
%
%
The multiplicative MM algorithm described in \citep{Fevotte2011} 
is used to solve problem \eqref{prob:beta15_l1}.

\Cref{fig:beta15_screening_rate_a,fig:beta15_screening_rate_b} show the screening performance for \Cref{alg:solver_screening_bis,alg:solver_screening_local}. 
Here, there is an even more pronounced difference between the two approaches (compared to the logistic regression case in \Cref{ssec:LogReg_experiments}) 
for the entire range of regularization values.
%
Note that \Cref{alg:solver_screening_bis} does not screen at all for $\lambda/\lambda_{\max} \leq 10^{-2}$, as opposed to \Cref{alg:solver_screening_local}.
This means that the refinement strategy in the latter approach manages to significantly improve, along the iterations, the initial strong-concavity bound (kept constant by the former).
Even in highly-regularized scenarios, the screening ratio grows significantly faster with \Cref{alg:solver_screening_local}.


A similar behavior is observed in \Cref{fig:beta15_convergence_vs_time_a,fig:beta15_convergence_vs_time_b} regarding execution times.
Because no screening is performed by \Cref{alg:solver_screening_bis} at regularization  $\lambda/\lambda_{\max} = 10^{-2}$ (and below), no speedup is obtained w.r.t. the basic solver---there is even a slight overhead due to unfruitful screening tests calculations.
\Cref{alg:solver_screening_bis} only provides speedup over the basic solver for more regularized scenarios.
\Cref{alg:solver_screening_local}, in turn, provides acceleration over the entire regularization range and significantly outperforms both the basic solver and \Cref{alg:solver_screening_bis}.
The previous observations are summarized in \Cref{fig:beta15_Reltime} for normalized execution times. 

\begin{figure}
\begin{subfigure}{\columnwidth}
  \centering
  \includegraphics[height=4cm, trim={0 0cm -1cm 0cm},clip]{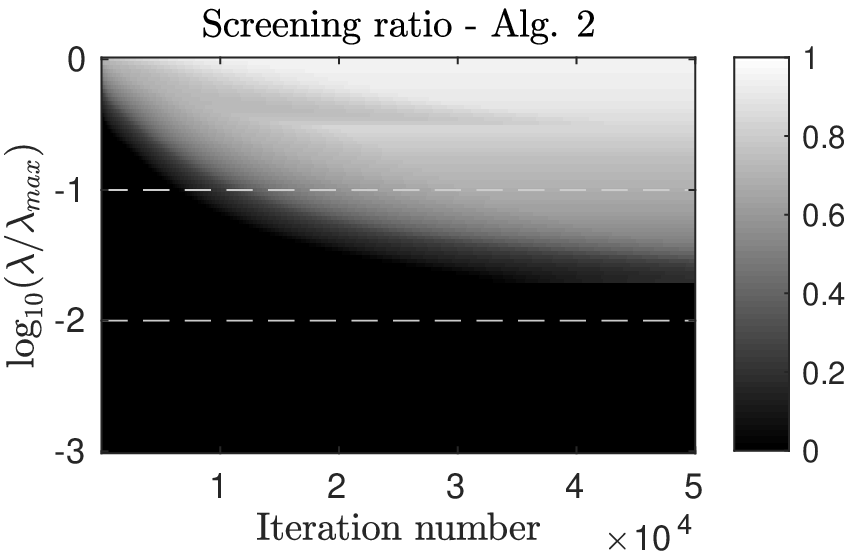}%
  \includegraphics[height=4cm, trim={-1.0cm 0cm 0cm 0cm}, clip]{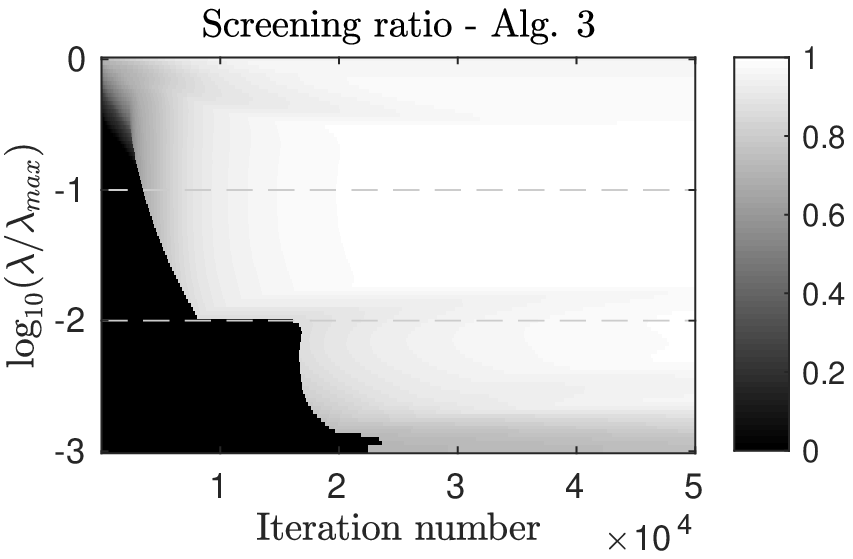}      
  \caption{Screening ratio against iterations for $\lambda/\lambda_{\max} \in [ 10^{-1}, 10^{-3}]$ (the lighter, the more screened coordinates) for Algorithms 2 and 3.} \label{fig:beta15_screening_rate_a}  
\end{subfigure}
%
\begin{subfigure}{\columnwidth}
  \centering
  \includegraphics[height=4.2cm, trim={0cm 0cm -2.0cm -0.2cm}]{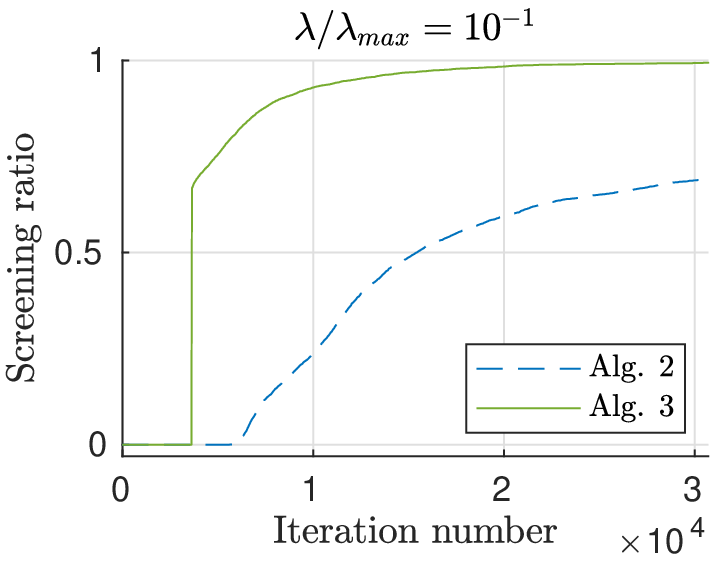}
  \includegraphics[height=4cm, trim={-1.0cm 0cm -1.0cm 0cm}, clip]{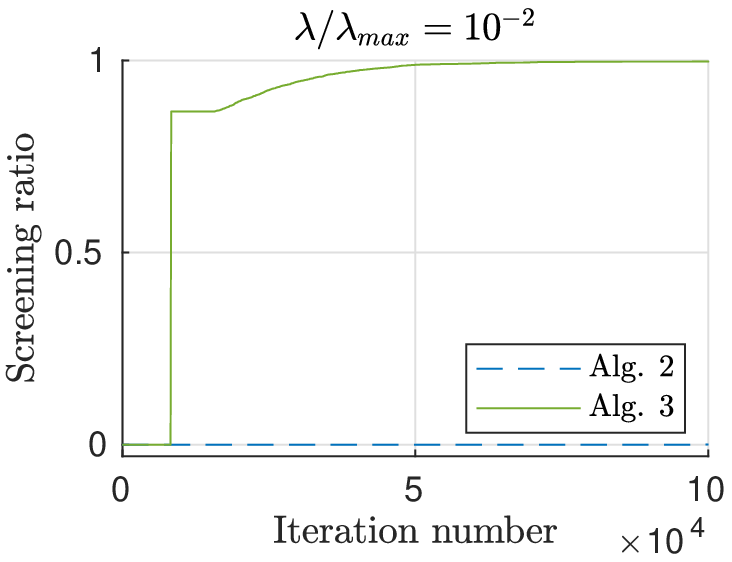}
  \caption{Screening ratio against iterations for  fixed regularization $\lambda/\lambda_{\max} =  10^{-1}$ (left) and $\lambda/\lambda_{\max} = 10^{-2}$ (right), corresponding to the dotted slices in \Cref{fig:beta15_screening_rate_a}.}\label{fig:beta15_screening_rate_b}
\end{subfigure}
%
%
\begin{subfigure}{\columnwidth}
  \includegraphics[height=4cm, trim={0 0cm 2.1cm 0cm},clip]{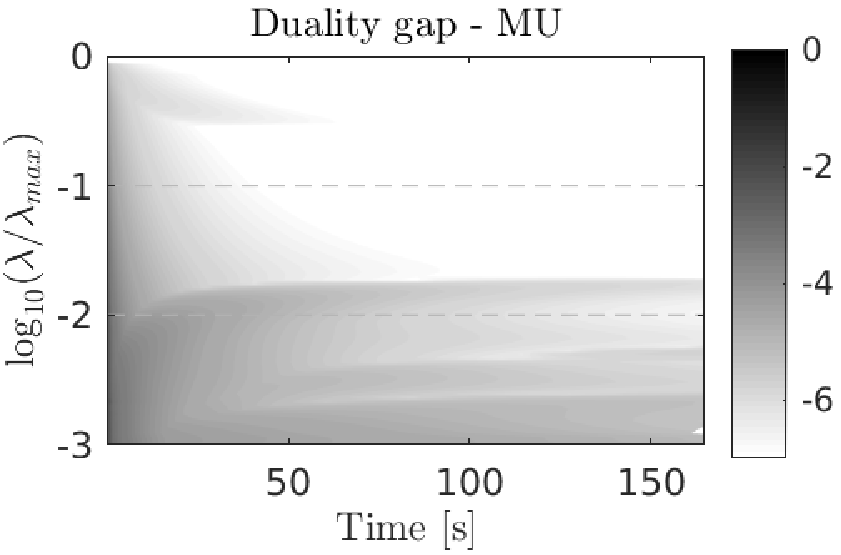}%
  \hfill
  \includegraphics[height=4cm, trim={1cm 0cm 2.1cm 0cm},clip]{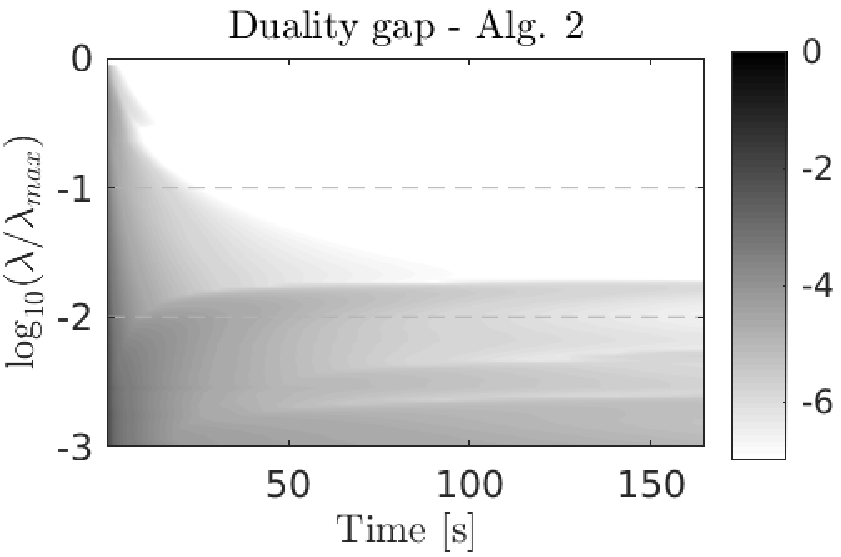}%
  \hfill
  \includegraphics[height=4cm, trim={1cm 0cm 1cm 0cm}, clip]{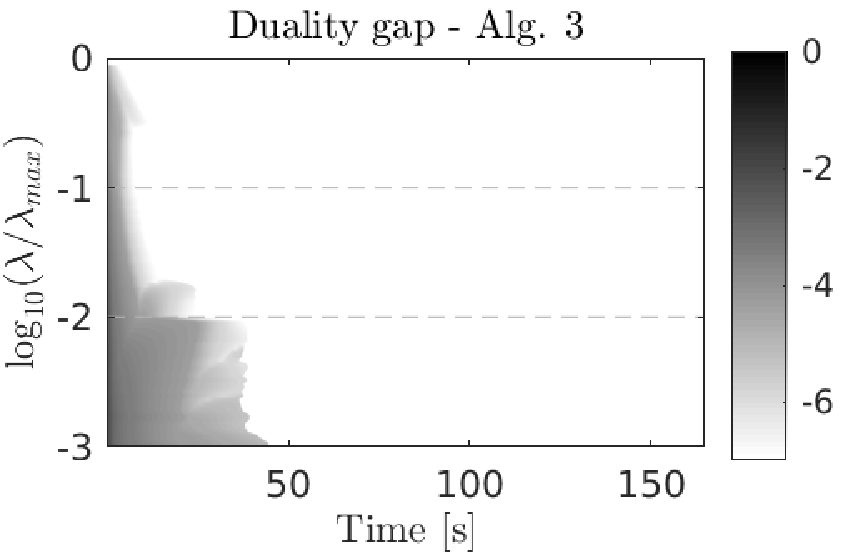}      
  \caption{Convergence rate (duality gap) against execution time for $\lambda/\lambda_{\max} \in [ 10^{-1}, 10^{-3}]$ (the lighter, the closer to convergence). From left to right: MU solver alone and Algorithms 2 and 3.}\label{fig:beta15_convergence_vs_time_a}   
\end{subfigure}
%
\begin{subfigure}{0.6\columnwidth}
  \includegraphics[height=4.2cm, trim={0cm 0cm 0cm -0.2cm}, clip]{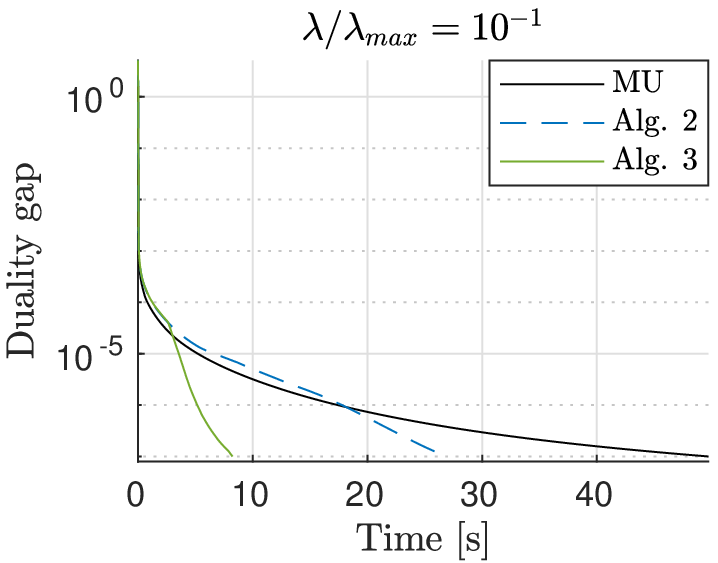}
  \includegraphics[height=4cm, trim={0.5cm 0cm 0cm 0cm}, clip]{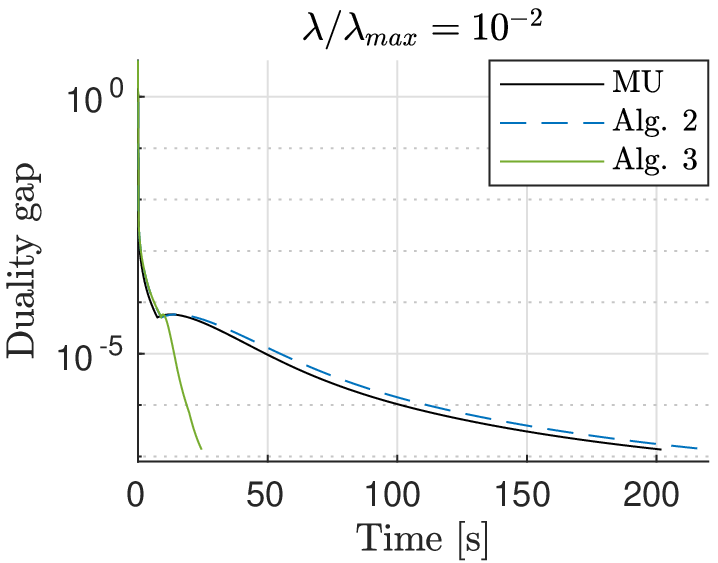}
  \caption{Duality gap against time for regularization $\lambda/\lambda_{\max} = 10^{-1}$ (left) and $\lambda/\lambda_{\max} =  10^{-2}$ (right), corresponding to the dotted slices in \Cref{fig:beta15_convergence_vs_time_a}.}\label{fig:beta15_convergence_vs_time_b}
\end{subfigure}
%
\hfill
\begin{subfigure}{0.32\columnwidth}
  \includegraphics[height=4.2cm, trim={0.5cm 0cm 0cm -0.2cm}, clip]{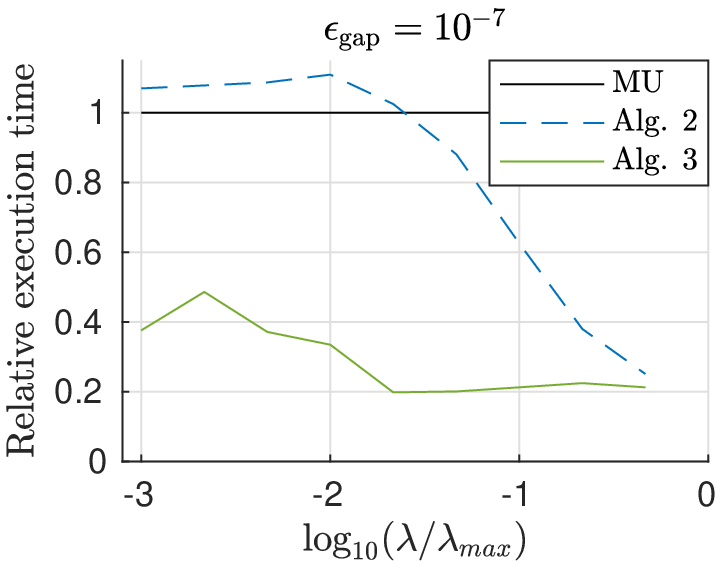}
  \caption{Relative execution times for $\lambda/\lambda_{\max}\in [10^{-3},1]$.}\label{fig:beta15_Reltime}
\end{subfigure}
\caption{Sparse non-negative hyperspectral decomposition using the Urban data set with $\beta =1.5$ and MU.} \label{fig:beta15_fig_panel}
\end{figure}

\subsection{Kullback-Leibler Divergence} \label{ssec:KL_experiments}

We here consider the NIPS papers word counts data set \citep{NIPSpapers1-17}.
The input vector $\y$ is a randomly selected column of the data matrix. The remaining data forms matrix $\A$ (of size $2483 \times 14035$), after removing any all-zero rows and renormalizing all columns to unit-norm. 
Three standard optimization algorithms of distinct types have been used to solve problem~\eqref{prob:KL_l1}: the multiplicative MM algorithm of \citet{Lee2001,Fevotte2011}, the coordinate descent of \citet{Hsieh2011} and the proximal gradient descent of \citet[SPIRAL, ][]{Harmany2012}.
As the different solvers lead to qualitatively similar results,
we have chosen to only report the results for the SPIRAL method in order to avoid redundancy. 
Yet, for completeness, speedup results for all solvers are summarized in Table~\ref{tab:KL_results_speedup}.

Differently from the previous cases, there is no difference between \Cref{alg:solver_screening_bis,alg:solver_screening_local} here. This indicates that: 
1) the strong concavity constant does not vary significantly within the dual feasible set while approaching the dual solution 
and 2) the initial bound for $\alpha_{\Delta_{\A}\cap \SO}$ given in \Cref{prop:KL_alpha_fixed} is nearly tight. 
This fact is verified both in terms of screening performance in \Cref{fig:KL_screening_rate_a,fig:KL_screening_rate_b} and convergence time in \Cref{fig:KL_convergence_vs_time_a,fig:KL_convergence_vs_time_b}.

Nonetheless, we still observe significant speedups w.r.t. the basic solver, which proves the interest of the proposed techniques (see \Cref{fig:KL_convergence_vs_time_a,fig:KL_convergence_vs_time_b} and \Cref{tab:KL_results_speedup}). 
%
Acceleration by a factor of 20 
are reported in Table~\ref{tab:KL_results_speedup} 
with remarkably stable results across the different regularization regimes.
%
The obtained results are also about 3 times better than those reported in our previous work \citep{Dantas2021}, which corresponds to \Cref{alg:solver_screening_bis} with a different 
bound for $\alpha_{\Delta_{\A} \cap \SO}$.
This indicates that the new strong concavity bound derived in the present paper is significantly tighter than the one given in~\citep{Dantas2021}. Indeed, a difference of around two orders of magnitude was observed experimentally between the two bounds.

\begin{figure}
\begin{subfigure}{\columnwidth}
  \centering
  \includegraphics[height=4cm, trim={0 0cm -1cm 0cm},clip]{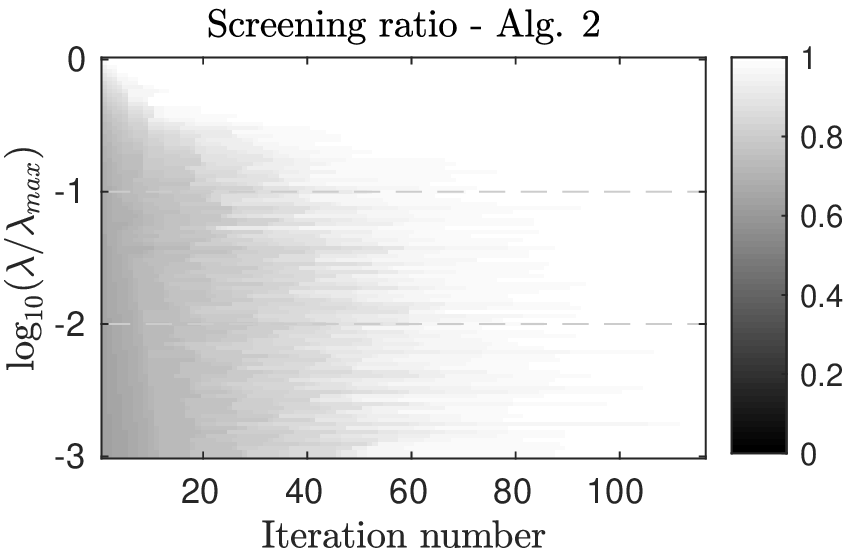}%
  \includegraphics[height=4cm, trim={-1cm 0cm 0cm 0cm}, clip]{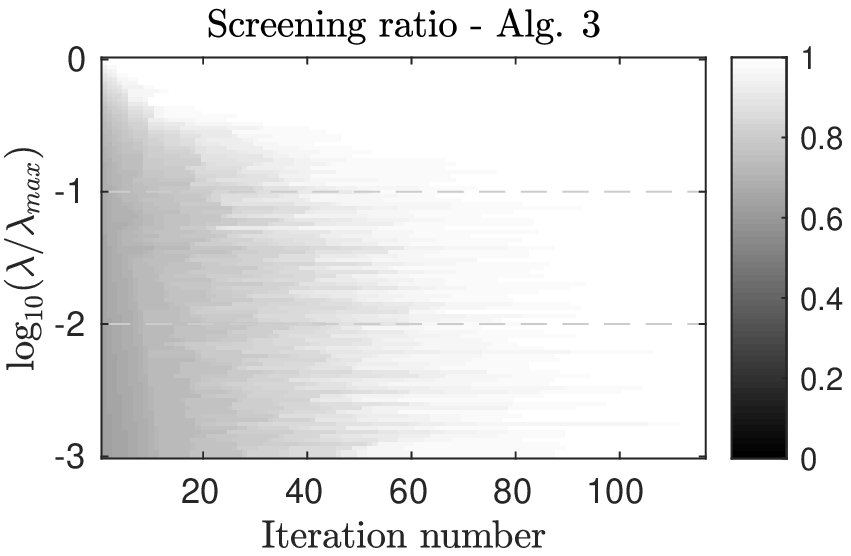}      
  \caption{Screening ratio against iterations for $\lambda/\lambda_{\max} \in [ 10^{-1}, 10^{-3}]$ (the lighter, the more screened coordinates), for Algorithms 2 and 3.}\label{fig:KL_screening_rate_a}
\end{subfigure}
%
\begin{subfigure}{\columnwidth}
  \centering
  \includegraphics[height=4.2cm, trim={0cm 0cm -2.0cm -0.2cm}]{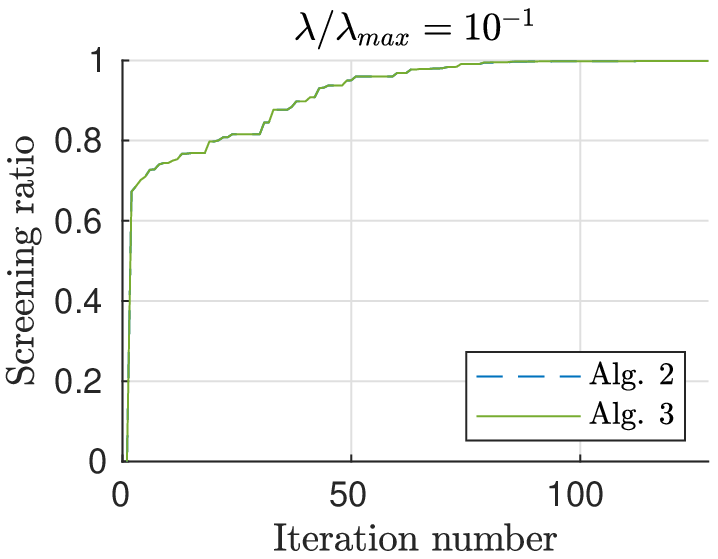}
  \includegraphics[height=4cm, trim={-1cm 0cm -1.0cm 0cm}, clip]{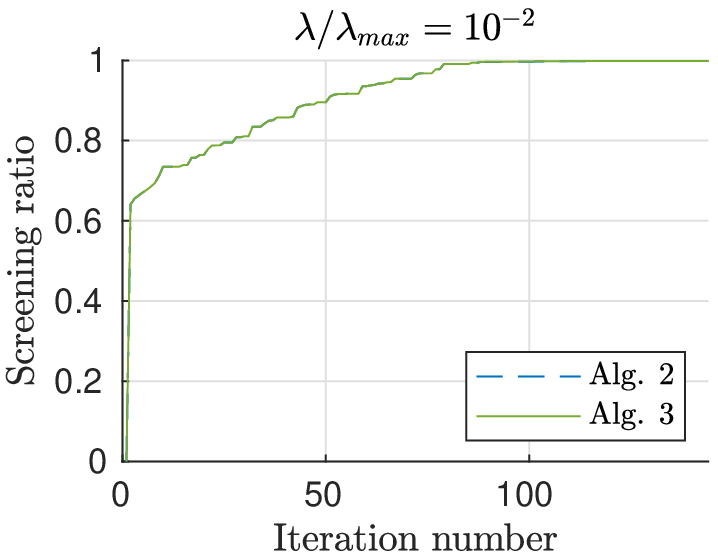}
  \caption{Screening ratio against iterations for fixed regularization $\lambda/\lambda_{\max} =  10^{-1}$ (left) and $\lambda/\lambda_{\max} = 10^{-2}$ (right), corresponding to the dotted slices in \Cref{fig:KL_screening_rate_a}.} \label{fig:KL_screening_rate_b}  
\end{subfigure}
%
%
\begin{subfigure}{\columnwidth}
  \includegraphics[height=4cm, trim={0 0cm 2.1cm 0cm},clip]{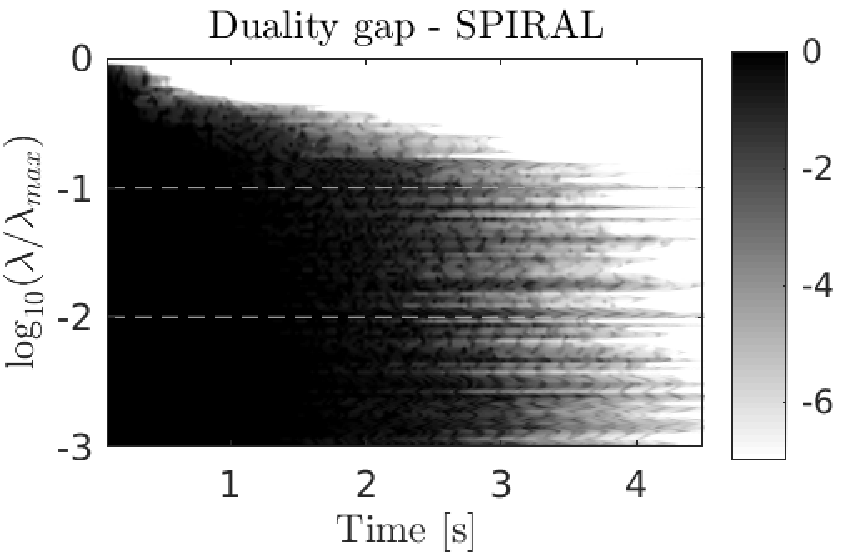}%
  \hfill
  \includegraphics[height=4cm, trim={1cm 0cm 2.1cm 0cm},clip]{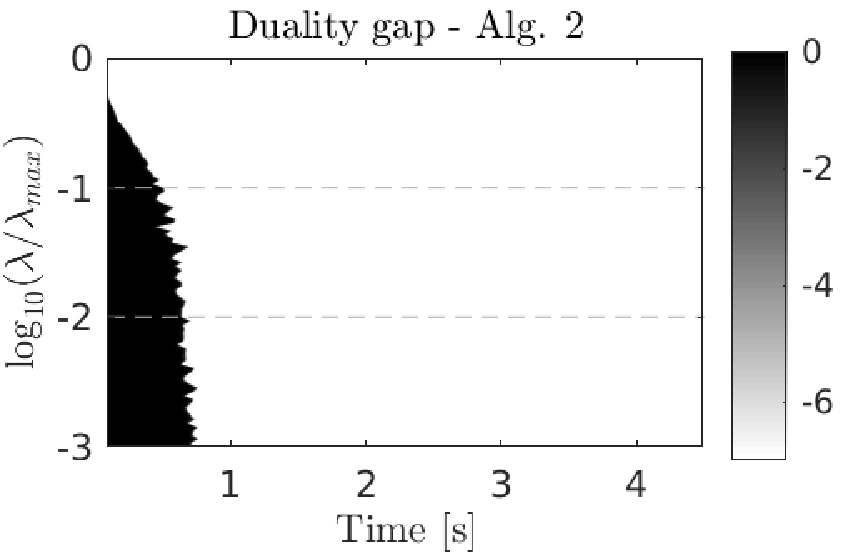}%
  \hfill
  \includegraphics[height=4cm, trim={1cm 0cm 1cm 0cm}, clip]{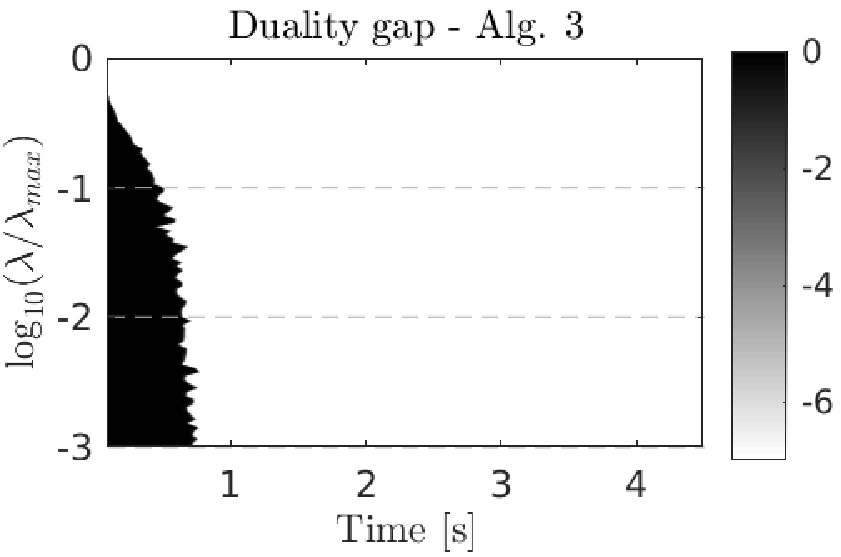}      
  \caption{Convergence rate (duality gap) against execution time for $\lambda/\lambda_{\max} \in [ 10^{-1}, 10^{-3}]$ (the lighter, the closer to convergence). From left to right: SPIRAL solver alone and in Algorithms 2 and 3.} \label{fig:KL_convergence_vs_time_a}
\end{subfigure}
%
\begin{subfigure}{0.6\columnwidth}
  \includegraphics[height=4.2cm, trim={0cm 0cm 0cm -0.2cm}, clip]{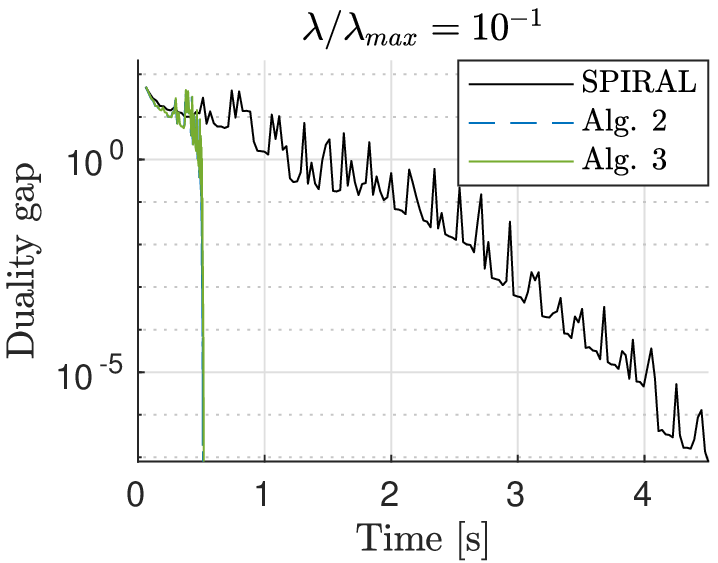}
  \includegraphics[height=4cm, trim={0.5cm 0cm 0cm 0cm}, clip]{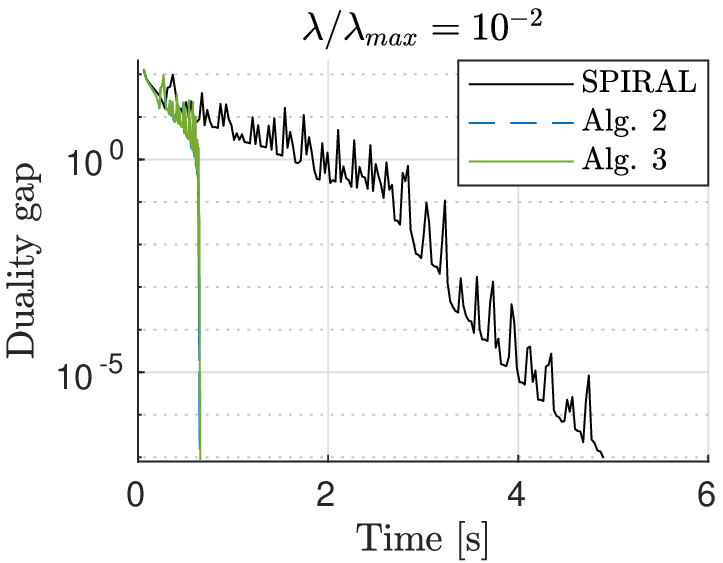}
  \caption{Duality gap against time for regularization $\lambda/\lambda_{\max} = 10^{-1}$ (left) and $\lambda/\lambda_{\max} =  10^{-2}$ (right) corresponding to the dotted slices in \Cref{fig:KL_convergence_vs_time_a}.}  \label{fig:KL_convergence_vs_time_b}  
\end{subfigure}
%
\hfill
\begin{subfigure}{0.32\columnwidth}
  \includegraphics[height=4.2cm, trim={0.5cm 0cm 0cm -0.2cm}, clip]{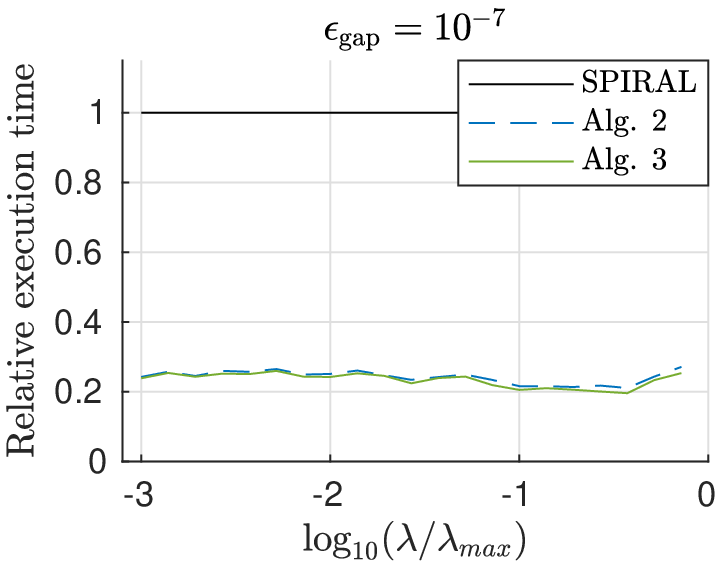}
  \caption{Relative execution times for $\lambda/\lambda_{\max}\in [10^{-3},1]$.} \label{fig:KL_Reltime}  
\end{subfigure}
\caption{Sparse KL regression using NIPS papers word count data set and the SPIRAL solver.} \label{fig:KL_fig_panel}
\end{figure}

%
\def\c#1{#1} 
\begin{table}
\centering
\begin{tabular}{l|l|ll|ll|ll|}
			 \cline{2-8}
             &        &           &           &           &           &           &           \\[-2.3ex]
             & \multicolumn{1}{r|}{$\lambda/\lambda_{\max}$}       
             		  & \multicolumn{2}{c|}{$10^{-1}$} 
             		  			  &\multicolumn{2}{c|}{$10^{-2}$}
             		  			              &\multicolumn{2}{c|}{$10^{-3}$} \\ \cline{2-8}
             &        &           &           &           &           &           &           \\[-2.3ex]
             & \multicolumn{1}{r|}{$\varepsilon_{\gap}$} 
             		  &$10^{-5}$ & $10^{-7}$ & $10^{-5}$ & $10^{-7}$ & $10^{-5}$ & $10^{-7}$ \\ \cline{2-8}
             &        &           &           &           &            &           &          \\ [-2.3ex]
             & \citep{Dantas2021}  & \c{2.77}  & \c{3.21}  & \c{2.50}  & \c{2.83}  &	\c{2.26}  &  \c{2.53}\\
SPIRAL       & \Cref{alg:solver_screening_bis}    & \c{8.81}  & \c{9.61}  & \c{8.68}  & \c{9.69}  & \c{8.54}  & \c{9.36}  \\
             & \Cref{alg:solver_screening_local}    & \c{8.75}  & \c{9.55}  & \c{8.61}  & \c{9.61}  & \c{8.44}  & \c{9.24}  \\ \cline{2-8}
             &           &           &           &           &           &           &           \\ [-2.3ex]
             & \citep{Dantas2021}  & \c{4.19}  & \c{5.35}  & \c{4.06}  & \c{5.13}  & \c{4.12}  & \c{5.06}  \\
CoD  		 & \Cref{alg:solver_screening_bis}    & \c{17.03}  & \c{19.70}  & \c{16.45}  & \c{18.73}  & \c{15.95}  & \c{18.26}  \\
             & \Cref{alg:solver_screening_local}    & \c{17.68}  & \c{20.57}  & \c{16.38}  & \c{18.66}  & \c{16.18}  & \c{18.51}  \\ \cline{2-8}
             &           &           &           &           &           & 		     & 		     \\ [-2.3ex]
             & \citep{Dantas2021}  & \c{6.71}  & \c{8.88}  & \c{6.67}  & \c{9.52}  & \c{5.74}  & \c{7.31}  \\
MU	         & \Cref{alg:solver_screening_bis}    & \c{17.24} & \c{23.56} & \c{20.46} & \c{26.58} & \c{18.28} & \c{23.73}  \\
             & \Cref{alg:solver_screening_local}    & \c{16.56} & \c{24.76} & \c{19.17} & \c{24.89} & \c{17.40} & \c{22.42}  \\ \cline{2-8}
\end{tabular}
\caption{Sparse KL regression: Average speedups (ratio of execution times without and with screening) using the NIPS papers data set. \citet{Dantas2021} is equivalent to Alg. \ref{alg:solver_screening_bis} but with a worse $\alpha_{\Delta_\A \cap \SO}$ constant.} \label{tab:KL_results_speedup}
\end{table}


\subsection{A Deeper Look} \label{ssec:further_experiments}
\subsubsection{Strong concavity bound evolution and robustness to initialization}

\Cref{fig:alpha_evolution} shows the value of the strong concavity bound $\alpha$ over the iterations in the logistic regression scenario. While it is kept constant in \Cref{alg:solver_screening,alg:solver_screening_bis}, it is  progressively refined in \Cref{alg:solver_screening_local}. 
To evaluate the robustness of the proposed refinement approach, we run \Cref{alg:solver_screening_local} with two different initializations: the global constant $\alpha_{\R^m}$ and the local constant $\alpha_{\Delta_{\A} \cap \SO}$ (used respectively in \Cref{alg:solver_screening,alg:solver_screening_bis}).

The left plot in \Cref{fig:alpha_evolution} shows that even with a poor initialization of $\alpha$, the proposed refinement approach will quickly improve the provided bound to match the better initialization, even though there is about one order of magnitude difference between both initializations.
This indicates that the proposed refinement approach makes \Cref{alg:solver_screening_local} quite robust to the initialization of $\alpha$
and that the global $\alpha_{\R^m}$ constant can be used as an initialization of \Cref{alg:solver_screening_local} instead of the proposed $\alpha_{\Delta_{\A} \cap \SO}$. This can be useful, for instance, when the additional full-rank assumption required by the latter bound is not met.
More broadly, it suggests that \Cref{alg:solver_screening_local} can be applied to a new problem at hand without requiring an accurate initial estimation of the problem's strong-concavity constant
(provided, obviously, that one is capable of performing the refinement step, i.e., computing the bounds over $\s{B}(\vtheta,r) \cap \SO$, which tends to be easier).

In the example depicted in \Cref{fig:alpha_evolution} the poor initialization of $\alpha$ is compensated dozens of iterations before screening even starts to take place 
and, as a consequence, no performance loss is inflicted by the poor initialization.
Obviously, in other cases, this compensation might not be as quick, causing some harm to the final execution time. 
Yet, in any case, the proposed refinement approach 
significantly mitigates the impact of a poor initialization on the overall screening performance (and execution time).

\begin{figure}
  \centering
  \includegraphics[width=0.48\linewidth, trim={0 0cm 0cm 0cm}]{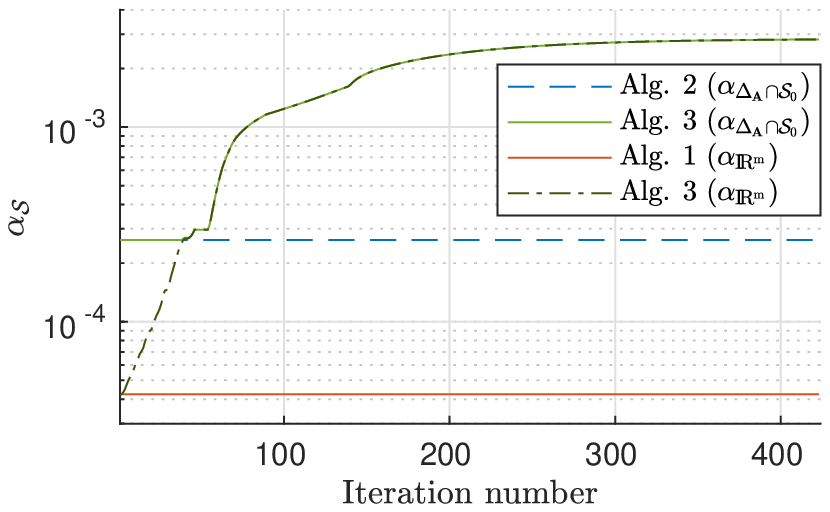}%
  \hfill
  \includegraphics[width=0.48\linewidth, trim={0 0cm 0cm 0cm}]{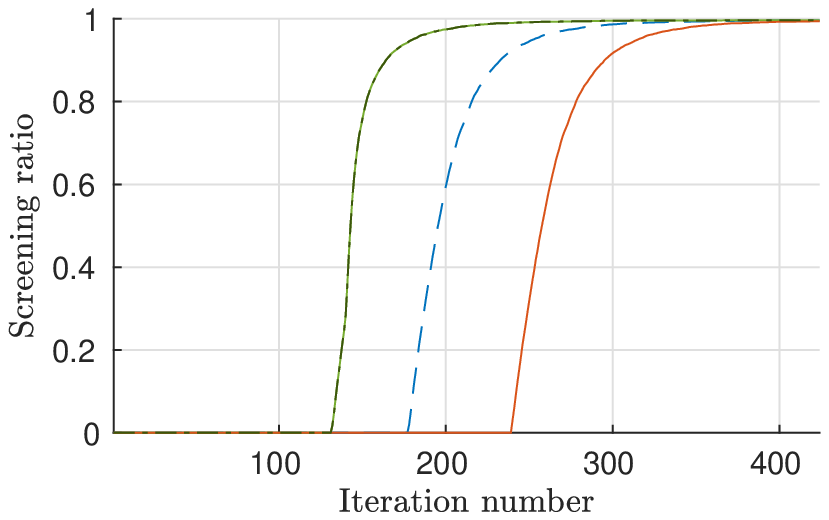}
\caption{Left: evolution of the strong concavity bound $\alpha$ against iterations for the logistic regression problem, $\lambda/\lambda_{\max} = 10^{-3}$. \Cref{alg:solver_screening_local} is initialized with local $\alpha_{\Delta_{\A}\cap \SO}$ (solid light green line) and global $\alpha_{\R^m}$  (dash-dotted dark green line) bounds. Right: impact of the initialization over the screening performance.} \label{fig:alpha_evolution}
\end{figure}

\subsubsection{Support identification for MU solver}

Finally, we discuss and illustrate the power of screening for MU solvers using a small-size synthetic experiment (for readability). Multiplicative updates are very standard in (sparse) NMF, in particular with the general $\beta$-divergence \citep{Fevotte2011}%
\footnote{More efficient, e.g., proximal-based, coordinate-descent or active set methods exist for the quadratic \citep{Friedman2010,Beck2009,Johnson2015} or KL particular cases \citep{Harmany2012,Hsieh2011,Virtanen2013}.}
but suffer from a well-known limitation. Because each coordinate (either of the dictionary or activation matrix) is multiplied by a strictly positive factor, convergence to zero values can only be asymptotical. This is shown in the left plot of \Cref{fig:coord_evolution} which shows the value of each coordinate $\xj$ of the primal estimate $\x$ over the iterations for the MU solver on the $\beta_{1.5}$-divergence case with dimensions $m=10$ and $n=20$.
In practice, some arbitrary thresholding may be performed to force some entries to zero. However, 
such an ad-hoc operation may mistakenly cancel coordinates that belong to the support but happen to have small values.
Conversely, the value of some coordinates might decrease very slowly with the iterations (see for instance the yellow line in \Cref{fig:coord_evolution_a}) and may be mistakenly kept by such an arbitrary thresholding procedure.
The proposed screening approach (\Cref{fig:coord_evolution_b}) tackles this issue by introducing \emph{actual} zeros to the solution with theoretical guarantees. Application of such strategies in NMF settings is an exciting and potentially fruitful perspective of this work.


%
%

\begin{figure}
  \centering
  \begin{subfigure}{.5\textwidth}
    \centering
    \includegraphics[width=0.96\linewidth, trim={0 0cm 0cm 0cm}]{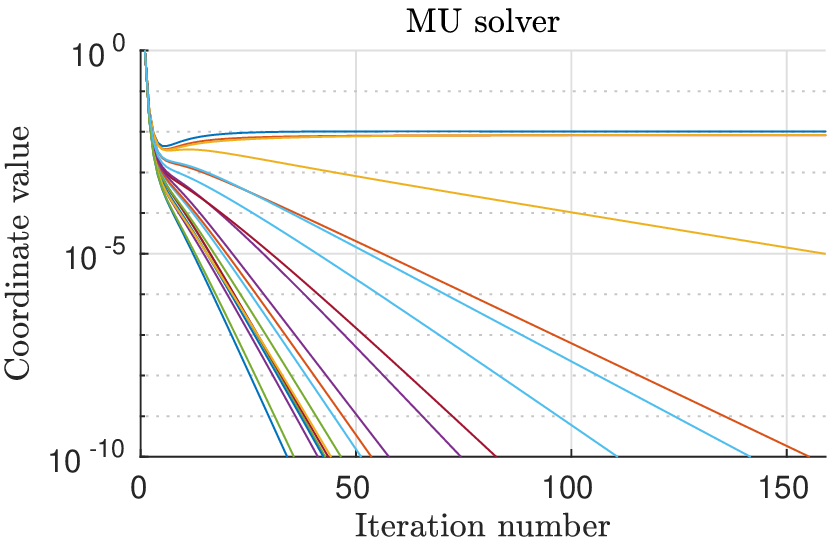}%
    \caption{} \label{fig:coord_evolution_a}
  \end{subfigure}%
  \hfill
  \begin{subfigure}{.5\textwidth}
    \centering
    \includegraphics[width=0.96\linewidth, trim={0 0cm 0cm 0cm}]{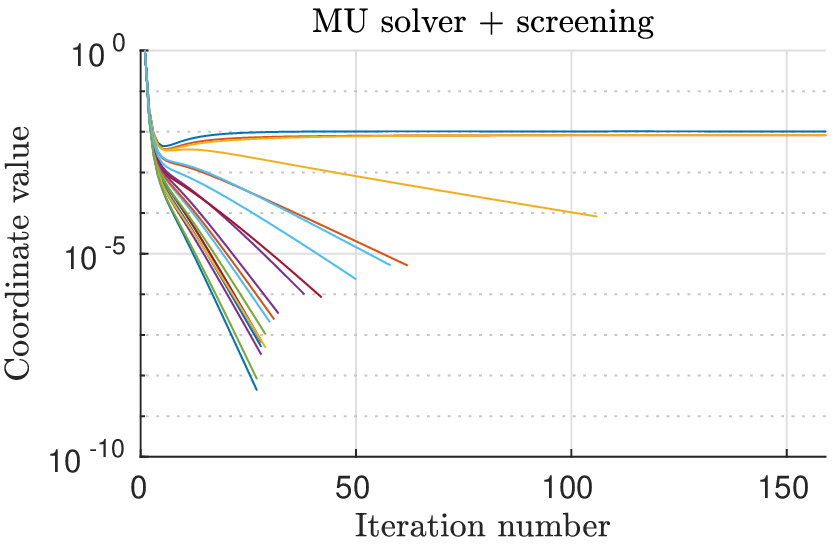}
    \caption{} \label{fig:coord_evolution_b}
  \end{subfigure}%
\caption{Coordinate values over the iterations of an MU solver, $\lambda/\lambda_{\max} = 10^{-2}$ (a) without screening and (b) with screening. Each colored line corresponds to one coordinate, the line stops when the coordinate is screened.} \label{fig:coord_evolution}
\end{figure}

\section{Conclusion}

In this paper, we proposed a safe screening framework that improves upon the existing Gap Safe screening approach, while extending its application to a wider range of problems---in particular, problems whose associated dual problem is not globally strongly concave.

Two screening algorithms have been proposed in this new framework, exploiting local properties of the involved functions. 
First, we defined a direct extension of the conventional dynamic screening approach,
by replacing the global strong concavity bound $\alpha_{\R^m}$ by a local one $\alpha_{\Delta_{\A \cap \SO}}$ in \Cref{alg:solver_screening_bis}.
Noting that a reinforcement loop arises between the current safe sphere and the strong concavity bound within the sphere itself, we proposed the iterative refinement approach in \Cref{alg:solver_screening_local}.
By construction, the latter approach can only improve the initial strong-concavity bound, which makes it strictly superior to the former (the computational overhead due to the extra refinement loop was empirically observed to be negligible).
%
%
Both proposed approaches lead to considerable speedups on several existing solvers and for various simulation scenarios.

\Cref{alg:solver_screening_local} can be superior to \Cref{alg:solver_screening_bis} for two main reasons:
1) The strong concavity constant varies significantly within the dual feasible set. Therefore, refining it on the Gap Safe sphere (which shrinks over the iterations) may lead to significant improvements. 
2) The initial strong concavity bound is loose (for instance, because it cannot be computed exactly and no tighter estimation is available). This initial handicap will then be progressively compensated by the refinement procedure in \Cref{alg:solver_screening_local}.

The proposed framework is quite generic and not restricted to the four treated cases. 
Its application to other problems demands the completion of the following few steps:
1) the possible definition of a subset $\SO$ (only for most challenging objective functions whose dual is not strongly~concave on the entire dual feasible set);
2) the computation of an initial strong concavity bound $\alpha$ such that $\alpha \leq \alpha_{\Delta_{\A} \cap \SO}$ (as observed experimentally, this bound does not need to be very precise);
3) the ability to compute the strong concavity bound on a $\ell_2$-ball (which possibly intersects with $\SO$) for the refinement step. 
%
Other data-fidelity functions that could benefit from this framework include, for instance, the Huber loss, 
the Hinge (and squared Hinge) loss, 
or the Log-Cosh loss.

Although we restricted ourselves to the dynamic screening setting, nothing prevents this idea to be applied to more advanced Gap Safe screening configurations such as sequential screening \citep{Ndiaye2017}, working sets \citep{Massias2017}, dual extrapolation techniques \citep{Massias2020}, or
even the stable safe screening framework \citep{Dantas2019a} in which approximation errors are tolerated in the data matrix.

\acks{This work was supported by the European Research Council (ERC FACTORY-CoG-6681839).}


\newpage
\appendix
\appendixpage

\addcontentsline{toc}{section}{Appendices} 
\part{} 
\parttoc 

\crefalias{section}{appendix} 

\section{Dual Problem and Optimality Conditions}\label{sec:finding_dual}

\subsection{Preliminaries}
In order to prove \Cref{thm:dual_GLM_and_optimality} and its ensuing particular cases,
we will use the classical result of \citet[Theorem 3.3.5]{Borwein2000} or \citet[Theorem 31.3]{Rockafellar1970} which relates generic primal and dual problems with the Fenchel conjugates of the composing functions.

\paragraph{Generic Primal.}
Consider a generic primal problem of the form: 
\begin{align} \label{eq:generic_primal}
\x^\star \in \argmin_{\x \in \R^n} \underbrace{F(\A\x) + \lambda G(\x)}_{:= P_\lambda(\x)}
\end{align}
with $F$ (resp $G$) a closed proper and convex function on $\R^m$ (resp. $\R^n$).
We denote by $P_\lambda$ the primal cost function.

\paragraph{Generic Dual.}
The associated Lagrangian dual problem  
can be shown to be given by 
\citet[Theorem 3.3.5]{Borwein2000} 
%
%
%
\begin{align} \label{eq:generic_dual}
\vtheta^\star 
\in & \argmax_{\vtheta \in \R^m} \underbrace{-F^*(-\lambda\vtheta) - \lambda G^*({A^\T \vtheta})}_{=: D_\lambda(\vtheta)}
\end{align}
where
$D_\lambda(\vtheta)$ denotes the dual (cost) function.
The optimality conditions are given by the following result \citep[Theorem 19.1]{Bauschke2011}.
%

\begin{theorem}\label{thm:optimality_conditions}
Let 
$F: \R^m \to (-\infty, +\infty]$ and  $G: \R^n \to (-\infty, +\infty]$ 
be convex lower semi-continuous,
and let the linear operator $\mt{A} \in \R^{m \times n}$ be such that $\dom (G) \cap \mt{A}\dom (F) \neq \emptyset$. Then the following assertions are equivalent:
	\begin{enumerate}[label=(\roman*)]
	\item $\x^\star$ is a primal solution, $\vtheta^\star$ is a dual solution and $P_\lambda(\x^\star) = D_\lambda(\vtheta^\star)$ (i.e. strong duality holds).
	\item $\A^\T \vtheta^\star \in \partial G(\x^\star)$ and $-\lambda \vtheta^\star \in \partial F(\A\x^\star)$.
	\end{enumerate}
\end{theorem}

Other useful results are given below.

\begin{proposition}[Separable sum property of the Fenchel conjugate]
\citep[Ch. X, Proposition 1.3.1~(ix)]{Hiriart-Urruty1993a}  \label{app:prop:Fenchel_separable_sum}
 Let $F = \sum_{i=1}^m f_i$ be coordinate-wise separable, then
 $$F^*(\vt{u}) = \sum_{i=1}^m f_i^*(u_i) $$
\end{proposition}

\begin{proposition}
Let $F = \sum_{i=1}^m f_i$ be differentiable, then 
its gradient $\nabla F: \R^m \to \R^m $ is given by the (scalar) derivatives of $f_i$ as follows
$$\nabla  F(\vt{z}) = [f_1'(z_1), \dots, f_m'(z_m)]^\T \in \R^m.$$
\end{proposition}

\begin{proposition}[Dual norm of a group-decomposable norm] \citep[Proposition 21]{Ndiaye2018} \label{prop:dual_norm_separable}
Let $\s{G}$ be a partition of $[n]$ and $\Omega(\vt{u}) = \sum_{g\in \s{G}} \Omega_g (\vt{u}_g)$  be a group-decomposable norm. Its dual norm is given by
$\overline{\Omega}(\vt{u}) = \max_{g\in \s{G}} \overline{\Omega}_g(\vt{u}_g)$ where $\overline{\Omega}_g$ is the dual norm of $\Omega_g$
\end{proposition}
\begin{proof} 
The result follows from \Cref{app:prop:Fenchel_separable_sum}:
$\Omega^*(\vt{u}) = \sum_{g\in \s{G}} \Omega_g^*(\vt{u}_g) = \sum_{g\in \s{G}} \ind_{\overline{\Omega}_g(\vt{u}_g) \leq 1}  =  \ind_{\{\vt{u} ~|~ \forall g\in \s{G}, ~ \overline{\Omega}_g(\vt{u}_g) \leq 1\}} =  \ind_{\max_{g\in \s{G}}  \left(\overline{\Omega}_g(\vt{u}_g) \right) \leq 1}$. 
Since $\Omega^*(\cdot) = \ind_{\overline{\Omega}(\cdot) \leq 1}$ we conclude that $\overline{\Omega}(\vt{u})= \max_{g\in\s{G}}\left( \overline{\Omega}_g (\vt{u}_g) \right)$, 
which finishes the proof.
\end{proof}

\subsection{Proof of Theorem~\ref{thm:dual_GLM_and_optimality}} \label{app:GLM_regularized}

Note that in formulation \eqref{eq:generic_primal}, there are no constraints on the primal variable $\x$. 
To make Problem~\eqref{eq:GLM_reg_problem} fit to this setting, we reformulate it as an unconstrained problem by using the indicator function $\ind_{\s{C}}(\x)$.
This strategy was used in \citet{Wang2019} for the non-negative Lasso problem and is extended here to a broader class of problems. 
%
\begin{align} \label{eq:GLM_reg_problem_unconstrained}
\x^\star \in  
\argmin_{\x \in \R^n}  \sum_{i=1}^m f_{i}(\Axi) + \lambda \left( \Omega(\x) +  \ind_{\s{C}}(\x) \right).
\end{align}

Comparing \eqref{eq:GLM_reg_problem_unconstrained} to \eqref{eq:generic_primal}, we have the following direct correspondences:
\begin{itemize}
\item $F(\A\x) = \sum_{i=1}^m f_{i}(\Axi)$
\item $G(\x)  = G_1(\x) + G_2(\x)
= \Omega(\x) +  \ind_{\s{C}}(\x)$
\end{itemize}

In this section, we address both cases $\s{C} = \R^n$ and $\s{C} = \R^n_{+}$, but we emphasize that the former case was previously treated in the literature (see \citet{Ndiaye2018} for instance).

\subsubsection{Dual problem}
The dual problem in \eqref{eq:GLM_reg_dual_problem} is a direct application of the generic result in \eqref{eq:generic_dual}. We just need to evaluate the Fenchel conjugates $F^*$ and $G^*$ in our particular case described above.


$F = \sum_{i=1}^m f_{i}$ being coordinate-wise separable, we can apply \Cref{app:prop:Fenchel_separable_sum}: 
\begin{align} \label{eq:F_fenchel}
F(\vt{z}) = \sum_{i=1}^m f_i(z_i) ~\Rightarrow~ F^*(\vt{u}) = \sum_{i=1}^m f_i^*(u_i)
\end{align}

Furthermore, we know that the Fenchel conjugate of a norm $\Omega$ is the indicator on the unit-ball of its corresponding dual norm $\overline{\Omega}$ \citep[see for instance][Example 13.32]{Bauschke2011}:
\begin{align*}
G_1(\x) = \Omega(\x) ~~\Rightarrow~~ G_1^*(\vt{u}) = \ind_{\overline{\Omega}(\vt{u})\leq 1}.
\end{align*}

To conclude the proof, we need to consider the indicator function corresponding to the constraint set $\s{C}$. The case $\s{C} = \R^n$ is trivial since $\ind_{\R^n}(\x) \equiv 0$ 
and the Fenchel conjugate $G^*=G_1^*$ is given above. 
For the non-negativity contraint, $\s{C} = \R_{+}^n$, we have \citep[Lemma 18 (i)]{Wang2019}:
\begin{align}
G_2(\x) = \ind_{\R_{+}^n}(\x) ~~\Rightarrow~~	G_2^*(\vt{u}) = \ind_{\R_{-}^n}(\vt{u}). 
\end{align}

Now, to calculate the conjugate of the sum $\Omega(\x) + \ind_{\R_+^n}(\x)$, we use
a property that relates the conjugate of a sum to the so-called \emph{infimal convolution}, denoted $\boxdot$, of the individual conjugates
\citep[Proposition 15.2]{Bauschke2011}. 
\begin{align}
G^*(\vt{u}) = (G_1 + G_2)^*(\vt{u}) = (G_1^* \boxdot G_2^*) (\vt{u}) = \ind_{\overline{\Omega}(\vt{u})\leq 1} \boxdot \ind_{\R_{-}^n}(\vt{u}).
\end{align}
Then, we use the fact that the infimal convolution of two indicator functions is the indicator of the Minkowski sum of both sets \citep[Example 12.3]{Bauschke2011}:
\begin{align}
G^*(\vt{u}) = \ind_{\overline{\Omega}(\vt{u})\leq 1} \boxdot \ind_{\R_{-}^n}(\vt{u}) 
= \ind_{\{\vt{u} = \vt{a} + \vt{b} ~|~ \overline{\Omega}(\vt{a})\leq 1, ~\vt{b} \leq 0 \}}
= \ind_{\overline{\Omega}([\vt{u}]^+)\leq 1}.
\end{align}

\begin{remark}
The last step applies only because both conjugates $G^*_1$ and $G^*_2$ are indicator functions. This is no longer the case when we assume, for instance, a bounded-variable constraint of the form $\xj \in [-a_j,b_j], ~ j \in [n]$.
\end{remark}

This leads to the following results:
\begin{itemize}
\item For $\s{C} = \R^n$, we have $G^*(\vt{u}) = G^*_1(\vt{u}) =\ind_{\overline{\Omega}(\vt{u})\leq 1}$.
\item For $\s{C} = \R_{+}^n$, we have $G^*(\vt{u}) = (G_1+G_2)^*(\vt{u}) =\ind_{\overline{\Omega}([\vt{u}]^+)\leq 1}$.
\end{itemize}
These can be written in a compact form using the non-linearity function $\phi$ as defined in \eqref{eq:phi_GLM_dual}:
\begin{align}\label{eq:G_fenchel}
G^*(\vt{u}) = \ind_{\overline{\Omega}(\phi(\vt{u}))\leq 1}
\end{align}
Applying \Cref{prop:dual_norm_separable} for the dual of a group-separable norm $\overline{\Omega}(\phi(\vt{u})) \!=\! \max_{g\in\s{G}} \! \left( \overline{\Omega}_g(\phi(\vt{u}_g)) \right) $:
\begin{align}\label{eq:G_fenchel_groups}
G^*(\vt{u}) = \ind_{\max_{g\in\s{G}}\left( \overline{\Omega}_g(\phi(\vt{u}_g))\right) \leq 1} .
\end{align}

The resulting dual problem is thus obtained by replacing $F^*$ \eqref{eq:F_fenchel} and $G^*$ \eqref{eq:G_fenchel_groups} in the generic dual problem \eqref{eq:generic_dual}, that is 
\begin{equation}
   D_\lambda(\vtheta) = -F^*(-\lambda\vtheta) -  \ind_{\overline{\Omega}(\phi(\A^\T\vtheta))\leq 1}.
\end{equation}

Although in  \eqref{eq:generic_dual} the dual solution is not guaranteed to be unique, our assumption that $F$ is differentiable implies that $D_\lambda$ is strictly concave \citep[Proposition 18.10]{Bauschke2011} which, in turn, implies the uniqueness of the dual solution. 



\begin{remark}
The constraint set being coordinate separable,
we can always treat each group separately by restricting ourselves to the set of $n_g$ coordinates on group $g$. For simplicity, we derive the remaining results for $\Omega$ on $\R^n$ (the extension for groups with $\Omega_g$ being straightforward).
\end{remark}




\subsubsection{Optimality conditions} \label{ssec:optimality_conditions}

Optimality conditions \eqref{eq:GLM_optimality_condition1} and \eqref{eq:GLM_optimality_condition2_explicit} are a direct application of \Cref{thm:optimality_conditions}.
Applying our definitions of $F$ and $G$, and knowing that $ \partial F(\vt{z}) \!=\! \{\nabla F(\vt{z})\}$ (a singleton as, by assumption, $F$ is differentiable) we have:
\begin{align}
-\lambda\vtheta^\star &= \nabla F(\A\x^\star) \!=\! \left[  f_1'([\A\x^\star]_1),\dots, f_m'([\A\x^\star]_m) \right]^\T \\
\A^\T\vtheta^\star &\in \partial G(\x^\star) = \partial \Omega(\x^\star) + \partial \ind_{\s{C}}(\x^\star) \label{app:eq:GLM_optimality_condition2}
\end{align}



To obtain \eqref{eq:GLM_optimality_condition2_explicit} explicitly,
the sum of subdifferentials in equation \eqref{app:eq:GLM_optimality_condition2} is further developed below.

\subsubsection{Sum of subdifferentials: \texorpdfstring{$\partial \Omega(\mathbf{x}) + \partial \ind_{\s{C}}(\mathbf{x})$}{}} \label{app:subdiff_sum}

Let us further develop the second optimality condition \eqref{app:eq:GLM_optimality_condition2}.
It requires computing  the subdifferentials of both $\Omega$ and $\ind_{\s{C}}$, which are discussed below:

\begin{proposition}[Subdifferential of a Norm] \citep[Proposition 1.2]{Bach2012} The subdifferential of a norm $\Omega: \R^n \rightarrow \R$ at $\x$ is given by
\begin{align}\label{eq:subdiff_norm}
\partial \Omega(\x) = 
\left\{
\begin{array}{lr}
 \{\vt{z} \in \R^n ~|~ \overline{\Omega}(\vt{z}) \leq 1\}, & ~\text{if } ~\x = \vt{0}  \\
 \{\vt{z} \in \R^n ~|~ \overline{\Omega}(\vt{z}) = 1 ~\text{ and }~ \vt{z}^\T\vt{x} = \Omega(\x)\}, & ~\text{otherwise.}
\end{array} \right.
\end{align}
\end{proposition}

\begin{proposition}[Subdifferential of an indicator function] \citep[Example 16.13]{Bauschke2011} 
The subdifferential at point $\x$ of the indicator of a set $\s{C}$ is given by the normal cone of $\s{C}$ at $\x$, denoted $N_{\s{C}}(\x)$:
\begin{align}\label{eq:subdiff_indicator_function}
\partial \ind_{\s{C}}(\x) = N_{\s{C}}(\x) =
\left\{
\begin{array}{lr}
 \{\vt{z} \in \R^n ~|~ \vt{z}^\T\vt{x} \geq \vt{z}^\T\vt{w}, ~\forall \vt{w} \in \s{C} \} & \text{if } \x \in \s{C}  \\
 \emptyset & \text{otherwise.}
\end{array} \right.
\end{align}
In particular, $N_{\s{C}}(\x) = \{\vt{0}\}, \forall \x \in \interior(\s{C})$.
\end{proposition}

When it comes to the constraint sets considered in this paper, we have: 
\begin{align}
&\partial \ind_{\R^n} \equiv \{0\}.\\
&\partial \ind_{R_{+}^n}(\x) 
= J_1 \times \dots \times J_n, 
& \text{with  } & J_j = \left\{
\begin{array}{ll}
\emptyset  & \text{if } \xj < 0  \\
 (-\infty,0] & \text{if } \xj = 0 \\ 
 \{0\} & \text{if } \xj > 0.  \label{eq:subdiff_positive_orthant}
\end{array} \right.
\end{align}



The case $\s{C} = \R^n$ is trivial, since $\partial \ind_{\R^n}(\x) \equiv \{\vt{0}\}$ and therefore $\partial \Omega(\x) + \partial \ind_{\R^n}(\x) = \partial \Omega(\x)$, which is given in equation \eqref{eq:subdiff_norm}.
Let us now analyze the case $\s{C} = \R_{+}^n$. First of all, note that the subdifferential $\partial \ind_{\R_{+}^n}(\x)$ in equation \eqref{eq:subdiff_positive_orthant} can be rewritten in the following form:
\begin{align}\label{eq:subdiff_positive_orthant_app}
\partial \ind_{R_{+}^n}(\x) 
&= \left\{
\begin{array}{lll}
\emptyset  & \text{if } \x \notin \R_{+}^n &  \\
\{\vt{z} \in \R^n ~|~ \vt{z} \leq \vt{0} \} = \R_{-}^n & \text{if } \x = \vt{0} & \\ 
 \{\vt{z} \in \R^n ~|~ \vt{z} \leq \vt{0} \text{ and } \vt{z}^\T\x = 0 \} & \text{otherwise} & \text{(i.e. } \x \geq \vt{0} \text{ and } \x \neq \vt{0} \text{).}
\end{array} \right.
\end{align}
Therefore, the (Minkowski) sum of the subdifferential sets in equations \eqref{eq:subdiff_norm} and \eqref{eq:subdiff_positive_orthant_app} gives:
\begin{align}
\partial \Omega(\x) + \partial \ind_{\R_{+}^n}(\x) 
%
&= \left\{
\begin{array}{lll}
 \emptyset &  ~\text{if } ~\x \notin \R_{+}^n &\\
 \{\vt{z} = \vt{z}_1 + \vt{z}_2 ~|~ \overline{\Omega}(\vt{z}_1) \leq 1, ~ \vt{z}_2 \leq 0\}, & ~\text{if } ~\x = \vt{0}  &\\
 \{\vt{z} = \vt{z}_1 + \vt{z}_2 ~|~ \overline{\Omega}(\vt{z}_1) = 1, ~\vt{z}_1^\T\vt{x} = \Omega(\x), \vt{z}_2 \leq 0 , ~\vt{z}_2^\T\x = 0\}, & ~\text{otherwise} & 
\end{array} \right. \nonumber \\ 
%
%
&= \left\{
\begin{array}{lll}
 \emptyset &  ~\text{if } ~\x \notin \R_{+}^n &\\
 \{\vt{z} \in \R^n ~|~ \overline{\Omega}([\vt{z}]^{+}) \leq 1\}, & ~\text{if } ~\x = \vt{0}  &\\
 \{\vt{z} \in \R^n ~|~ \overline{\Omega}([\vt{z}]^{+}) = 1, ~\vt{z}^\T\vt{x} = \Omega(\x)\}, & ~\text{otherwise} & 
\end{array} \right. \label{eq:subdiff_sum_nonnegative}
\end{align}
where, in the last equality, we have used the fact that $\vt{z}_1^\T \x = (\vt{z} - \vt{z}_2)^\T \x  = \vt{z}^\T \x $, since $\vt{z}_2^\T \x = 0$.


Comparing equations \eqref{eq:subdiff_norm} and \eqref{eq:subdiff_sum_nonnegative} respectively for $\s{C}=\R^n$ and $\s{C}=\R_{+}^n$,
we can see that the sum 
$\partial \Omega + \partial \ind_{\s{C}}$
for the considered constraint sets 
writes compactly as:
\begin{align}\label{eq:subdiff_sum_constrained}
\partial \Omega(\x) + \partial \ind_{\s{C}}(\x) = 
\left\{
\begin{array}{lrl}
 \emptyset &  ~\text{if } ~\x \notin \s{C} &\\
 \{\vt{z} \in \R^n ~|~ \overline{\Omega}(\phi(\vt{z})) \leq 1\}, & ~\text{if } ~\x = \vt{0}  &\\
 \{\vt{z} \in \R^n ~|~ \overline{\Omega}(\phi(\vt{z})) = 1, ~\vt{z}^\T\vt{x} = \Omega(\x)\}, & ~\text{otherwise} & 
\end{array} \right.
\end{align}
where $\phi$ is defined as in \eqref{eq:phi_GLM_dual}. 

Applying the above results to equation \eqref{app:eq:GLM_optimality_condition2}, it takes the following form which corresponds to \eqref{eq:GLM_optimality_condition2_explicit}
\begin{align}
\left\{
\begin{array}{lrl}
 \overline{\Omega}(\phi(\A^\T \vtheta^\star)) \leq 1, & ~\text{if } ~\x^\star = \vt{0}  &\\
 \overline{\Omega}(\phi(\A^\T \vtheta^\star)) = 1,  ~~ (\A^\T \vtheta^\star)^\T\x^\star = \Omega(\x^\star), & ~\text{otherwise} & 
\end{array} \right.
\end{align}

\section{Proof of Theorem~\ref{prop:local_concavity}} \label{app:local_GAP}

In this section we prove Theorem \ref{prop:local_concavity}, i.e. 
given any feasible primal-dual pair $(\x,\vtheta) \in {(\dom (P_\lambda) \cap \s{C})} \times (\Delta_A \cap \s{S})$ and supposing $D_\lambda$ to be $\alpha_{\s{S}}$-strongly concave on $\s{S}$, we prove that a region
\begin{align}
\s{B}(\vtheta,r), \quad \text{with }  r = \sqrt{\frac{2\Gap_\lambda({\x},{\vtheta})}{\alpha_\s{S}}}
\end{align}
is safe, i.e. $\vtheta^\star \in \s{B}(\vtheta,r)$. 
Equivalently, we prove that
$\|{\vtheta} - \vtheta^\star  \|_2 \leq r.$  
 
The proof follows quite similarly to \citet[Theorem 6]{Ndiaye2017}.
(except that  we assume strong concavity of $D_\lambda$ exclusively on $\s{S}$ instead of $\R^m$).
From the $\alpha_\s{S}$-strong concavity of $D_\lambda$ on $\s{S}$ we have:
\begin{align*}
\forall (\vtheta^\prime,\vtheta'') \in \s{S} \times \s{S}, \quad
D_\lambda(\vtheta'') \leq D_\lambda(\vtheta^\prime) + \langle \nabla D_\lambda(\vtheta^\prime), \vtheta'' - \vtheta^\prime \rangle - \frac{\alpha_{\s{S}}}{2}\|\vtheta'' - \vtheta'\|_2^2 .
\end{align*}
In particular, we can take  $\vtheta^\prime = \vtheta^\star$, $\vtheta'' = {\vtheta}$ (since we suppose  ${\vtheta} \in \s{S}$ and $\vtheta^\star \in \s{S}$):
\begin{align*}
D_\lambda({\vtheta}) \leq D_\lambda(\vtheta^\star) + \langle \nabla D_\lambda(\vtheta^\star), {\vtheta} - \vtheta^\star \rangle - \frac{\alpha_{\s{S}}}{2}\|\vtheta - {\vtheta}^\star\|_2^2 .
\end{align*}
By definition $\vtheta^\star$  maximizes $D_\lambda$ on $\Delta_\A$ and hence $\langle \nabla D_\lambda(\vtheta^\star), {\vtheta} - \vtheta^\star \rangle \leq 0$ (otherwise, ${\vtheta} - \vtheta^\star$ would be a feasible direction of improvement), which implies:
\begin{align*}
D_\lambda({\vtheta}) \leq D_\lambda(\vtheta^\star) - \frac{\alpha_{\s{S}}}{2}\|\vtheta - {\vtheta}^\star\|_2^2 .
\end{align*}
By weak duality $D_\lambda(\vtheta^\star) \leq P_\lambda(\x')$, $\forall \x' \in  \dom(P_\lambda) \cap \s{C}$,  and in particular for $\x'={\x}$, giving:
\begin{align*}
D_\lambda({\vtheta})  &\leq P_\lambda({\x}) - \frac{\alpha_{\s{S}}}{2}\|\vtheta - \vtheta^\star\|_2^2 \\
\underbrace{D_\lambda({\vtheta}) - P_\lambda({\x})}_{-G_\lambda({\x},{\vtheta})} &\leq  - \frac{\alpha_{\s{S}}}{2}\|\vtheta - {\vtheta}^\star\|_2^2 \\
\frac{2}{\alpha_{\s{S}}}G_\lambda({\x},{\vtheta}) &\geq \|\vtheta - {\vtheta}^\star\|_2^2
\end{align*} 
which holds for all ${\x} \in \dom(P_\lambda) \cap \s{C}, {\vtheta} \in \Delta_\A \cap \s{S}$, concluding the proof.


\section{Proofs of Section~\ref{ssec:main_ingredients}} \label{app:main_ingredients}

\subsection{Proof of Proposition~\ref{prop:max_regularization}}
\label{app:prop:max_regularization}

\textit{Statement.~}
\begin{align}
    \vt{0} \in \argmin_{\x \in \s{C}} P_\lambda(\x) \iff \lambda \geq \lambda_{\max} := \overline{\Omega}\left(\phi\left(-\A^\T\nabla F(\vt{0})\right)\right)
\end{align}
\begin{proof}
Writing our primal problem in an equivalent unconstrained form, we have
$$\argmin_{\x \in \R^n}  F(\A\x) + \lambda \left( \Omega(\x) +  \ind_{\s{C}}(\x) \right) := P^{\unc}_\lambda (\x)$$
denoting $P^{\unc}_\lambda$ the unconstrained objective function, with $\argmin_{\x \in \R^n} P^{\unc}_\lambda (\x) \!=\! \argmin_{\x \in \s{C}} P_\lambda (\x)$.
From Fermat's rule we have that 
\begin{align*}
  \vt{0} \in \argmin_{\x \in \R^n} P^{\unc}_\lambda(\x) & \iff  \vt{0} \in \partial P^{\unc}_\lambda(\vt{0}) = \{\A^\T \nabla F(\vt{0})\} + \lambda \left( \partial \Omega(\vt{0}) + \partial \ind_{\s{C}}(\vt{0}) \right) \\ &\iff -\A^\T\nabla F(\vt{0})/\lambda \in  \partial \Omega(\vt{0}) + \partial \ind_{\s{C}}(\vt{0}).  
\end{align*}
Using the expression of $\partial \Omega(\x) + \partial \ind_{\s{C}}(\x)$ given in \eqref{eq:subdiff_sum_constrained} we obtain
\begin{align*}
    \vt{0} \in  \argmin_{\x \in \R^n} P^{\unc}_\lambda(\x) \iff \overline{\Omega}\left(\phi(-\A^\T\nabla F(\vt{0}))\right) \leq \lambda,
\end{align*}
which completes the proof.
\end{proof}

\subsection{Proof of Lemma~\ref{lem:dual_scaling}} \label{app:lem:dual_scaling}

\textit{Statement.~}
Assume that $\dom(D_\lambda)$ is stable by contraction and let $\dscale : \R^m \rightarrow \R^m$ be the scaling operator defined by
\begin{equation} 
    \dscale(\vt{z}) := \frac{\vt{z}}{\max\left(\overline{\Omega}(\phi(\A^\T \vt{z})),1\right)}.
\end{equation}
Then, for any point $\vt{z} \in \dom(D_\lambda)$, we have $\dscale(\vt{z}) \in \Delta_\A$.
Moreover, for any primal point $\x \in \dom(P_\lambda)$, we have that $\vt{z} = (-\nabla F(\A\x)/\lambda) \in \dom(D_\lambda)$ and therefore 
\begin{align}
    \dscale(-\nabla F(\A\x)/\lambda) \in \Delta_\A.
\end{align}
%
Finally, if $F \in C^1$, then $\dscale(-\nabla F(\A\x)/\lambda) \rightarrow \vtheta^\star$ as $\x \rightarrow \x^\star$.
\begin{proof}
By definition of $\dscale$, one can see that $\dscale(\vt{z}) \leq \vt{z}$ and, given the assumption that  $\dom(D_\lambda)$ is stable by contraction, we get that $\forall \vt{z} \in  \dom(D_\lambda)$, $\dscale(\vt{z}) \in \dom(D_\lambda)$. Combining this with the fact that $\overline{\Omega}\left(\phi(\A^\T \dscale(\vt{z}))\right)  \leq 1$ (by definition of $\dscale$), we obtain
\begin{equation}
    \forall \vt{z} \in  \dom(D_\lambda), \; \dscale(\vt{z}) \in \Delta_\A,
\end{equation}
where we recall that $\Delta_{\A} = \{\vtheta \in \R^m ~|~  \overline{\Omega}(\phi(\A^\T \vtheta)) \leq 1\} \cap \dom(D_\lambda)$.

To complete the proof, it remains to show that, for any primal point $\x \in \dom(P_\lambda)$, $(-\nabla F(\A\x)/\lambda) \in \dom(D_\lambda)$. Because $D_\lambda = - F^*(-\lambda \cdot)$, we get that $\vtheta \in \dom(D_\lambda) \, \Longleftrightarrow \, -\lambda \vtheta \in \dom(F^*)$. Now, let $\x \in \dom (P_\lambda)$, then
\begin{equation}
    F^*(\nabla F(\A\x)) = \sup_{\vt{z}\in \R^m} \left<\vt{z},\nabla F(\A\x) \right> - F(\vt{z}) = \left<\A\x,\nabla F(\A\x) \right> - F(\A\x) < \infty.
\end{equation}
(The sup is attained for vector(s)  $\vt{z}\in \R^m$ such that $\nabla F(\A\x) - \nabla F(\vt{z}) = 0$, such as $\vt{z} = \A\x$.) This shows that $\nabla F(\A\x) \in \dom(F^*)$ and thus that $(-\nabla F(\A\x)/\lambda) \in \dom(D_\lambda)$.

Lastly, assuming that $F \in C^1$, we get by continuity of $\nabla F(\A\x)$ that $\nabla F(\A\x) \rightarrow \nabla F(\A\x^\star)$ as $\x \rightarrow \x^\star$ and, from the primal-dual link in optimality condition \eqref{eq:GLM_optimality_condition1} we have that $-\nabla F(\A\x^\star) = \lambda \vtheta^\star$. 
%
Then, because $\vtheta^\star \in \Delta_\A$,  the scaling factor in $\dscale$ is equal to 1 (by definition) which leads to $\dscale(-\nabla F(\A\x^\star)/\lambda) = -\nabla F(\A\x^\star)/\lambda = \vtheta^\star$. 
%
Finally, by continuity of $\dscale$ (which is a composition of continuous functions), we conclude that $\dscale(-\nabla F(\A\x)/\lambda) \rightarrow \vtheta^\star$ as $\x \rightarrow \x^\star$.
%
%
\end{proof}

\subsection{Proof of Proposition~\ref{prop:strong_concavity}} \label{app:prop:strong_concavity}

\textit{Statement.~}
Assume that $D_\lambda$ given in Theorem~\ref{thm:dual_GLM_and_optimality} is twice differentiable. Let $\s{S} \in \R^m$ be a convex  set  and $\s{I} = \{ i \in [m]: \forall (\vtheta,\vtheta') \in \s{S}^2, \theta_i = \theta_i'\}$ 
a (potentially empty) set of coordinates in which $\s{S}$ reduces to a singleton.
Then, $D_\lambda$ is $\alpha_{\s{S}}$-strongly concave on $\s{S}$ if and only if
\begin{align}
0 < \alpha_{\s{S}} \leq 
\min_{i \in \s{I}^\complement} \;- \sup_{\vtheta \in \s{S}} \;  \sigma_i(\theta_i),
\end{align}
where $\sigma_i(\theta_i) = -\lambda^2 (f_i^*)''(-\lambda\vthetai)$ is the (negative) i-th eigenvalue of the Hessian matrix $\nabla^2  D_\lambda(\vtheta)$.
\begin{proof}
First of all, let us recall that, from Theorem~\ref{thm:dual_GLM_and_optimality}, $D_\lambda(\vtheta)= -\sum_i f_i^*(-\lambda\vthetai)$.
Then by definition of strong concavity, $D_\lambda$ is $\alpha$-strongly concave on a set $\s{S} \subset \R^m$ if and only if, $\forall (\vtheta^\prime,\vtheta) \in \s{S}^2$,
\begin{align}
& \; D_\lambda(\vtheta) \leq D_\lambda(\vtheta^\prime) + \langle \nabla D_\lambda(\vtheta^\prime), \vtheta - \vtheta^\prime \rangle - \frac{\alpha^2}{2}\|\vtheta - \vtheta^\prime\|_2^2 \\
\Longleftrightarrow & \; D_\lambda |_{\s{I}^\complement} (\vtheta_{\s{I}^\complement})  \leq D_\lambda |_{\s{I}^\complement} (\vtheta'_{\s{I}^\complement}) +\langle \nabla D_\lambda |_{\s{I}^\complement} (\vtheta'_{\s{I}^\complement}), \vtheta_{\s{I}^\complement} - \vtheta^\prime_{\s{I}^\complement}  \rangle -  \frac{\alpha^2}{2} \|\vtheta_{\s{I}^\complement} - \vtheta'_{\s{I}^\complement}\|_2^2
\end{align}
where, for $\vt{u} \in \R^{|\s{I}^\complement|}$, $D_\lambda |_{\s{I}^\complement} (\vt{u}) = - \sum_{i \in \s{I}^\complement}  f_i^*(-\lambda u_i)$ is the restriction of $D_\lambda$ to the coordinates that belong to $\s{I}^\complement$. The second inequality has been obtained from the definition of $\s{I}$ which implies that $\forall (\vtheta^\prime,\vtheta) \in \s{S}^2$, $\vthetai = \vthetai'$ for all $i \in \s{I}$. Hence, we have that  $D_\lambda$ is $\alpha$-strongly concave on $\s{S}$ if and only if its restriction $D_\lambda |_{\s{I}^\complement}$ is $\alpha$-strongly concave on $\s{S}|_{\s{I}^\complement} = \{ \vt{u} \in \R^{|\s{I}^\complement|} : \exists \vtheta \in \s{S} \text{ such that } \vtheta_{\s{I}^\complement} = \vt{u}\}$. 
From  \citet[Chapter IV, Theorem 4.3.1 and Remark 4.3.2]{Hiriart-Urruty1993}, as $\mathrm{int}(\s{S} |_{\s{I}^\complement}) \neq \emptyset$, the latter is equivalent to
\begin{equation}
\alpha \leq \min_{i \in \{1,\ldots,|\s{I}^\complement|\}} \;- \sup_{\vt{u} \in \s{S}|_{\s{I}^\complement}} \;  \sigma_i(\nabla^2 D_\lambda |_{\s{I}^\complement}(\vt{u}))
\end{equation}
where $ \sigma_i(\nabla^2 D_\lambda |_{\s{I}^\complement}(\vt{u}))$ denotes the i-th eigenvalue of the Hessian matrix $(\nabla^2 D_\lambda |_{\s{I}^\complement}(\vt{u})) \in \R^{|\s{I}^\complement|\times |\s{I}^\complement|}$. By definition of $D_\lambda |_{\s{I}^\complement}$, we have that for $\vtheta \in \s{S}$,  $\nabla^2 D_\lambda |_{\s{I}^\complement}(\vtheta_{\s{I}^\complement}) = [\nabla^2 D_\lambda ]_{\s{I}^\complement,\s{I}^\complement}(\vtheta)$
where 
\begin{equation}
    \nabla^2  D_\lambda(\vtheta) = - \lambda^2 \diag{ (f_1^*)''(-\lambda\theta_1), \dots,  (f_m^*)''(-\lambda\theta_m) }.
\end{equation}
Then, it follows that 
\begin{equation}
\alpha \leq \min_{i \in \s{I}^\complement} \;- \sup_{\vtheta \in \s{S}} \;  \sigma_i(\vthetai),
\end{equation}
where $\sigma_i(\vthetai) = -\lambda^2 (f_i^*)''(-\lambda\theta_i)$. This completes the proof.
%
\end{proof}

\section{Particular Cases of Section~\ref{sec:Particular_cases}: Calculation Details}\label{apdx:Particular_cases}

In this appendix, we provide details on the derivation of the quantities that are summarized in Table~\ref{tab:data-fidelity_cases}.

\subsection{Quadratic Distance (\texorpdfstring{$\beta$}{beta}-Divergence with \texorpdfstring{$\beta =2$}{beta=2})}
\label{ssec:example_Euc}

This corresponds to the standard Lasso problem (when combined with an $\ell_1$-norm regularization) and its group-variants.
The data-fidelity term in this case is:
\begin{align}
F(\A\x) = \frac{1}{2}\|\y -  \A\x\|_2^2 && f_i(\Axi) = \frac{1}{2}(\yi -  \Axi)^2
\end{align}
%
with gradient given by
\begin{align}
\nabla F(\A\x) = \A\x - \y  &&  f'_i(\Axi) = \Axi - \yi.
\end{align}

We then deduce from Theorem~\ref{thm:dual_GLM_and_optimality}
that the first-order optimality condition \eqref{eq:GLM_optimality_condition1} (primal-dual link) is given by:
\begin{align}
\lambda \vtheta^\star = \y -  \A\x^\star   
&&
\lambda \vthetai^\star =  \yi - \Axi^\star  , ~~ \forall i \in [m] 
\end{align}

The maximum regularization parameter $\lambda_{\max}$ is obtained by substituting $\nabla F(\vt{0}) = -\y$ in \eqref{eq:max_regularization}:
\begin{align}
    \lambda_{\max} = \|\A^\T \y\|_{\infty}
\end{align}

The dual function $D_\lambda(\vtheta) = -\sum_{i=1}^m f^*_i(-\lambda \vthetai)$ is given by:
\begin{align}
D_\lambda(\vtheta) =  \frac{\|\y\|_2^2 - \|\y  - \lambda \vtheta\|_2^2  }{2}
= \sum_{i=1}^m  \frac{\yi^2 - (\yi  - \lambda \vthetai)^2  }{2}.
\end{align}
with $\dom (D_\lambda) = \R^m$.
%
Then, we get from Theorem~\ref{thm:dual_GLM_and_optimality}, $\s{C} = \R^n$, and~\eqref{eq:DualL1} (dual norm of the $\ell_1$-norm)  that
the dual feasible set is given by:
\begin{align}
  \Delta_\A = \{\vtheta \in \R^m ~|~ \|\A^\T\vtheta\|_\infty  \leq 1 \}.
\end{align}

The Hessian $\nabla^2 D_\lambda(\vtheta)$  and corresponding eigenvalues $\sigma_i(\vthetai)$ are given by
\begin{align}
    \nabla^2 D_\lambda(\vtheta) = \diag{[\sigma_i(\vthetai)]_{i \in [m]}},
    &&
    \sigma_i\left( \vthetai \right)  = -\lambda^2
\end{align}
Hence, from \Cref{prop:strong_concavity} we get that
$D_\lambda$ is strongly concave on $\R^m$ with global constant $\alpha=\lambda^2$.

The fact that the dual function 
is quadratic implies that its second derivative is a constant.
Hence, the local strong concavity bound on any subset of $\R^m$ is also $\lambda^2$.
The proposed local approach thus reduces to the standard Gap Safe screening in this particular case.

\paragraph{Dual Update.}
The residual w.r.t. a primal estimate $\x$ is given by
$\vt{\rho}(\x) = -  \nabla F(\A\x) =  \y - \A\x$.
%
Because $\SO = \R^m$, given any primal feasible point $\x \in \R^n \left(= \s{C} \cap \dom (P_\lambda)\right)$, we get from Lemma~\ref{lem:dual_scaling} that 
\begin{align}\label{eq:Quadratic_dual_point}
\vt{\Theta}(\x) = \dscale\left( (\y - \A\x)/\lambda \right).
\end{align}
is such that $\vt{\Theta}(\x) \in \Delta_\A$. Moreover, $\vt{\Theta}(\x) \rightarrow \vtheta^\star$ as $\x \rightarrow \x^\star$.


\subsection{Logistic Regression} \label{ssec:example_LogReg}

We consider the two-class logistic regression as formulated in \citet[Chapter 3]{Buhlmann2011}.
The observation vector entries $\yi \in \{0,1\}$ are binary class labels.


The data-fidelity term in this case is:
\begin{align}\label{eq:LogReg_data_fidelity}
F(\A\x) = \vt{1}^\T\log\left(1 + \e^{\A\x}\right) - \y^\T\A\x 
&& f_i(\Axi) = \log\left(1 + \e^{\Axi}\right) - \yi\Axi 
\end{align}
and $\s{C} = \R^n$.
%
The gradient of $F$ is given by
\begin{align}
\nabla F(\mathbf{z}) = \frac{\e^{\mathbf{z}}}{1+\e^{\mathbf{z}}} - \y
&&  f'_i(z_i) = \frac{\e^{z_i}}{1+\e^{z_i}} - \yi
\end{align}
%
We then deduce from Theorem~\ref{thm:dual_GLM_and_optimality} that the first-order optimality condition \eqref{eq:GLM_optimality_condition1} (primal-dual link) is given by:
\begin{align}
\lambda \vtheta^\star = \y - \frac{\e^{\A\x^\star}}{1+\e^{\A\x^\star}} 
&&
\lambda \vthetai^\star = \yi - \frac{\e^{[\A\x^\star]_i}}{1+\e^{[\A\x^\star]_i}} , ~~ \forall i \in [m] 
\end{align}

The maximum regularization parameter $\lambda_{\max}$ is obtained by substituting $\nabla F(\vt{0}) = \frac{1}{2}-\y$ in \eqref{eq:max_regularization}:
\begin{align}
    \lambda_{\max} = \left\lVert \A^\T \left(\y-1/2\right)\right\rVert_\infty
\end{align}

\begin{proposition}\citep{Ndiaye2017}
The Fenchel conjugate of the logistic regression data term in equation \eqref{eq:LogReg_data_fidelity} is given by $F^*(\vt{u}) = \sum_{i=1}^m f_i^*(u_i)$ where
\begin{align}
f_i^*(u) = (y_i+u)\log(y_i+u) + (1-y_i-u)\log(1-y_i-u) 
\end{align}
with $\dom(f_i^*) = [-y_i,1-y_i]$. 
\end{proposition}
%

The dual function $D_\lambda(\vtheta) = -\sum_{i=1}^m f^*_i(-\lambda \vthetai)$ is given by
\begin{align} \label{eq:LogReg_dual_function}
D_\lambda(\vtheta) 
= -(\y -\lambda \vtheta)^\T \log(\y -\lambda \vtheta) - (1-\y + \lambda \vtheta)^\T\log(1-\y + \lambda \vtheta)
\end{align}
with $\dom(D_\lambda) = \{\vtheta ~|~ \y-1 \leq \lambda\vtheta \leq \y\}$. It is also known as the (binary) entropy function.
Then, we get from Theorem~\ref{thm:dual_GLM_and_optimality}, $\s{C} = \R^n$, and~\eqref{eq:DualL1} (dual norm of the $\ell_1$-norm)  that 
\begin{align}
  \Delta_\A = \{\vtheta \in \R^m ~|~ \|\A^\T\vtheta\|_\infty  \leq 1, \y-1 \leq \lambda\vtheta \leq \y \}
\end{align}

The Hessian $\nabla^2 D_\lambda(\vtheta)$  and corresponding eigenvalues $\sigma_i(\vthetai)$ are given by
\begin{align}
    \nabla^2 D_\lambda(\vtheta) = \diag{[\sigma_i(\vthetai)]_{i \in [m]}},
    &&
    \sigma_i\left( \vthetai \right)  &= - \frac{\lambda^2}{(\yi -\lambda \vthetai)(1 -\yi + \lambda \vthetai)} \\
    &&&=-\frac{4\lambda^2}{1 -4(\lambda\vthetai - \yi + \frac{1}{2})^2} 
\end{align}

The eigenvalues $\sigma_i(\vthetai)$ 
are all negative with maximum value $-4\lambda^2$ attained at $\lambda \vthetai = (2\yi -1)/2$.
Therefore, the dual function is strongly concave on $\dom(D_\lambda)$ with a global constant 
$$\alpha = 4\lambda^2.$$

Let us now examine the dual function restricted to some particular subsets $\s{S}$ of the domain $\dom (D_\lambda)$.

\subsubsection{Strong-Concavity Bound on \texorpdfstring{$\Delta_\A \cap \SO$}{the Feasible Set}}

In this section we evaluate the strong concavity of $D_\lambda$ on the set $\s{S}=\Delta_\A$ or equivalently $\s{S}=\Delta_\A \cap \SO$ with $\SO = \R^m$.

\begin{proposition} \label{prop:LogReg_alpha_fixed}
Assuming that $\rank(\A)=\min(m,n)$,
the dual function $D_\lambda$  as defined in \eqref{eq:LogReg_dual_function} is $\alpha_{\Delta_\A}$ strongly concave on $\Delta_\A$ with constant:
\begin{align}\label{eq:LogReg_alpha_fixed}
\alpha_{\Delta_\A} 
&= \frac{4\lambda^2}{1 - 4 \left(\min(\lambda\|\A^\dagger\|_1,\frac{1}{2}) - \frac{1}{2} \right)^2} 
\end{align}
where $\A^\dagger$ denotes the right pseudo-inverse of $\A$ such that $\A\A^\dagger = \Id$,
and $\|\A^\dagger\|_1$ is the maximum absolute column sum of $\A^\dagger$.
\\ Moreover, the local bound $\alpha_{\Delta_\A}$ in \eqref{eq:LogReg_alpha_fixed} improves upon the global bound $\alpha = 4\lambda^2$ for all $\lambda < \frac{1}{2\|\A^\dagger\|_1}.$
\end{proposition}
\begin{proof}
From \Cref{prop:strong_concavity} (with $\s{I} = \emptyset$) we have to prove that
 \begin{equation}
     \alpha_{\Delta_\A} \leq \min_{i \in [m]} - \sup_{\vtheta \in  \Delta_\A} \sigma_i(\vthetai).
 \end{equation}
where $\sigma_i(\vthetai) =  - \frac{\lambda^2}{(\yi -\lambda \vthetai)(1 -\yi + \lambda \vthetai)}$ is the $i$-th eigenvalue of $\nabla^2 D_\lambda (\vtheta)$.
Given that $\yi \in \{0,1\}$, we can compactly rewrite $\sigma_i(\vthetai)$ as follows:
\begin{align}
     \sigma_i(\vthetai) &= \left\lbrace
    \begin{array}{cc}
         - \frac{\lambda^2}{(1 -\lambda \vthetai)(\lambda \vthetai)} \quad \forall \lambda\vthetai \in (0,1) &\text{  if } \yi = 1\\
         - \frac{\lambda^2}{(-\lambda \vthetai)(1+\lambda \vthetai)}    \quad \forall \lambda\vthetai \in (-1,0) &\text{ if } \yi = 0
    \end{array}
    \right. \\
    & =  - \frac{\lambda^2}{\lambda |\vthetai|(1 -\lambda |\vthetai|)}  \quad \forall |\lambda\vthetai| \in (0,1)
\end{align}

First, note that if $\lambda \|\vtheta\|_\infty \geq 1/2$ then the global maximum of $\sigma_i(\vthetai)$ is attained for some~$i$.
Indeed, the global maximum is attained for $\lambda|\vthetai| = 1/2$:
\begin{align*}
    \lambda\vthetai = \yi - \frac{1}{2} = \left\lbrace
    \begin{array}{cc}
         1/2  & \text{if } \yi = 1\\
         -1/2 & \text{if } \yi = 0
    \end{array}
    \right.
\end{align*}

For $\lambda|\vthetai| \leq 1/2$, one can see that $\sigma_i(\vthetai)$ is increasing with $|\vthetai|$, which implies that
\begin{align}\label{eq:logistic_max_eigenvalue_infty}
    \min_{i \in [m]} - \sup_{\|\vtheta\|_\infty \leq a} \sigma_i(\vthetai)
    = \frac{\lambda^2}{\min(\lambda a,\frac{1}{2})(1 - \min(\lambda a,\frac{1}{2}))} 
    = \left\lbrace
    \begin{array}{cc}
         \frac{\lambda}{a(1 - \lambda a)} & \text{if } \lambda a <  \frac{1}{2}\\
         4 \lambda^2 & \text{otherwise} 
    \end{array}\right.
\end{align}

Now, we can bound $\|\vtheta\|_\infty$ on $\Delta_\A$ as follows, knowing that $\|\A^\T \vtheta\|_\infty \leq 1$ and $\rank(\A)=\min(m,n)$:
\begin{align*}
\|\vtheta\|_\infty &= \max_{i \in [m]} |\vthetai|
= \max_{i \in [m]}  \left|\left[ (\A \A^\dagger)^\T\vtheta \right]_i \right|
= \max_{i \in [m]} |\langle(\A^\dagger)_i, \A^\T\vtheta \rangle| \\
&\leq \max_{i \in [m]} \|(\A^\dagger)_i\|_1 \|\A^\T\vtheta\|_\infty
= \|\A^\dagger\|_1 \|\A^\T\vtheta\|_\infty \leq \|\A^\dagger\|_1
\end{align*}
Hence, we have that $\Delta_\A \subseteq \{\vtheta \in \R^m  ~|~  \|\vtheta \|_\infty \leq  \|\A^\dagger\|_1 \}$ and 
\begin{align*}
     \min_{i\in [m]} - \sup_{\vtheta \in \Delta_\A} \sigma_i(\vthetai) \geq  \min_{i\in [m]} - \sup_{\|\vtheta\|_\infty \leq \|\A^\dagger\|_1} \sigma_i(\vthetai) 
    &= \frac{\lambda^2}{\min(\lambda\|\A^\dagger\|_1,\frac{1}{2})(1 - \min(\lambda\|\A^\dagger\|_1,\frac{1}{2}))}\\
    &= \frac{4\lambda^2}{1 - 4 \left(\min(\lambda\|\A^\dagger\|_1,\frac{1}{2}) - \frac{1}{2} \right)^2}    \\
    & =  \alpha_{\Delta_\A},
\end{align*}
where the first equality is obtained by taking $a = \|\A^\dagger\|_1$ in  equation \eqref{eq:logistic_max_eigenvalue_infty}
Finally, one can easily verify that $\alpha_{\Delta_\A} > 4 \lambda^2 \iff \lambda < \frac{1}{2\|\A^\dagger\|_1}$.
\end{proof}


\subsubsection{Strong-Concavity Bound on \texorpdfstring{$\s{B}({\vtheta},r) \cap \SO$}{a Ball}}


\begin{proposition}\label{prop:LogReg_alpha_adaptive}
For any ball $\s{B}(\vtheta,r)$ with $ \vtheta \in \dom(D_\lambda)$, the dual function $D_\lambda$  as defined in \eqref{eq:LogReg_dual_function} is $\alpha_{\s{B}(\vtheta,r)}$ strongly concave on $\s{B}(\vtheta,r)$ with constant:
\begin{align}\label{eq:LogReg_alpha_adaptive}
  \alpha_{\s{B}(\vtheta,r)} 
  &= \frac{4 \lambda^2}{1 - 4 ([\min_i(|\lambda \vthetai - \yi + \frac{1}{2}|) - \lambda r]^+)^2 }
  %
  %
\end{align}
\end{proposition}
\begin{proof}
%
%
From \Cref{prop:strong_concavity} (with $\s{I} = \emptyset$) we have to prove that
 \begin{equation}
     \alpha_{\s{B}(\vtheta,r)} \leq \min_{i \in [m]} - \sup_{\vtheta' \in \s{B}(\vtheta,r)} \sigma_i(\vthetai').
 \end{equation}
where the $i$-th eigenvalue of $\nabla^2 D_\lambda (\vtheta)$ is given by
\begin{align*}
     \sigma_i(\vthetai)  
     = - \frac{\lambda^2}{ (\yi - \lambda \vthetai)(1 - \yi + \lambda \vthetai)}  
     = - \frac{4\lambda^2}{1 -4(\lambda\vthetai - \yi + \frac{1}{2})^2}, \quad \lambda \vthetai \in (\yi-1,\yi)
\end{align*}
The maximum value $\sup \sigma_i(\vthetai) = 4\lambda^2$ attained at $\lambda \vthetai = \yi -  \frac{1}{2} := a^\star$. 
Also note that $\sigma_i(\vthetai)$ symmetric w.r.t. $a^\star/\lambda$ and a decreasing function w.r.t. $|\vthetai - a^\star/\lambda|$ (i.e. $\sigma_i(\vthetai)$ only decreases as $\lambda\vthetai$ gets further away from $a^\star$). 

Now 
we evaluate $\sup_{ \vtheta' \in \s{B}(\vtheta,r) } \sigma_i(\vthetai') = \sup_{ |\vthetai - \vthetai' | \leq r } \sigma_i(\vthetai')$ for the i-th eigenvalue. 
If  $| \vthetai - a^\star/\lambda | < r$, the global maximum is attained as $a^\star/\lambda$ lies inside the interval.
Otherwise, the maximum lies on the border of the interval which is closest to $a^\star/\lambda$, 
i.e. $\vthetai' = \vthetai - r$ if $\vthetai > a^\star/\lambda + r$ and $\vthetai' = \vthetai + r$ if $\vthetai < a^\star/\lambda - r$.
\begin{align*}
    \sup_{ |\vthetai - \vthetai' | \leq r } \sigma_i(\vthetai') &= \left\lbrace
    \begin{array}{ll}
        - 4\lambda^2 &  \text{if } -r < \vthetai - a^\star/\lambda < r\\
        - \frac{4\lambda^2}{1 -4(\lambda(\vthetai-r) - a^\star)^2} & \text{if } \vthetai - a^\star/\lambda \geq r \\
        - \frac{4\lambda^2}{1 -4(\lambda(\vthetai+r) - a^\star)^2} & \text{if }\vthetai - a^\star/\lambda \leq - r \\
    \end{array} \right. \\
    & = \left\lbrace
    \begin{array}{ll}
        - 4\lambda^2 &  \text{if } |\lambda\vthetai - a^\star| < \lambda r\\
        - \frac{4\lambda^2}{1 -4(|\lambda\vthetai - a^\star| - \lambda r)^2} & \text{if } |\lambda\vthetai - a^\star| \geq \lambda r \\
    \end{array} \right.
\end{align*}
Finally, $\alpha_{B(\vtheta,r)}$ is obtained by 
\begin{align*}
    \alpha_{B(\vtheta,r)} & \leq \min_{i \in [m]} -  \sup_{ |\vthetai - \vthetai' | \leq r } \sigma_i(\vthetai') \\ 
    & = \left\lbrace
    \begin{array}{ll}
        4\lambda^2 &  \text{if } \min_i (|\lambda\vthetai - a^\star|) < \lambda r\\
        \frac{4\lambda^2}{1 -4(\min_i |\lambda \vthetai - a^\star| - \lambda r)^2} & \text{if } \min_i (|\lambda\vthetai - a^\star|) \geq \lambda r \\
    \end{array} \right. \\
    & = \left\lbrace
    \begin{array}{ll}
       4\lambda^2 &  \text{if } \min_i (|\lambda  \vthetai - \yi + \frac{1}{2} |) \leq \lambda r \\
       \frac{4 \lambda^2}{1 - 4( \min_i(|\lambda \vthetai - \yi + \frac{1}{2}|) - \lambda r)^2 } & \text{otherwise.} 
    \end{array} \right. \\
    &= \frac{4 \lambda^2}{1 - 4 ([\min_i(|\lambda \vthetai - \yi + \frac{1}{2}|) - \lambda r]^+)^2 }
\end{align*}
This completes the proof.
\end{proof}

\subsubsection{Dual Update}
The generalized residual w.r.t. a primal estimate $\x$ is given by
$$\vt{\rho}(\x) := -  \nabla F(\A\x) =  \y - \frac{\e^{\A\x}}{1+\e^{\A\x}}.$$ 
%
Because $\SO = \R^m$, given any primal feasible point $\x \in \s{C} \cap \dom (P_\lambda)$, we get from Lemma~\ref{lem:dual_scaling} that 
\begin{align}\label{eq:LogReg_dual_point}
\vt{\Theta}(\x) = \dscale\left(\frac{1}{\lambda} \left( \y - \frac{\e^{\A\x}}{1+\e^{\A\x}} \right) \right)
\end{align}
is such that $\vt{\Theta}(\x) \in \Delta_\A$. Moreover, $\vt{\Theta}(\x) \rightarrow \vtheta^\star$ as $\x \rightarrow \x^\star$.

%


\subsection{\texorpdfstring{$\beta$}{beta}-Divergence with \texorpdfstring{$\beta \in (1, 2)$}{beta (1,2)}} \label{ssec:example_beta}


The data fidelity term is given by the $\beta$-divergence between the  input signal $\y \in \R^m_+$ and its reconstruction $\A\x$, i.e.
\begin{align} \label{eq:beta_data_fidelity}
F(\A\x) 
= \frac{1}{\beta(\beta-1)} \left( \|\y\|_{\beta}^{\beta} + (\beta-1)\|\A\x+\epsilon\|_{\beta}^{\beta} - \beta \y^\T (\A\x+\epsilon)^{\beta-1} \right) \\
f_i(\Axi) 
= \frac{1}{\beta(\beta-1)} \left( \yi^\beta + (\beta-1)(\Axi+\epsilon)^\beta - \beta \yi (\Axi+\epsilon)^{\beta -1} \right)  \label{eq:beta_data_fidelity-coord}
\end{align}
Note that we introduce an $\epsilon$-smoothing factor ($\epsilon >0$) on the second variable.
This is a common practice in the literature~\citep{Harmany2012} to avoid singularities around zero.

For the above equations to be well-defined we need that $\A\x + \epsilon \geq 0$ for all $\x \in \s{C}$. We therefore take $\s{C}=\R^n_{+}$. Moreover, we consider that $\A \in \R^{m\times n}_{+}$.

The first derivative is given by:
\begin{align}
\nabla F(\vt{z}) =  (\vt{z} + \epsilon)^{\beta-2}(\vt{z} + \epsilon - \y)
& &
f'_i(z_i) = (z_i + \epsilon)^{\beta-2}(z_i + \epsilon - \yi)
\end{align}

We then deduce from Theorem~\ref{thm:dual_GLM_and_optimality} that the first-order optimality condition \eqref{eq:GLM_optimality_condition1} (primal-dual link) is given by:
\begin{align} \label{eq:beta_optimality_condition}
\lambda \vthetai^\star =   ([\A\x^\star]_i + \epsilon)^{\beta-2}(\yi -[\A\x^\star]_i - \epsilon) , ~~ \forall i \in [m] 
\end{align}

The maximum regularization parameter $\lambda_{\max}$ is obtained by substituting $\nabla F(\vt{0}) = \epsilon^{\beta-2}(\epsilon -\y)$ in \eqref{eq:max_regularization}:
\begin{align}
    \lambda_{\max} = \epsilon^{\beta-2} \max(\A^\T (\y-\epsilon))
\end{align}

\begin{proposition} \label{prop:beta_fenchel}
The Fenchel conjugate of the $\beta$-divergence data-fidelity function  defined  in \eqref{eq:beta_data_fidelity}  is given by $F^*(\vt{u}) = \sum_{i=1}^m f_i^*(u_i)$ where
\begin{align} \label{eq:beta_fenchel}
 f^*_i(u_i) = \left\{
 \begin{array}{lr}
  -\epsilon u_i, & ~\text{if } ~ \yi = 0, u_i \leq 0 \\
   u_i \hat{z_i} - f(\hat{z_i}) , & ~\text{otherwise} \\ 
 \end{array} \right. 
\end{align}
with $\hat{z_i} \geq 0$ such that
\begin{align}\label{eq:eq_zhat_Fench_betaDiv}
    (\hat{z_i} + \epsilon)^{\beta-2}(\hat{z_i} + \epsilon - \yi) = u_i
\end{align}
\end{proposition}
\begin{proof}
Given that $\dom(f_i) = \{z_i \in \R ~|~ z_i + \epsilon \geq 0\}$, the Fenchel conjugate of $f_i$ is given by
\begin{align}
    f_i^*(u_i) &= \sup_{z_i \in \R} \; \underbrace{z_i u_i  - f_i(z_i)}_{\varphi(z_i)}
    = \sup_{z_i + \epsilon \geq 0 } \varphi(z_i) \label{eq:beta_phi_z}.
\end{align}
Because $f_i$ is a convex function, we have that $\varphi$ is concave. Hence, every stationary point $\hat{z}_i$ such that $\hat{z}_i + \epsilon \geq 0$  and   $0 = \varphi'(\hat{z_i}) = u_i - (\hat{z_i}+\epsilon)^{\beta-2}(\hat{z_i} + \epsilon - \yi)$
is a global maximum.
%
We now distinguish two cases:
\begin{itemize}
    \item When $y_i>0$, $\varphi'(z_i) \rightarrow +\infty$ as $z_i \rightarrow -\epsilon$  and 
$\varphi'(z_i) \rightarrow -\infty$ as $z_i \rightarrow + \infty$, which combined to the fact that $\varphi'$ is continuous 
implies that there exists $\hat{z_i} \geq  -\epsilon$ such that $\varphi'(\hat{z_i}) = 0$. 
    \item When $y_i=0$, we have $\varphi(z_i) = z_i u_i - \frac{1}{\beta}(z_i+\epsilon)^{\beta}$. The equation $\varphi'({z}_i)=0 \, \Longleftrightarrow \, u_i = (z_i+\epsilon)^{\beta-1}$ admits a solution $\hat{z}_i$ that satisfies the constraint $\hat{z}_i+\epsilon \geq 0$ if and only if $u_i\geq 0$. In contrast, when $u_i <0$, because $\varphi'$ is  decreasing (as $\varphi$ is concave) the sup in~\eqref{eq:beta_phi_z} is attained at $z_i+\epsilon=0$ with value $-\epsilon u_i$.
\end{itemize}
This completes the proof.
\end{proof}

Unfortunately, equation~\eqref{eq:eq_zhat_Fench_betaDiv} does not admit an explicit closed-form expression without fixing a value for $\beta$. This prevents the derivation of an explicit form of $f^*_i$ in function of $\beta$. We thus focus on the case $\beta=1.5$ in the  remainder of this section.

\subsubsection{Particular Case of \texorpdfstring{$\beta= 1.5$}{beta=1.5}} 

Taking $\beta = 1.5$ on the previous results we obtain: 
\begin{align} \label{eq:beta15_data_fidelity}
F(\A\x) = 
\frac{4}{3} \left( \|\y\|_{1.5}^{1.5} + \frac{1}{2}\|\A\x+\epsilon\|_{1.5}^{1.5} - \frac{3}{2} \y^\T (\A\x+\epsilon)^{0.5} \right) \\
f_i(\Axi) 
= \frac{4}{3} \left( \yi^{1.5} + \frac{1}{2}(\Axi+\epsilon)^{1.5} - \frac{3}{2} \yi (\Axi+\epsilon)^{0.5} \right)
\end{align}
%
with first derivative is given by
\begin{align}
\nabla F(\vt{z}) =   \sqrt{\vt{z} +\epsilon} - \frac{\y}{\sqrt{\vt{z}+\epsilon}}
& &
f_i'(z_i) =  \sqrt{z_i+\epsilon} - \frac{\yi}{\sqrt{z_i+\epsilon}}.
\end{align}

The first optimality condition becomes:
\begin{align} \label{eq:beta15_optimality_condition1}
\lambda \vthetai^\star =  \frac{\yi}{\sqrt{[\A\x^\star]_i+\epsilon}} - \sqrt{[\A\x^\star]_i+\epsilon} , ~~ \forall i \in [m] 
\end{align}

The maximum regularization parameter $\lambda_{\max}$ is obtained by substituting $\nabla F(\vt{0}) = \frac{\epsilon -\y}{\sqrt{\epsilon}}$ in \eqref{eq:max_regularization}:
\begin{align}
    \lambda_{\max} = \max(\A^\T (\y-\epsilon)/\sqrt{\epsilon})
\end{align}

\begin{proposition}
The Fenchel conjugate of the $\beta$-divergence with $\beta = 1.5$ in equation \eqref{eq:beta15_data_fidelity} is given by $F^*(\vt{u}) = \sum_{i=1}^m f_i^*(u_i)$ where
\begin{align}\label{eq:beta15_fenchel}
f_i^*(u_i) = 
\frac{u_i^3}{6} + \frac{1}{6} (u_i^2 + 4\yi)^{1.5} + u_i \yi - \frac{4}{3}\yi^{1.5} - \epsilon u_i
\end{align}
with $\dom(f_i^*) = \R$.
\end{proposition}
\begin{proof}
According to \Cref{prop:beta_fenchel}, the Fenchel conjugate for $\beta=1.5$ is given by
\begin{align*}\label{eq:beta_fenchel_tmp}
     f^*_i(u_i) = \left\{
 \begin{array}{lr}
  -\epsilon u_i, & ~\text{if } ~ \yi = 0, u_i \leq 0 \\
   u_i \hat{z_i} - f(\hat{z_i}) , & ~\text{otherwise} \\ 
 \end{array} \right. 
\end{align*}
where $\hat{z_i}$ is the solution of the following equation:
\begin{align}
    \frac{(\hat{z}_i+\epsilon - \yi)}{\sqrt{\hat{z}_i+\epsilon}} = u_i.
\end{align}
It gives a quadratic equation on $\sqrt{\hat{z_i}+\epsilon}$ with solution
\begin{align}
    \sqrt{\hat{z_i}+\epsilon} = \frac{u_i \pm \sqrt{u_i^2+4\yi}}{2}  \implies
    \hat{z_i} = \left( \frac{u_i + \sqrt{u_i^2+4\yi}}{2} \right)^2 -\epsilon
\end{align}
since $u_i \leq \sqrt{u_i^2+4\yi}$ for $\yi \geq 0$.
Plugging $\hat{z_i}$ in \eqref{eq:beta_fenchel_tmp} we obtain after some calculation: 
\begin{align}\label{eq:beta_fenchel_tmp2}
 f^*_i(u_i) = \left\{
 \begin{array}{lr}
  -\epsilon u_i, & ~\text{if } ~ \yi = 0, u_i \leq 0 \\
   \frac{u_i^3}{6} + \frac{1}{6} (u_i^2 + 4\yi)^{1.5} + u_i \yi - \frac{4}{3}\yi^{1.5} - \epsilon u_i , & ~\text{otherwise} \\ 
 \end{array} \right. 
\end{align}

Finally, note that taking $\yi=0$ and $u_i\leq 0$ in \eqref{eq:beta15_fenchel} we obtain $f_i^*(u_i) = -\epsilon u_i$, in accordance with equation \eqref{eq:beta_fenchel_tmp2}.
Therefore, \eqref{eq:beta15_fenchel} alone summarizes both cases depicted in~\eqref{eq:beta_fenchel_tmp2}.
\end{proof}

The dual function $D_\lambda(\vtheta) = \sum_{i=1}^m - f^*_i(-\lambda \vthetai)$ is given by:
\begin{align} \label{eq:beta15_dual}
D_\lambda(\vtheta) = 
 & \sum_{i=1}^m \frac{1}{6}(\lambda\vthetai)^3 - \frac{1}{6} \left((\lambda\vthetai)^2 + 4\yi\right)^{1.5} + \lambda \vthetai \yi + \frac{4}{3}\yi^{1.5}   - \epsilon \lambda \vthetai 
\end{align}
with $\dom (D_\lambda) = \R^m$. 
Then, we get from Theorem~\ref{thm:dual_GLM_and_optimality}, $\s{C} = \R^n_+$, and~\eqref{eq:DualL1} (dual norm of the $\ell_1$-norm)  that 
the dual feasible set is given by:
\begin{align}
  \Delta_\A = \{\vtheta \in \R^m ~|~ \A^\T\vtheta  \leq \vt{1} \}
\end{align}

The Hessian $\nabla^2 D_\lambda(\vtheta)$  and corresponding eigenvalues $\sigma_i(\vthetai)$ are given by
\begin{align} \label{eq:beta_eigenvalues}
    \nabla^2 D_\lambda(\vtheta) = \diag{[\sigma_i(\vthetai)]_{i \in [m]}},
    &&
    \sigma_i\left( \vthetai \right)  = -\lambda^2 \left( \frac{(\lambda\vthetai)^2 + 2\yi}{\sqrt{ (\lambda\vthetai)^2 + 4\yi}} - \lambda\vthetai  \right)
\end{align}

Although the eigenvalues are all non-positive, they tend to zero as $\vthetai$ tends to infinity.
%
Note that $\sigma_i(\vthetai)$ also vanishes when $\yi = 0$ and $\vthetai \geq 0$.
Therefore $D_\lambda(\vtheta)$ is not globally strongly concave and the standard Gap Safe approach cannot be applied in this case.

In the following sections we find local strong concavity bounds that allows us to deploy the proposed screening strategy (Algorithms~\ref{alg:solver_screening_bis}  and~\ref{alg:solver_screening_local}).

\begin{proposition}\label{prop:beta15_eigenvalues_increasing}
The eigenvalues $\sigma_i(\vthetai)$ in \eqref{eq:beta_eigenvalues} are an  increasing (resp. strictly increasing) function of $\vthetai$ for any $\yi \geq 0$ (resp. $\yi > 0$).
\end{proposition}
\begin{proof}
Indeed, its first derivative is non-negative
\begin{align} 
    \sigma_i'(\vthetai) = \lambda^3 \frac{( (\lambda\vthetai)^2 + 4\yi)^{1.5} - (\lambda\vthetai)^3 - 6 \yi \lambda\vthetai}{( (\lambda\vthetai)^2 + 4\yi)^{1.5}} \geq 0
\end{align}
since  (let us recall that $\vt{y} \in \R^m_+$) for $\lambda\vthetai \leq 0$ we have $((\lambda\vthetai)^2 + 4\yi)^{1.5} \geq 0 \geq (\lambda\vthetai)^3 + 6 \yi \lambda\vthetai$  and
for  $\lambda\vthetai > 0$ one can easily verify that  $\left(((\lambda\vthetai)^2 + 4\yi)^{1.5}\right)^2 \geq \left( (\lambda\vthetai)^3 + 6 \yi \lambda\vthetai \right)^2$ with strict inequalities when $\yi>0$.
\end{proof}


\subsubsection{Validity of the Set \texorpdfstring{$\SO$}{S0}}

In order to be able to derive local strong concavity bounds (see Sections~\ref{ssec:beta15_fixed_alpha} and \ref{ssec:beta15_adaptive_alpha} hereafter), we need to use the set $\SO$ below (see discussion in \Cref{ssec:beta15_discussion_SO}). 
To comply with Theorem~\ref{prop:local_concavity}, we need to show that $\vtheta^\star \in \SO$.

\begin{proposition} \label{prop:beta15_S0}
For any $\A\in\R_{+}^{m \times n}$ with no all-zero row, $\y \in \R_{+}^m$, the set
\begin{align} \label{eq:beta15_SO}
  &\SO = \left\lbrace\vtheta \in \R^m ~\middle|~ \lambda \vtheta \leq \min( \vt{b},(\y-\epsilon)/\sqrt{\epsilon} )  \right\rbrace  \\
  \text{with } \;
  %
  &b_i := \lambda \min_{\{j \in [n] ~|~ a_{ij} \neq 0 \}} \left( \frac{1 - c \|\cA_j\|_1 }{ a_{ij}}  \right) + \lambda c , ~~ i \in [m] \nonumber \\
  %
  &c := -\frac{1}{\lambda} \sqrt[3]{\frac{4 \|\y\|_{1.5}^{1.5} + 2(m-1)\epsilon^{1.5} + 3\epsilon}{1-3\epsilon} }. \nonumber  
\end{align}
is such that $\vtheta^\star \in \SO$.
\end{proposition}
%
\begin{proof}
Because $\A \in \R_+^{m \times n}$ and $\x^\star \in \R^n_+$, we get from the optimality condition \eqref{eq:beta15_optimality_condition1} that
\begin{align*}
    \lambda \vtheta^\star = \frac{\y -\A\x^\star - \epsilon}{\sqrt{\A\x^\star + \epsilon}} \leq \frac{\y-\epsilon}{\sqrt{\epsilon}}
\end{align*}
In particular, for coordinates $\s{I}_0 = \{i \in [m] ~|~ \yi = 0\}$, this inequality simplifies to $\lambda \vtheta_{\s{I}_0} \leq -\sqrt{\epsilon}$ (which is strictly negative, as desired to prevent Hessian eigenvalues from vanishing). 
Let us emphasize that defining $\SO = \left\lbrace\vtheta \in \R^m ~\middle|~ \lambda \vtheta \leq (\y-\epsilon)/\sqrt{\epsilon}  \right\rbrace$ is already a valid choice. However, combining it with the tricky bound $\mathbf{b}$ (derived hereafter) can lead to improved strong concavity bounds, which are particularly relevant for Algorithm~\ref{alg:solver_screening_bis}.

Then, 
to show that $\lambda \vthetai^\star \leq b_i$ we will start by finding a lower bound $c_i \leq \vthetai^\star$ which, combined with the definition of the feasible set, will directly lead to the desired upper bounds~$b_i$.
Consider the dual point $\vtheta = \vt{0}$, which is always feasible, then we have:
\begin{align} \label{eq:app:beta_15_SO_proof1}
    D_\lambda(\vtheta^\star) \geq D_\lambda(\vt{0})  = 0
\end{align}
%
Also, denoting for simplicity $d_i(\vthetai) \!=\! -f_i^*(-\lambda \vthetai)$ such that
$D_\lambda (\vtheta) \!=\! \sum_i d_i(\vthetai)$, one can verify that:
\begin{align} \label{eq:app:beta15_upperbound_di}
\sup_{\vthetai} d_i(\vthetai) =  \frac{4}{3}\yi^{1.5} - 2\yi\sqrt{\epsilon} + \frac{2}{3} \epsilon^{1.5} \leq \frac{4}{3}\yi^{1.5} + \frac{2}{3} \epsilon^{1.5}
\end{align}
Indeed, 
$d_i(\vthetai)$ being concave we have that the stationary point $\vthetai = \frac{\yi-\epsilon}{\sqrt{\epsilon}}$ (obtained by setting $d'_\lambda(\vthetai) = 0$ with the help of some symbolic calculation tool) 
is the global maximum, with value $d_i(\vthetai) =  \frac{4}{3}\yi^{1.5} - 2\yi\sqrt{\epsilon}  + \frac{2}{3} \epsilon^{1.5}$.
Combining \eqref{eq:app:beta_15_SO_proof1} and \eqref{eq:app:beta15_upperbound_di} we obtain
\begin{align*}
    d_i (\vthetai^\star) &= D_\lambda(\vtheta^\star) - \sum_{i'\neq i} d_{i'}(\theta_{i'}^\star)
    \geq 0 - \sum_{i'\neq i} d_{i'}(\theta_{i'}^\star) \\ 
    &\geq -  \sum_{i'\neq i} \left(\frac{4}{3}y_{i'}^{1.5} + \frac{2}{3} \epsilon^{1.5} \right) 
    = -\frac{4}{3}( \|\y\|_{1.5}^{1.5} - \yi^{1.5}) - \frac{2}{3}(m-1) \epsilon^{1.5} 
\end{align*}
Since $d_i(\vthetai)$ is concave, continuous, and $\lim_{\vthetai \rightarrow -\infty} d_i(\vthetai) = -\infty$,  
it implies that there exists a $\hat{c}_i \leq \vthetai^\star$ such that $d_i(\hat{c}_i) = -\frac{4}{3}( \|\y\|_{1.5}^{1.5} - \yi^{1.5})  - \frac{2}{3}(m-1) \epsilon^{1.5}$ 
and which can be obtained by solving the resulting equation for $\hat{c}_i$.

This, however, leads to a cumbersome calculation.
Instead, we provide another bound $c_i \leq \hat{c}_i \leq \vthetai^\star$.
To do so, we use an upper bound $g \geq d_i$ and define $c_i$ such that $g(c_i) = d_i(\hat{c}_i)$.
We then have  $g(c_i) = d_i(\hat{c}_i) \leq g(\hat{c}_i)$ and because we choose $g$ an increasing function, it implies that $c_i \leq \hat{c}_i$ as desired.
Noting that $\hat{c}_i \leq 0$ since $-\frac{4}{3}( \|\y\|_{1.5}^{1.5} - \yi^{1.5}) - \frac{2}{3}(m-1) \epsilon^{1.5}\leq 0 = d_i(0)$, 
we use the following bound on $d_i$ valid for any $\vthetai \leq 0$:
\begin{align*}
    (\forall~ \vthetai \leq 0) ~~d_i(\vthetai)  &= \frac{1}{6}(\lambda\vthetai)^3 - \frac{1}{6} \left((\lambda\vthetai)^2 + 4\yi\right)^{1.5} + \lambda \vthetai \yi + \frac{4}{3}\yi^{1.5} - \epsilon \lambda \vthetai \\
    &\leq \frac{1}{6}(\lambda\vthetai)^3 - \frac{1}{6} |\lambda\vthetai|^3  + \lambda \vthetai \yi + \frac{4}{3}\yi^{1.5} - \epsilon \lambda \vthetai \\
    &\leq \frac{1}{3}(\lambda\vthetai)^3 + \frac{4}{3}\yi^{1.5} - \epsilon \lambda \vthetai\\
    &\leq \left(\frac{1}{3}-\epsilon\right)(\lambda\vthetai)^3 + \frac{4}{3}\yi^{1.5} + \epsilon = g(\vthetai)
\end{align*}
where in the last step we used the fact that $- \epsilon \lambda \vthetai \leq - \epsilon\left( (\lambda \vthetai)^3 - 1 \right)$, $\forall \vthetai \leq 0$. 
We can now compute $c_i$ such that  $g(c_i) = d_i(\hat{c}_i)= -\frac{4}{3}( \|\y\|_{1.5}^{1.5} - \yi^{1.5}) - \frac{2}{3}(m-1) \epsilon^{1.5}$ 
by solving the following equation for $c_i$: 
\begin{align*}
    \left(\frac{1}{3}-\epsilon\right)(\lambda c_i)^3 + \frac{4}{3}\yi^{1.5} + \epsilon = -\frac{4}{3}( \|\y\|_{1.5}^{1.5} - \yi^{1.5}) - \frac{2}{3}(m-1) \epsilon^{1.5} 
\end{align*} 
%
%
\begin{align*}
   \Longrightarrow \;  c  := c_i = -\frac{1}{\lambda} \sqrt[3]{\frac{4 \|\y\|_{1.5}^{1.5} + 2(m-1)\epsilon^{1.5}+3\epsilon}{1-3\epsilon} } \leq \vthetai^\star. 
\end{align*}
As $c_i$ turns out to be independent of $i$, we denote simply $c$ such that $c\leq \vthetai^\star$ for all $i \in [m]$.
%

Now, we can upper bound $\vthetai^\star$ for any $i\in [m]$ by combining the above lower bound 
with fact that $ \A^\T \vtheta^\star \leq \vt{1}$ (since $\vtheta^\star \in \Delta_\A$).
For all $i \in [m]$ and $\A \in \R^{m\times n}_{+}$ we have:
\begin{align*}
    \forall j \in [n], ~~
    1 \geq \cA_j^\T \vtheta^\star = \sum_{i'=1}^m a_{i'j} \theta_{i'}^\star  = a_{ij} \vthetai^\star + \sum_{i' \neq i } a_{i'j} \theta_{i'}^\star
    \geq a_{ij} \vthetai^\star + \sum_{i' \neq i} a_{i'j}  c 
\end{align*}
which leads to the following upper bound $b_i$ for $\lambda \vthetai^\star$, with $i \in [m]$:
\begin{align}
    \forall j \in [n], ~~
    &a_{ij} \vthetai^\star \leq 1 -  \sum_{i'\neq i} a_{i'j} c
    = 1 - c (\|\cA_j\|_1 - a_{ij})
    \\
    \iff ~~
    & \vthetai^\star \leq \min_{\{j \in [n] ~|~ a_{ij} \neq 0 \}} \left( \frac{1 - c \|\cA_j\|_1 }{ a_{ij}}  \right)  + c := b_i/ \lambda
\end{align}
where our assumption that $\A$ has no all-zero row implies that the set $\{j \in [n] ~|~ a_{ij} \neq 0 \}$ is non-empty.
Finally,  because $b_i>0 ~\forall i$, the inequality $\lambda \vtheta_{\s{I}_0} \leq -\sqrt{\epsilon}$ is always more restrictive for coordinates $i\in \s{I}_0$.
\end{proof}

\subsubsection{Strong-Concavity Bound on \texorpdfstring{$\Delta_\A \cap \SO$}{the Feasible Set}} \label{ssec:beta15_fixed_alpha}
\begin{proposition} \label{prop:beta15_alpha_fixed}
Let $\SO$ be the set defined in Proposition~\ref{prop:beta15_S0}
and $\A \in \R_{+}^{m \times n}$. Then the dual function $D_\lambda$  as defined in \eqref{eq:beta15_dual} is $\alpha_{\Delta_\A \cap \SO}$ strongly concave on $\Delta_\A\cap \SO$ with constant:
\begin{align} \label{eq:beta15_alpha_fixed}
    \alpha_{\Delta_\A \cap \SO} = \min_{i \in [m]} - \sigma_i \left(\frac{1}{\lambda}\min\left( b_i, \frac{\yi - \epsilon}{\sqrt{\epsilon}}\right)\right), 
\end{align}
where the $b_i$ are quantities that define $\SO$.
\end{proposition}
%
\begin{proof}
From \Cref{prop:strong_concavity} (with $\s{I}=\emptyset$) we have to prove that
 \begin{equation}
     \alpha_{\Delta_\A \cap \SO} \leq \min_{i \in [m]} - \sup_{\vtheta \in  \Delta_\A \cap \SO} \sigma_i(\vthetai).
 \end{equation}
 Here, we have $ \Delta_\A \cap \SO = \{\vtheta \in \R^m~|~ \A^\T \vtheta \leq \vt{1} , ~  \lambda \vtheta \leq \vt{b}, \lambda \vtheta \leq \frac{\y - \epsilon}{\sqrt{\epsilon}}  \}$.

Within $\SO$ the coordinates $\vthetai$ are upper bounded with $\lambda \vthetai \leq \min \left(b_i, \frac{\yi -\epsilon}{\sqrt{\epsilon}} \right)$ (cf. \eqref{eq:beta15_SO}).
Using the fact that $\sigma_i$ is an increasing function of $\vthetai$ (\Cref{prop:beta15_eigenvalues_increasing}) and that $\Delta_\A \cap \SO \subset \SO$,
 we have that 
$\sup_{\vtheta \in \Delta_\A \cap \SO} \sigma_i(\vthetai) \leq \sup_{\vtheta \in \SO} \sigma_i(\vthetai)  = \sigma_i \left( \frac{1}{\lambda} \min \left(b_i, \frac{\yi -\epsilon}{\sqrt{\epsilon}} \right) \right)$ for all $i\in [m]$.
This leads to the following result:
\begin{align*}
    \alpha_{\Delta_\A \cap \SO} = \min_{i \in [m]} - \sigma_i \left( \frac{1}{\lambda} \min \left(b_i, \frac{\yi -\epsilon}{\sqrt{\epsilon}} \right) \right) \leq 
    \min_{i \in [m]} - \sup_{\vtheta \in  \Delta_\A \cap \SO} \sigma_i(\vthetai).
\end{align*}
which concludes the proof.
\end{proof}

\subsubsection{Strong-Concavity Bound on \texorpdfstring{$\s{B}({\vtheta},r) \cap \SO$}{a Ball} } \label{ssec:beta15_adaptive_alpha}

\begin{proposition} \label{prop:beta15_alpha_adaptive}
Let $\SO$ be the set defined in Proposition~\ref{prop:beta15_S0}. Then, for any ball $\s{B}(\vtheta,r)$ with $\vtheta \in \dom(D_\lambda)\cap \SO$, 
the dual function $D_\lambda$  as defined in \eqref{eq:beta15_dual} is $\alpha_{\s{B}(\vtheta,r) \cap \SO}$ strongly concave on $\s{B}(\vtheta,r) \cap \SO$ with constant:
%
\begin{align} \label{eq:beta15_alpha_adaptive}
  &\alpha_{\s{B}(\vtheta,r)\cap \SO} = \min_{i \in [m]} -\sigma_i(d_i/\lambda) \\
  \text{with } \; &d_i = \min \left( \lambda(\vthetai+r), \; b_i, \; \frac{\yi -\epsilon}{\sqrt{\epsilon}} \right).
\end{align}
where the $b_i$ are quantities that define $\SO$.
\end{proposition}
\begin{proof}
From  \Cref{prop:strong_concavity} (with $\s{I}=\emptyset$) we have to prove that
 \begin{equation}
     \alpha_{\s{B}(\vtheta,r) \cap \SO} \leq \min_{i \in [m]} - \sup_{\vtheta' \in  \s{B}(\vtheta,r) \cap \SO} \sigma_i(\vthetai').
 \end{equation}
By definition of $\SO$ in \eqref{eq:beta15_SO} we have that
\begin{equation}
    \vtheta' \in \s{B}(\vtheta,r) \cap \SO \, \Longleftrightarrow \, \forall i \in [m], \vthetai' \in
    \left[\vthetai - r , \; \min \left(\vthetai + r, \; \frac{b_i}{\lambda}, \; \frac{\yi - \epsilon}{\lambda\sqrt{\epsilon}}\right)\right].
\end{equation}
Combining that with the fact that $\sigma_i$ is an increasing function of $\vthetai$ (Proposition~\ref{prop:beta15_eigenvalues_increasing}), we obtain 
 \begin{equation}
     \sup_{\vtheta' \in  \s{B}(\vtheta,r) \cap \SO} \sigma_i(\vthetai') = 
     \sigma_i\left(\min\left(\vthetai + r, \; \frac{b_i}{\lambda}, \; \frac{\yi - \epsilon}{\lambda\sqrt{\epsilon}} \right)\right),
 \end{equation}
which completes the proof.
%
\end{proof}

\subsubsection{Dual Update}


The so-called generalized residual \citep{Ndiaye2017} is given by 
$$\vt{\rho}(\x) := -  \nabla F(\A\x) = \frac{\y}{ \sqrt{\A\x + \epsilon}} -  \sqrt{\A\x+\epsilon}$$
and the dual feasible point $\vtheta \in \Delta_\A$ defined via scaling in equation \eqref{eq:dual_scaling_proposed} 
becomes:
\begin{align} \label{eq:beta15_dual_scaling}
  \vtheta =  \dscale(\vt{\rho(\x)}/\lambda) =  
  \dscale\left( \frac{1}{\lambda} \left( \frac{\y}{\sqrt{\A\x +\epsilon}} -  \sqrt{\A\x+\epsilon} \right) \right)
\end{align}

When using Algorithms \ref{alg:solver_screening_bis} and \ref{alg:solver_screening_local}, the computation of the dual feasible point should also take into account $\SO$. A dual feasible point $\vt{\Theta}(\x) \in \Delta_\A \cap \SO$ can be computed as shown in Proposition~\ref{prop:beta15_dual_update}. 

\begin{proposition} \label{prop:beta15_dual_update}
Let $\SO$ be the set defined in \Cref{prop:beta15_S0} and $\A \in \R_+^{m \times n}$.
Let $\x \in \R_+^n$ be a primal feasible point and $\vt{\rho}(\x) \in \R^m$ the corresponding residual.
Then $\vt{\Theta}(\x) \in \R^m$ defined as follows
\begin{align}\label{eq:beta15_dual_update}
    [\vt{\Theta}(\x)]_i 
    &= \min\left([\dscale\left(\vt{\rho}(\x)/\lambda \right)]_i,~ \frac{b_i}{\lambda},~ \frac{\yi - \epsilon}{\lambda\sqrt{\epsilon}}\right)
\end{align}
is such that $\vt{\Theta}(\x) \in \Delta_\A \cap \SO$.
Moreover, we have $\vt{\Theta}(\x) \rightarrow \vtheta^\star$ as $\x \rightarrow \x^\star$
\end{proposition}
\begin{proof}
First, note that  $\vt{\Theta}(\x) \in \SO$ by construction, since $\lambda [\vt{\Theta}(\x)]_i \leq b_i$ and $\lambda [\vt{\Theta}(\x)]_i \leq \frac{\yi - \epsilon}{\sqrt{\epsilon}}$ for all $i \in [m]$.
To show that $\vt{\Theta}(\x) \in \Delta_\A = \{\vtheta \in \R^m ~|~ \A^\T \vtheta \leq \vt{1}\}$, we use the fact that  $\vt{\Theta}(\x) \leq  \dscale(\vt{\rho}(\x)/\lambda)$ and, multiplying both sides  by $\A^\T$, we obtain $\A^\T \vt{\Theta}(\x) \leq \A^\T \dscale(\vt{\rho}(\x)/\lambda) \leq \vt{1}$ using the definition of $\dscale$ (see~\eqref{eq:dual_scaling_L1}) for the last inequality.

%
Moreover, because $F \in C^1$, we have from \Cref{lem:dual_scaling} that
$\dscale(-\vt{\rho(\x)}/\lambda) \rightarrow \vtheta^\star$ as $\x \rightarrow \x^\star$.
%
Then, because $\vtheta^\star \in \SO$, we have that   $\min\left([\dscale(\vt{\rho}(\x^\star)/\lambda)]_i, ~ \frac{b_i}{\lambda}, ~ \frac{\yi-\epsilon}{\lambda \sqrt{\epsilon}}\right) = [\dscale(\vt{\rho}(\x^\star)/\lambda)]_i = \vthetai^\star$ for all $i\in [m]$.
Hence, by continuity of the operator $\vt{\Theta}$ in \eqref{eq:beta15_dual_update} (which is the $\min$ of two continuous functions), we conclude that $\vt{\Theta}(\x) \rightarrow \vtheta^\star$ as $\x \rightarrow \x^\star$.
\end{proof}


\begin{remark}
Because $\vtheta = \vt{\Theta}(\x)$ in eq. \eqref{eq:beta15_dual_update} is no longer a simple scaling of $\vt{\rho}(\x)$, the quantity $\cA_j^\T\vtheta$ required by the screening test cannot be obtained directly from $\cA_j^\T \vt{\rho}(\x)$ (which is often available from the solver's update step).
Obviously, calculating $\cA_j^\T\vtheta$ from scratch is always an option,
which may remain interesting if the computational savings provided by screening compensates this computational overhead on the screening test computation.
Another option (the one adopted in our experiments) is to use $\vtheta' = \dscale(\vt{\rho(\x)}/\lambda)$ in eq. \eqref{eq:beta15_dual_scaling} for the purpose of the screening test only. This way, $\cA_j^\T\vtheta'$ is just a scaling of $\cA_j^\T \vt{\rho}(\x)$
and the screening test remains safe since $\cA_j^\T\vtheta' \geq \cA_j^\T\vtheta$ (indeed, $\vthetai' = [\dscale(\vt{\rho(\x)}/\lambda)]_i \geq [\vt{\Theta}(\x)]_i = \vthetai ~\forall i \in [m]$).
\end{remark}

\subsection{Kullback-Leibler Divergence  (\texorpdfstring{$\beta$}{beta}-Divergence with \texorpdfstring{$\beta =1$}{beta=1}) } \label{ssec:example_KL}

%
%
%
%
%

The data fidelity term is given by the Kullback Leibler divergence
between the  input signal $\y \in \R^m_+$ and its reconstruction $\A\x$, i.e.
\begin{align}\label{eq:KL_data_fidelity}
F(\A\x) =  \y^\T \log\left(\frac{\y}{\A\x+\epsilon}\right)  + \vt{1}^\T (\A\x+\epsilon - \y)
\\
f_i(\Axi) = \yi \log\left(\frac{\yi}{\Axi+\epsilon}\right)  + \Axi+\epsilon - \yi
\end{align}
where, just like in the previous case of $\beta \in (1,2)$, we introduce an $\epsilon$-smoothing factor ($\epsilon >0$) on the second variable. 
%
Also, similarly to \Cref{ssec:example_beta}, we set $\s{C}=\R^n_{+}$ and we consider that $\A \in \R^{m\times n}_{+}$.


Its first derivative is given by:
\begin{align}
\nabla F(\vt{z}) = -\frac{ \y}{\vt{z}+\epsilon} + \vt{1}  
& &
f_i'(z_i) = - \frac{\yi}{z_i+\epsilon} + 1 
\end{align}

We then deduce from Theorem~\ref{thm:dual_GLM_and_optimality} that the first-order optimality condition \eqref{eq:GLM_optimality_condition1} (primal-dual link) is given by:
\begin{align}\label{eq:KL_optimality_condition1}
\lambda \vtheta^\star = \frac{\y}{\A\x^\star+\epsilon} - \vt{1}   
&&
\lambda \vthetai^\star =  \frac{\yi}{[\A\x^\star]_i+\epsilon} - 1  , ~~ \forall i \in [m] 
\end{align}

The maximum regularization parameter $\lambda_{\max}$ is obtained by substituting $\nabla F(\vt{0}) = \frac{\epsilon -\y}{\epsilon}$ in \eqref{eq:max_regularization}:
\begin{align}
    \lambda_{\max} = \max(\A^\T (\y-\epsilon)/\epsilon)
\end{align}

\begin{proposition}
The Fenchel conjugate of the KL-divergence  in equation \eqref{eq:KL_data_fidelity} is given by $F^*(\vt{u}) = \sum_{i=1}^m f_i^*(u_i)$ where
\begin{align}\label{eq:app:KL_fenchel}
f_i^*(u) = -\yi\log(1-u) - \epsilon u
\end{align}
with 
$\dom(f_i^*) = (-\infty,1)$.
\end{proposition}
\begin{proof}
The Fenchel conjugate $f_i^*$ of the scalar function $f_i = d_{\KL}$ is given by 
\begin{align}
f_i^*(u) &= \sup_{z \in \R} \; z u  - f_i(z) \nonumber \\ 
& = \sup_{z + \epsilon \geq 0} z u  - \yi\log\left(\frac{\yi}{z+\epsilon}\right) + \yi - z - \epsilon \nonumber \\ 
& = \yi - \epsilon +  \sup_{z + \epsilon \geq 0} \underbrace{z(u-1)  - \yi\log\left(\frac{\yi}{z+\epsilon}\right)}_{\varphi(z)} \label{eq:phi_z}
\end{align}
To solve this supremum problem, we distinguish two cases.
\begin{itemize}
    \item When $\yi>0$, the result follows by solving
$\varphi'(z^\star) = u - 1 + \frac{\yi}{z^\star+\epsilon} = 0 \iff z^\star = \frac{\yi}{1-u} - \epsilon$,
which is a global maximum 
as $\varphi$ is strictly concave with $\varphi''(z) = -\frac{\yi}{(z+\epsilon)^2}  < 0$. 
Plugging $z^\star$ into \eqref{eq:phi_z} we obtain:
\begin{align}\label{eq:app:fenchel_F_epsilon}
f_i^*(u) = -\yi\log(1-u) - \epsilon u.
\end{align}
with domain given by $\dom(f_i^*) =  (-\infty,1)$.
\item When $\yi=0$, one can easily verify that
\begin{equation}
 f_i^*(u) = - \epsilon +  \sup_{z + \epsilon \geq 0} z(u-1) = \left\lbrace 
\begin{array}{ll}
     -\epsilon u & \text{if } u < 1 \\
     +\infty & \text{otherwise.}
\end{array}
\right.   
\end{equation}
with the supremum being attained at $z^\star = -\epsilon$ in the case $u<1$. 
Then, note that the same result is retrieved by taking $\yi =0$ in equation \eqref{eq:app:fenchel_F_epsilon}, which can therefore be used in all cases.
\end{itemize}
This completes the proof.
\end{proof}

The dual function $D_\lambda(\vtheta) = -\sum_{i=1}^m f^*_i(-\lambda \vthetai)$ is given by:
\begin{align} \label{eq:KL_dual_function}
D_\lambda(\vtheta) 
= \sum_{i=1}^m \yi \log(1+\lambda\vthetai) - \epsilon \lambda \vthetai 
\end{align}
with domain given by
$ \dom (D_\lambda) = \{\vtheta \in \R^m ~|~ \vtheta \geq -\vt{1}/\lambda\} $.
%
%
Then, we get from Theorem~\ref{thm:dual_GLM_and_optimality}, $\s{C} = \R^n_+$, and~\eqref{eq:DualL1} (dual norm of the $\ell_1$-norm)  that
\begin{align}
  \Delta_\A = \{\vtheta \in \R^m ~|~ \A^\T\vtheta  \leq \vt{1}, \vtheta \geq -\vt{1}/\lambda \}
\end{align}

The Hessian $\nabla^2 D_\lambda(\vtheta)$  and corresponding eigenvalues $\sigma_i(\vthetai)$ are given by
\begin{align} \label{eq:KL_eigenvalues}
    \nabla^2 D_\lambda(\vtheta) = \diag{[\sigma_i(\vthetai)]_{i \in [m]}},
    &&
    \sigma_i\left( \vthetai \right)  = -\frac{\lambda^2\yi}{(1+\lambda\vthetai)^2}
\end{align}

Although the eigenvalues are all non-positive, they tend to zero as $|\vthetai|$ tends to infinity.
Moreover, $\sigma_i(\vthetai)$ also vanishes when $\yi = 0$.
Therefore $D_\lambda(\vtheta)$ is not globally strongly concave and the standard Gap Safe approach cannot be applied in this case.

In the following sections we find local strong concavity bounds that allows us to deploy the proposed screening strategy (Algorithms~\ref{alg:solver_screening_bis}  and~\ref{alg:solver_screening_local}).

\begin{proposition}\label{prop:KL_eigenvalues_increasing}
The eigenvalues $\sigma_i(\vthetai)$ in \eqref{eq:KL_eigenvalues} are an  increasing (resp. strictly increasing) function of $\vthetai$ 
on $[ - 1/\lambda, +\infty )$ for any  $\yi \geq 0$ (resp. $\yi > 0$).
\end{proposition}
\begin{proof}
Indeed, its first derivative is non-negative
\begin{align} 
    \sigma_i'(\vthetai) =  \lambda^3 \frac{2\yi}{(1+\lambda\vthetai)^3} \geq 0
\end{align}
since $\vt{y} \in \R^m_+$ and $\vthetai \geq - 1/\lambda \implies 1+ \lambda\vthetai \geq 0$.
\end{proof}

\subsubsection{Validity of the Set \texorpdfstring{$\SO$}{S0}} 


In order to be able to derive local strong concavity bounds (see Sections~\ref{ssec:KL_fixed_alpha} and \ref{ssec:KL_adaptive_alpha} hereafter), we need to use the set $\SO$ below (see discussion in \Cref{ssec:KL_discussion_SO}). 
To comply with Theorem~\ref{prop:local_concavity}, we need to show that $\vtheta^\star \in \SO$.

%
\begin{proposition} \label{prop:KL_S0}
Let $\s{I}_0 = \{i \in [m] ~|~ \yi = 0\}$. Then, for any $\A\in\R_{+}^{m \times n}$, $\y \in \R_{+}$, the set
\begin{align} \label{eq:KL_SO}
  &\SO = \{\vtheta \in \R^m ~|~ \vtheta_{\s{I}_0} = - \vt{1}/\lambda\} 
\end{align}
is such that $\vtheta^\star \in \SO$.
\end{proposition}
\begin{proof}
Optimality condition \eqref{eq:KL_optimality_condition1} directly implies that the dual solution takes the value $\vthetai^\star = -1/\lambda$ for all coordinates $i\in\s{I}_0$, i.e.  all coordinates $i$ such that $\yi = 0$.
\end{proof}

With that in hand, we are now able to compute local strong concavity bounds on both $\Delta_\A \cap \SO$ and $\s{B}(\vtheta,r) \cap \SO$.

\subsubsection{Strong-Concavity Bound on \texorpdfstring{$\Delta_\A \cap \SO$}{the Feasible Set}} \label{ssec:KL_fixed_alpha}

\begin{proposition}
\label{prop:KL_alpha_fixed}
Let $\SO$ be the set defined in Proposition~\ref{prop:KL_S0} and $\A \in \R_{+}^{m \times n}$ with no all-zero row. Then, the dual function $D_\lambda$  as defined in \eqref{eq:KL_dual_function} is $\alpha_{\Delta_\A \cap \SO}$ strongly concave on $\Delta_\A\cap \SO$ with:
\begin{align} \label{eq:KL_alpha_fixed}
    \alpha_{\Delta_\A \cap \SO} = \lambda^2 \min_{i \in \s{I}_0^\complement} \frac{\yi}{ \left(\min_{\{j\in [n] ~|~ a_{ij} \neq 0\}}\left( \frac{\lambda + \|\cA_j\|_1}{a_{ij}} \right) \right)^2}
\end{align}
\end{proposition}
\begin{proof}
From  \Cref{prop:strong_concavity} (with $\s{I} = \s{I}_0$) we have to prove that
\begin{align}\label{eq:KL_fixed_alpha_proof-1}
\alpha_{\Delta_\A \cap \SO} \leq 
\min_{i \in \s{I}_0^\complement} - \sup_{\vtheta \in \Delta_\A \cap \SO}  \sigma_i (\vthetai) 
\end{align}
where $\sigma_i (\vthetai) = -\frac{\lambda^2 \yi}{(1+\lambda\vthetai)^2}$ denotes the $i$-th eigenvalue of $\nabla^2 D_\lambda(\vtheta)$. 
Let $\vtheta \in \Delta_\A \cap \SO = \{\vtheta \in \R^m ~|~ \A^\T \vtheta \leq \vt{1} , ~ \vtheta \geq -1/\lambda, ~\vtheta_{\s{I}_0} = -\vt{1}/\lambda \}$. Then, for all $i \in [m]$, we have 
\begin{equation}
    1 \geq \cA_j^\T \vtheta = \sum_{i'=1}^m a_{i'j} \theta_{i'}  = a_{ij} \vthetai + \sum_{i' \neq i } a_{i'j} \theta_{i'}
    \geq a_{ij} \vthetai - \sum_{i' \neq i} a_{i'j}\frac{1}{\lambda}, \ \forall j \in [n],
\end{equation}
using the facts that $\A \in \R_{+}^{m\times n}$ and  $\vthetai \geq -1/\lambda$, for all $i \in [m]$. Reorganising the terms in the above inequalities and using the definition of $\dom(D_\lambda)$ (with $\vthetai \geq - 1/\lambda$)  we obtain, for all $i \in [m]$,
\begin{align}
    &- \frac{1}{\lambda} \leq \vthetai \leq \frac{1 + \sum_{i'\neq i} a_{i'j} \frac{1}{\lambda} }{a_{ij}}
    = \frac{\lambda + \|\cA_j\|_1 }{ \lambda a_{ij}} - \frac{1}{\lambda}, 
    ~~\forall j \in [n] \text{ s.t. } a_{ij} \neq 0 \\
    &\iff ~~
    0 \leq 1 + \lambda \vthetai \leq \min_{\{j\in [n] ~|~ a_{ij} \neq 0\}} \left( \frac{\lambda + \|\cA_j\|_1 }{ a_{ij}} \right) \label{eq:bound_theta_KL}
\end{align}
where our assumption that $\A$ has no all-zero row implies that the set $\{j \in [n] ~|~ a_{ij} \neq 0 \}$ is non-empty.
As $\sigma_i$ is an increasing function of $\vthetai$ (\Cref{prop:KL_eigenvalues_increasing}), we get that the sup in \eqref{eq:KL_fixed_alpha_proof-1} is attained for $1 + \lambda \vthetai=  \min_{\{j\in [n] ~|~ a_{ij} \neq 0\}} \left( \frac{\lambda + \|\cA_j\|_1 }{ a_{ij}} \right)$, which completes the proof. 
\end{proof}

\subsubsection{Strong-Concavity Bound on \texorpdfstring{$\s{B}({\vtheta},r) \cap \SO$}{a Ball}} \label{ssec:KL_adaptive_alpha}


\begin{proposition}
\label{prop:KL_alpha_adaptive}
Let $\SO$ be the set defined in Proposition~\ref{prop:KL_S0}. Then, for any ball $\s{B}(\vtheta,r)$ with $\vtheta \in \dom(D_\lambda)\cap \SO$, 
the dual function $D_\lambda$  as defined in \eqref{eq:KL_dual_function} is $\alpha_{\s{B}(\vtheta,r) \cap \SO}$ strongly concave on $\s{B}(\vtheta,r) \cap \SO$ with:
\begin{align}\label{eq:KL_alpha_adaptive}
  \alpha_{\s{B}(\vtheta,r)\cap \SO} = \lambda^2 \min_{i \in \s{I}_0^\complement} \frac{\yi}{(1+\lambda(\vthetai+r))^2}.  
\end{align}
\end{proposition}
\begin{proof}
From \Cref{prop:strong_concavity} (with $\s{I} = \s{I}_0$) we have to prove that
\begin{align}\label{eq:KL_fixed_alpha_proof}
\alpha_{\s{B}(\vtheta,r) \cap \SO} & \leq 
\min_{i \in \s{I}_0^\complement} - \sup_{\vtheta' \in  \s{B}(\vtheta,r) \cap \SO} \sigma_i(\vthetai')\\ &=
\min_{i \in \s{I}_0^\complement} - \sup_{\vtheta' \in \s{B}(\vtheta,r) \cap \SO}  -\frac{\lambda^2 \yi}{(1+\lambda\vthetai')^2} \\
& =  \lambda^2\min_{i \in \s{I}_0^\complement}\inf_{  |\vthetai' -\vthetai | \leq r } \frac{ \yi}{(1+\lambda\vthetai')^2}\\
& =  \lambda^2 \min_{i \in \s{I}_0^\complement} \frac{ \yi}{(1+\lambda(\vthetai+r))^2}
\end{align}
where we used the facts that $y_i >0$ and $1+\lambda\vthetai \geq 0$ (using the definition of $\dom(D_\lambda)$ since $\vtheta \in \dom(D_\lambda)$) to conclude that the infimum is attained for $\vthetai + r$ rather than  $\vthetai - r$.
\end{proof}

\subsubsection{Dual Update}

The generalized residual w.r.t. a primal estimate $\x$ is given by 
$$\vt{\rho}(\x) := -  \nabla F(\A\x) = \frac{\y}{\A\x + \epsilon} - \vt{1}$$ 
and the dual feasible point $\vtheta \in \Delta_\A$ obtained via scaling in equation \eqref{eq:dual_scaling_proposed} 
is given by:
\begin{align}
\vtheta  =  \dscale(\vt{\rho(\x)}/\lambda) =  
\dscale\left( \frac{1}{\lambda} \left(\frac{\y}{\A\x + \epsilon} - \vt{1}\right) \right)
\end{align}


However, this is not sufficient in order to apply Algorithms \ref{alg:solver_screening_bis} and~\ref{alg:solver_screening_local}. Indeed, one needs to compute a dual point $\vtheta \in \Delta_\A \cap \SO$.

\begin{proposition} \label{prop:KL_dual_update_S0}
Let $\SO$ be the set defined in Proposition~\ref{prop:KL_S0} and $\A \in \R_+^{m \times n}$.
Let $\x \in \R_+^n$  
be a primal feasible point and $\vt{\rho}(\x) \in \R^m$ the corresponding residual.
Then $\vt{\Theta}(\x) \in \R^m$ defined as follows
\begin{equation}\label{eq:KL_dual_update_SO}
    [\vt{\Theta}(\x)]_i 
    = \left\lbrace 
    \begin{array}{ll}
       \left[\dscale(\vt{\rho}(\x)/\lambda)\right]_i  & \text{if}\ \ i \in \s{I}_0^\complement \\
       -\frac{1}{\lambda}  & \text{if}\ \ i \in \s{I}_0
    \end{array}\right.
\end{equation}
is such that $\vt{\Theta}(\x) \in \Delta_\A \cap \SO$. Moreover, we have $\vt{\Theta}(\x) \rightarrow \vtheta^\star$ as $\x \rightarrow \x^\star$.
\end{proposition}
\begin{proof}
By construction, we have $\vt{\Theta}(\x) \in \SO$. It remains to show that $\vt{\Theta}(\x)  \in \Delta_\A = \{\vtheta \in \R^m ~|~ \A^\T \vtheta \leq \vt{1}, \lambda \vtheta \geq -\vt{1}\}$, where the second inequality corresponds to the domain of the dual function.
Because $\A \in \R_+^{m \times n}$, $\x \in \R_+^n$, and $\y \in \R_+^m$, we have that $\vt{\rho}(\x)/\lambda \geq -\vt{1}/\lambda$. 
From the definition of the scaling $\dscale$ in~\eqref{eq:dual_scaling_L1}, 
--the scaling factor always being on the interval $(0,1]$--
we obtain that $ \dscale(\vt{\rho}(\x)/\lambda) \geq \vt{\Theta}(\x)  \geq -\vt{1}/\lambda$ . 
Multiplying the left and right hand side of the first inequality by $\A^\T$, we obtain $\A^\T \vt{\Theta}(\x) \leq \A^\T\dscale(\vt{\rho}(\x)/\lambda) \leq \vt{1}$ (by the definition of $\dscale$ in~\eqref{eq:dual_scaling_L1} for the second inequality) which shows that $\vt{\Theta}(\x) \in  \Delta_\A$. \\
Moreover, from the optimality condition \eqref{eq:KL_optimality_condition1} we have $[\vt{\Theta}(\x)]_{\s{I}_0} = \vtheta^\star_{\s{I}_0}$ regardless of $\x$ and for the remaining coordinates
we obtain from \Cref{lem:dual_scaling} (since $F \in C^1$) that $\dscale(-\vt{\rho(\x)}/\lambda) \rightarrow \vtheta^\star$ as $\x \rightarrow \x^\star$
%
which concludes the proof that $\vt{\Theta}(\x) \rightarrow \vtheta^\star$ as $\x \rightarrow \x^\star$.
\end{proof}

\subsubsection{Improved Screening Test}\label{apdx:KL_screening_test}

An improved screening test can be defined on a Gap Safe sphere when intersected with $\SO$.

\begin{proposition} \label{prop:KL_screening_test}
Let $\s{I}_0 = \{i\in [m] :  \yi=0\}$.
Let $\vtheta \in \SO$ 
and $r>0$ be such that  $\vtheta^\star \in \s{B}(\vtheta,r)$. Then, 
\begin{align}
\cA_j^\T \vtheta + r \|[\cA_j]_{\s{I}_0^{\complement}}\|_2 < 1 \quad \implies \quad x_j^\star = 0.
\end{align}
\end{proposition}
\begin{proof}
First of all, one can see from~\eqref{eq:KL_optimality_condition1} that the dual solution takes the value $\theta^\star_i=-1/\lambda$ for all coordinates $i \in \s{I}_0$. Hence, by definition of $\SO$, 
we have $\vtheta^\star \in \SO$ and thus $\vtheta^\star \in \s{B}(\vtheta,r) \cap  \SO$. Therefore, from Proposition~\ref{propo:Safe_screen}, we have that
\begin{equation*}
    \sup_{\vt{\xi} \in \s{B}(\vtheta,r) \cap  \SO} \cA_j^\T \vt{\xi} < 1 \; \implies \; \cA_j^\T \vtheta^\star <1 \; \implies \;  x_j^\star = 0.
\end{equation*}
Finally, we have
\begin{align*}
 \sup_{\vt{\xi} \in \s{B}(\vtheta,r) \cap  \SO} \cA_j^\T \vt{\xi}
&= \cA_j^\T \vtheta + r \sup_{ \vt{u} \in \s{B}(\vt{0},1),  \vt{u}_{\s{I}_0} = \vt{0} } \cA_j^\T \vt{u} \nonumber \\
&=\cA_j^\T \vtheta + r \|[\cA_j]_{\s{I}_0^{\complement}}\|_2
\end{align*}
which completes the proof.
\end{proof}

\begin{remark}
%
When using $\vtheta = \vt{\Theta}(\x)$ in eq. \eqref{eq:KL_dual_update_SO},
the quantity $\cA_j^\T\vtheta$ required for the screening test can be obtained with mild computational effort from $\cA_j^\T \vt{\rho}(\x)$ (which is usually calculated in the solver's update step) even if $\vt{\Theta}(\x)$ is no longer a simple scaled version of $\vt{\rho}(\x)$. Indeed, we have:
\begin{align*}
\cA_j^\T \vt{\Theta}(\x) &=  \sum_{i' \in \s{I}_0^\complement}  a_{i'j} [\dscale(\vt{\rho}(\x)/\lambda)]_{i'} + \sum_{i\in \s{I}_0} a_{ij} \left(-\frac{1}{\lambda}\right) \\
&=  \cA_j^\T \dscale(\vt{\rho}(\x)/\lambda)  - \sum_{i\in \s{I}_0} a_{ij} \left( \frac{1}{\lambda} +   [\dscale(\vt{\rho}(\x)/\lambda)]_{i} \right) \\
&= s\; \cA_j^\T \vt{\rho}(\x)/\lambda - \frac{1}{\lambda}\left( 1 - s \right) \|[\cA_j]_{\s{I}_0}\|_1 
\end{align*}
where 
we denoted $s$ the scaling factor in $\dscale$, such that $\dscale(\vt{z}) = s \vt{z}$, and in the last equality we used the fact that $[\dscale(\vt{\rho}(\x)/\lambda)]_{\s{I}_0} = -s\frac{\vt{1}}{\lambda}$.  
The final expression can be computed efficiently, since the norms $\|[\cA_j]_{\s{I}_0}\|_1$ can be precomputed and all remaining operations are scalar sums and multiplications.
\end{remark}

\bibliography{./bib/PhD,./bib/Mestrado}

\begin{thebibliography}{53}
\providecommand{\natexlab}[1]{#1}
\providecommand{\url}[1]{\texttt{#1}}
\expandafter\ifx\csname urlstyle\endcsname\relax
  \providecommand{\doi}[1]{doi: #1}\else
  \providecommand{\doi}{doi: \begingroup \urlstyle{rm}\Url}\fi

\bibitem[Bach et~al.(2012)Bach, Jenatton, Mairal, and Obozinski]{Bach2012}
Francis Bach, Rodolphe Jenatton, Julien Mairal, and Guillaume Obozinski.
\newblock Optimization with sparsity-inducing penalties.
\newblock \emph{Foundations and Trends® in Machine Learning}, 4\penalty0
  (1):\penalty0 1--106, 2012.
\newblock ISSN 1935-8237.
\newblock \doi{10.1561/2200000015}.
\newblock URL \url{http://dx.doi.org/10.1561/2200000015}.

\bibitem[Basu et~al.(1998)Basu, Harris, Hjort, and Jones]{Basu1998}
Ayanendranath Basu, Ian~R. Harris, Nils~L. Hjort, and M.~C. Jones.
\newblock Robust and efficient estimation by minimising a density power
  divergence.
\newblock \emph{Biometrika}, 85\penalty0 (3):\penalty0 549--559, 1998.
\newblock ISSN 00063444.
\newblock URL \url{http://www.jstor.org/stable/2337385}.

\bibitem[Bauschke and Combettes(2011)]{Bauschke2011}
Heinz~H. Bauschke and Patrick~L. Combettes.
\newblock \emph{Convex Analysis and Monotone Operator Theory in Hilbert
  Spaces}.
\newblock Springer, 1st edition, 2011.
\newblock ISBN 1441994661.

\bibitem[Beck and Teboulle(2009)]{Beck2009}
Amir Beck and Marc Teboulle.
\newblock A fast iterative shrinkage-thresholding algorithm for linear inverse
  problems.
\newblock \emph{SIAM Journal on Imaging Sciences}, 2\penalty0 (1):\penalty0
  183--202, Jan 2009.
\newblock \doi{10.1137/080716542}.
\newblock URL \url{http://dx.doi.org/10.1137/080716542}.

\bibitem[Bonnefoy et~al.(2015)Bonnefoy, Emiya, Ralaivola, and
  Gribonval]{Bonnefoy2015}
Antoine Bonnefoy, Valentin Emiya, Liva Ralaivola, and R{\'e}mi Gribonval.
\newblock Dynamic screening: Accelerating first-order algorithms for the lasso
  and group-lasso.
\newblock \emph{IEEE Transactions on Signal Processing}, 63\penalty0
  (19):\penalty0 5121--5132, Oct 2015.
\newblock ISSN 1053-587X.
\newblock \doi{10.1109/TSP.2015.2447503}.

\bibitem[Borwein and Lewis(2000)]{Borwein2000}
J.M. Borwein and A.S. Lewis.
\newblock \emph{Convex Analysis and Nonlinear Optimization: Theory and
  Examples}.
\newblock CMS books in mathematics. Springer, 2000.
\newblock ISBN 9780387989402.
\newblock URL \url{https://www.springer.com/gp/book/9780387295701}.

\bibitem[B{\"u}hlmann and van~de Geer(2011)]{Buhlmann2011}
Peter B{\"u}hlmann and Sara van~de Geer.
\newblock \emph{Statistics for high-dimensional data}.
\newblock Springer Series in Statistics. Springer, Heidelberg, 2011.
\newblock ISBN 978-3-642-20191-2.
\newblock \doi{10.1007/978-3-642-20192-9}.
\newblock URL \url{http://dx.doi.org/10.1007/978-3-642-20192-9}.
\newblock Methods, theory and applications.

\bibitem[Chambolle and Pock(2011)]{Chambolle2011}
Antonin Chambolle and Thomas Pock.
\newblock A first-order primal-dual algorithm for convex problems with
  applications to imaging.
\newblock \emph{Journal of Mathematical Imaging and Vision}, 40\penalty0
  (1):\penalty0 120--145, May 2011.
\newblock ISSN 1573-7683.
\newblock \doi{10.1007/s10851-010-0251-1}.
\newblock URL \url{https://doi.org/10.1007/s10851-010-0251-1}.

\bibitem[Cutler and Breiman(1994)]{Cutler1994}
Adele Cutler and Leo Breiman.
\newblock Archetypal analysis.
\newblock \emph{Technometrics}, 36\penalty0 (4):\penalty0 338--347, 1994.
\newblock ISSN 00401706.
\newblock URL \url{http://www.jstor.org/stable/1269949}.

\bibitem[{Dantas} and {Gribonval}(2019)]{Dantas2019a}
C.~F. {Dantas} and R.~{Gribonval}.
\newblock Stable safe screening and structured dictionaries for faster $\ell
  _{1}$ regularization.
\newblock \emph{IEEE Transactions on Signal Processing}, 67\penalty0
  (14):\penalty0 3756--3769, July 2019.
\newblock ISSN 1053-587X.
\newblock \doi{10.1109/TSP.2019.2919404}.

\bibitem[Dantas et~al.(2021)Dantas, Soubies, and F{\'e}votte]{Dantas2021}
Cassio~F. Dantas, Emmanuel Soubies, and C{\'e}dric F{\'e}votte.
\newblock Safe screening for sparse regression with the kullback-leibler
  divergence.
\newblock In \emph{IEEE International Conference on Acoustics, Speech and
  Signal Processing (ICASSP)}, Toronto, ON, Canada, June 2021.

\bibitem[El~Ghaoui et~al.(2012)El~Ghaoui, Viallon, and Rabbani]{ElGhaoui2012}
Laurent El~Ghaoui, Vivian Viallon, and Tarek Rabbani.
\newblock Safe feature elimination for the lasso and sparse supervised learning
  problems.
\newblock \emph{Pacific Journal of Optimization}, 8\penalty0 (4):\penalty0
  667--698, Oct 2012.
\newblock Special Issue on Conic Optimization.

\bibitem[Fan and Lv(2008)]{Fan2008}
Jianqing Fan and Jinchi Lv.
\newblock Sure independence screening for ultrahigh dimensional feature space.
\newblock \emph{Journal of the Royal Statistical Society: Series B (Statistical
  Methodology)}, 70\penalty0 (5):\penalty0 849--911, Nov 2008.
\newblock \doi{10.1111/j.1467-9868.2008.00674.x}.
\newblock URL \url{https://doi.org/10.1111/j.1467-9868.2008.00674.x}.

\bibitem[Fercoq et~al.(2015)Fercoq, Gramfort, and Salmon]{Fercoq2015}
O.~Fercoq, A.~Gramfort, and J.~Salmon.
\newblock Mind the duality gap: safer rules for the lasso.
\newblock In \emph{International Conference on Machine Learning}, volume~37,
  pages 333--342, July 2015.

\bibitem[F\'{e}votte and Idier(2011)]{Fevotte2011}
C\'{e}dric F\'{e}votte and J\'{e}r\^{o}me Idier.
\newblock Algorithms for nonnegative matrix factorization with the
  $\beta$-divergence.
\newblock \emph{Neural Computation}, 23\penalty0 (9):\penalty0 2421–2456, Sep
  2011.
\newblock ISSN 0899-7667.
\newblock \doi{10.1162/NECO_a_00168}.
\newblock URL \url{https://doi.org/10.1162/NECO_a_00168}.

\bibitem[F{\'e}votte et~al.(2018)F{\'e}votte, Vincent, and Ozerov]{Fevotte2018}
C{\'e}dric F{\'e}votte, Emmanuel Vincent, and Alexey Ozerov.
\newblock {Single-channel audio source separation with NMF: divergences,
  constraints and algorithms}.
\newblock In \emph{{Audio Source Separation}}. {Springer}, March 2018.
\newblock URL \url{https://hal.inria.fr/hal-01631185}.

\bibitem[Friedman et~al.(2010)Friedman, Hastie, and Tibshirani]{Friedman2010}
Jerome~H. Friedman, Trevor Hastie, and Rob Tibshirani.
\newblock Regularization paths for generalized linear models via coordinate
  descent.
\newblock \emph{Journal of Statistical Software, Articles}, 33\penalty0
  (1):\penalty0 1--22, 2010.
\newblock ISSN 1548-7660.
\newblock \doi{10.18637/jss.v033.i01}.
\newblock URL \url{https://www.jstatsoft.org/v033/i01}.

\bibitem[{Févotte} and {Dobigeon}(2015)]{Fevotte2015}
C.~{Févotte} and N.~{Dobigeon}.
\newblock Nonlinear hyperspectral unmixing with robust nonnegative matrix
  factorization.
\newblock \emph{IEEE Transactions on Image Processing}, 24\penalty0
  (12):\penalty0 4810--4819, 2015.
\newblock \doi{10.1109/TIP.2015.2468177}.

\bibitem[Globerson et~al.(2007)Globerson, Chechik, Pereira, and
  Tishby]{NIPSpapers1-17}
A.~Globerson, G.~Chechik, F.~Pereira, and N.~Tishby.
\newblock {Euclidean Embedding of Co-occurrence Data}.
\newblock \emph{The Journal of Machine Learning Research}, 8:\penalty0
  2265--2295, 2007.

\bibitem[Golub et~al.(1999)Golub, Slonim, Tamayo, Huard, Gaasenbeek, Mesirov,
  Coller, Loh, Downing, Caligiuri, and Bloomfield]{Leukemia}
T.~R. Golub, D.~K. Slonim, P.~Tamayo, C.~Huard, M.~Gaasenbeek, J.~P. Mesirov,
  H.~Coller, M.~L. Loh, J.~R. Downing, M.~A. Caligiuri, and C.~D. Bloomfield.
\newblock Molecular classification of cancer: class discovery and class
  prediction by gene expression monitoring.
\newblock \emph{Science}, 286:\penalty0 531--537, 1999.

\bibitem[{Harmany} et~al.(2012){Harmany}, {Marcia}, and {Willett}]{Harmany2012}
Z.~T. {Harmany}, R.~F. {Marcia}, and R.~M. {Willett}.
\newblock {This is SPIRAL-TAP: Sparse Poisson Intensity Reconstruction
  ALgorithms—Theory and Practice}.
\newblock \emph{IEEE Transactions on Image Processing}, 21\penalty0
  (3):\penalty0 1084--1096, March 2012.
\newblock ISSN 1941-0042.
\newblock \doi{10.1109/TIP.2011.2168410}.

\bibitem[Hiriart-Urruty and
  Lemar{\'{e}}chal(1993{\natexlab{a}})]{Hiriart-Urruty1993}
Jean-Baptiste Hiriart-Urruty and Claude Lemar{\'{e}}chal.
\newblock \emph{Convex Analysis and Minimization Algorithms {I}}.
\newblock Springer Berlin Heidelberg, 1993{\natexlab{a}}.
\newblock \doi{10.1007/978-3-662-02796-7}.
\newblock URL \url{https://doi.org/10.1007/978-3-662-02796-7}.

\bibitem[Hiriart-Urruty and
  Lemar{\'{e}}chal(1993{\natexlab{b}})]{Hiriart-Urruty1993a}
Jean-Baptiste Hiriart-Urruty and Claude Lemar{\'{e}}chal.
\newblock \emph{Convex Analysis and Minimization Algorithms {II}}.
\newblock Springer Berlin Heidelberg, 1993{\natexlab{b}}.
\newblock \doi{10.1007/978-3-662-06409-2}.
\newblock URL \url{https://doi.org/10.1007/978-3-662-06409-2}.

\bibitem[Hsieh and Dhillon(2011)]{Hsieh2011}
Cho-Jui Hsieh and Inderjit~S. Dhillon.
\newblock Fast coordinate descent methods with variable selection for
  non-negative matrix factorization.
\newblock In \emph{Proc. ACM SIGKDD International Conference on Knowledge
  Discovery and Data Mining (KDD)}, page 1064–1072, 2011.
\newblock ISBN 9781450308137.
\newblock \doi{10.1145/2020408.2020577}.
\newblock URL \url{https://doi.org/10.1145/2020408.2020577}.

\bibitem[{Jia} and {Qian}(2007)]{Jia2007}
S.~{Jia} and Y.~{Qian}.
\newblock Spectral and spatial complexity-based hyperspectral unmixing.
\newblock \emph{IEEE Transactions on Geoscience and Remote Sensing},
  45\penalty0 (12):\penalty0 3867--3879, 2007.
\newblock \doi{10.1109/TGRS.2007.898443}.

\bibitem[Johnson and Guestrin(2015)]{Johnson2015}
Tyler Johnson and Carlos Guestrin.
\newblock Blitz: A principled meta-algorithm for scaling sparse optimization.
\newblock In \emph{International Conference on Machine Learning}, pages
  1171--1179, 2015.

\bibitem[Lee and Seung(2001)]{Lee2001}
Daniel~D. Lee and H.~Sebastian Seung.
\newblock Algorithms for non-negative matrix factorization.
\newblock In T.~K. Leen, T.~G. Dietterich, and V.~Tresp, editors,
  \emph{Advances in Neural Information Processing Systems (NeurIPS)}, pages
  556--562. MIT Press, 2001.
\newblock URL
  \url{http://papers.nips.cc/paper/1861-algorithms-for-non-negative-matrix-factorization.pdf}.

\bibitem[Liu et~al.(2014)Liu, Zhao, Wang, and Ye]{Liu2014}
Jun Liu, Zheng Zhao, Jie Wang, and Jieping Ye.
\newblock Safe screening with variational inequalities and its application to
  lasso.
\newblock In \emph{Proceedings of the 31st International Conference on Machine
  Learning}, volume 32(2) of \emph{ICML'14}, pages 289--297. JMLR.org, June
  2014.
\newblock URL \url{http://dl.acm.org/citation.cfm?id=3044805.3044925}.

\bibitem[Malti and Herzet(2016)]{Malti2016}
A.~Malti and C.~Herzet.
\newblock Safe screening tests for {LASSO} based on firmly non-expansiveness.
\newblock In \emph{2016 {IEEE} International Conference on Acoustics, Speech
  and Signal Processing ({ICASSP})}. {IEEE}, Mar 2016.
\newblock \doi{10.1109/icassp.2016.7472575}.
\newblock URL \url{https://doi.org/10.1109/icassp.2016.7472575}.

\bibitem[Massias et~al.(2017)Massias, Gramfort, and Salmon]{Massias2017}
M.~Massias, A.~Gramfort, and J.~Salmon.
\newblock {From safe screening rules to working sets for faster Lasso-type
  solvers}.
\newblock In \emph{NIPS Workshop on Optimization for Machine Learning}, Long
  Beach, USA, Dec 2017.

\bibitem[Massias et~al.(2020)Massias, Vaiter, Gramfort, and
  Salmon]{Massias2020}
Mathurin Massias, Samuel Vaiter, Alexandre Gramfort, and Joseph Salmon.
\newblock Dual extrapolation for sparse glms.
\newblock \emph{Journal of Machine Learning Research}, 21\penalty0
  (234):\penalty0 1--33, 2020.
\newblock URL \url{http://jmlr.org/papers/v21/19-587.html}.

\bibitem[Ndiaye(2018)]{Ndiaye2018}
Eugene Ndiaye.
\newblock \emph{{Safe optimization algorithms for variable selection and
  hyperparameter tuning}}.
\newblock Thesis, {Universit{\'e} Paris-Saclay}, Oct 2018.
\newblock URL \url{https://pastel.archives-ouvertes.fr/tel-01962450}.

\bibitem[Ndiaye et~al.(2016)Ndiaye, Fercoq, Gramfort, and Salmon]{Ndiaye2016}
Eugene Ndiaye, Olivier Fercoq, Alexandre Gramfort, and Joseph Salmon.
\newblock Gap safe screening rules for sparse-group lasso.
\newblock In D.~Lee, M.~Sugiyama, U.~Luxburg, I.~Guyon, and R.~Garnett,
  editors, \emph{Advances in Neural Information Processing Systems}, volume~29,
  pages 388--396. Curran Associates, Inc., 2016.
\newblock URL
  \url{https://proceedings.neurips.cc/paper/2016/file/555d6702c950ecb729a966504af0a635-Paper.pdf}.

\bibitem[Ndiaye et~al.(2017)Ndiaye, Fercoq, Gramfort, and Salmon]{Ndiaye2017}
Eugene Ndiaye, Olivier Fercoq, Alexandre Gramfort, and Joseph Salmon.
\newblock Gap safe screening rules for sparsity enforcing penalties.
\newblock \emph{Journal of Machine Learning Research}, 18\penalty0
  (128):\penalty0 1--33, Nov 2017.

\bibitem[Ogawa et~al.(2013)Ogawa, Suzuki, and Takeuchi]{Ogawa2013}
Kohei Ogawa, Yoshiki Suzuki, and Ichiro Takeuchi.
\newblock Safe screening of non-support vectors in pathwise svm computation.
\newblock In Sanjoy Dasgupta and David McAllester, editors, \emph{Proceedings
  of the 30th International Conference on Machine Learning}, volume~28 of
  \emph{Proceedings of Machine Learning Research}, pages 1382--1390, Atlanta,
  Georgia, USA, 17--19 Jun 2013. PMLR.
\newblock URL \url{http://proceedings.mlr.press/v28/ogawa13b.html}.

\bibitem[{Ren} et~al.(2018){Ren}, {Huang}, {Ye}, and {Qian}]{Ren2018}
S.~{Ren}, S.~{Huang}, J.~{Ye}, and X.~{Qian}.
\newblock Safe feature screening for generalized lasso.
\newblock \emph{IEEE Transactions on Pattern Analysis and Machine
  Intelligence}, 40\penalty0 (12):\penalty0 2992--3006, 2018.
\newblock \doi{10.1109/TPAMI.2017.2776267}.

\bibitem[Rockafellar(1970)]{Rockafellar1970}
R.~Tyrrell Rockafellar.
\newblock \emph{Convex analysis}.
\newblock Princeton Mathematical Series. Princeton University Press, Princeton,
  N. J., 1970.

\bibitem[Tibshirani et~al.(2011)Tibshirani, Bien, Friedman, Hastie, Simon,
  Taylor, and Tibshirani]{Tibshirani2011}
Robert Tibshirani, Jacob Bien, Jerome Friedman, Trevor Hastie, Noah Simon,
  Jonathan Taylor, and Ryan~J. Tibshirani.
\newblock Strong rules for discarding predictors in lasso-type problems.
\newblock \emph{Journal of the Royal Statistical Society: Series B (Statistical
  Methodology)}, 74\penalty0 (2):\penalty0 245--266, Nov 2011.
\newblock \doi{10.1111/j.1467-9868.2011.01004.x}.
\newblock URL \url{https://doi.org/10.1111/j.1467-9868.2011.01004.x}.

\bibitem[Tseng and Yun(2009)]{Tseng2009a}
Paul Tseng and Sangwoon Yun.
\newblock A coordinate gradient descent method for nonsmooth separable
  minimization.
\newblock \emph{Mathematical Programming}, 117\penalty0 (1):\penalty0 387--423,
  2009.
\newblock \doi{10.1007/s10107-007-0170-0}.

\bibitem[{Virtanen} et~al.(2013){Virtanen}, {Gemmeke}, and {Raj}]{Virtanen2013}
T.~{Virtanen}, J.~F. {Gemmeke}, and B.~{Raj}.
\newblock Active-set newton algorithm for overcomplete non-negative
  representations of audio.
\newblock \emph{IEEE Transactions on Audio, Speech, and Language Processing},
  21\penalty0 (11):\penalty0 2277--2289, 2013.
\newblock \doi{10.1109/TASL.2013.2263144}.

\bibitem[{Wang} et~al.(2015){Wang}, {Fan}, and {Ye}]{Wang-Ye2015}
J.~{Wang}, W.~{Fan}, and J.~{Ye}.
\newblock Fused lasso screening rules via the monotonicity of subdifferentials.
\newblock \emph{IEEE Transactions on Pattern Analysis and Machine
  Intelligence}, 37\penalty0 (9):\penalty0 1806--1820, 2015.
\newblock \doi{10.1109/TPAMI.2014.2388203}.

\bibitem[Wang et~al.(2014)Wang, Wonka, and Ye]{Wang-Wonka2014}
Jie Wang, Peter Wonka, and Jieping Ye.
\newblock Scaling svm and least absolute deviations via exact data reduction.
\newblock In \emph{International Conference on Machine Learning (ICML)},
  ICML'14, page II–523–II–531. JMLR.org, 2014.

\bibitem[Wang et~al.(2015{\natexlab{a}})Wang, Wonka, and Ye]{Wang-Wonka2015}
Jie Wang, Peter Wonka, and Jieping Ye.
\newblock Lasso screening rules via dual polytope projection.
\newblock \emph{Journal of Machine Learning Research}, 16\penalty0
  (1):\penalty0 1063--1101, May 2015{\natexlab{a}}.
\newblock ISSN 1532-4435.
\newblock URL \url{http://jmlr.org/papers/v16/wang15a.html}.

\bibitem[Wang et~al.(2019)Wang, Zhang, and Ye]{Wang2019}
Jie Wang, Zhanqiu Zhang, and Jieping Ye.
\newblock Two-layer feature reduction for sparse-group lasso via decomposition
  of convex sets.
\newblock \emph{Journal of Machine Learning Research}, 20\penalty0
  (163):\penalty0 1--42, 2019.
\newblock URL \url{http://jmlr.org/papers/v20/16-383.html}.

\bibitem[Wang et~al.(2015{\natexlab{b}})Wang, Zhang, and Wu]{Wang2015}
Suyu Wang, Zongxiang Zhang, and Ying Wu.
\newblock Spatial and spectral coordinate super resolution of hyperspectral
  imagery based on redundant dictionary.
\newblock In \emph{Seventh International Conference on Graphic and Image
  Processing}, pages 98170E--98170E. International Society for Optics and
  Photonics, 2015{\natexlab{b}}.

\bibitem[Xiang and Ramadge(2012)]{Xiang2012}
Z.~J. Xiang and P.~J. Ramadge.
\newblock Fast lasso screening tests based on correlations.
\newblock In \emph{2012 IEEE International Conference on Acoustics, Speech and
  Signal Processing (ICASSP)}, pages 2137--2140, March 2012.
\newblock \doi{10.1109/ICASSP.2012.6288334}.

\bibitem[{Xiang} et~al.(2017){Xiang}, {Wang}, and {Ramadge}]{Xiang2016}
Z.~J. {Xiang}, Y.~{Wang}, and P.~J. {Ramadge}.
\newblock Screening tests for lasso problems.
\newblock \emph{IEEE Transactions on Pattern Analysis and Machine
  Intelligence}, 39\penalty0 (5):\penalty0 1008--1027, May 2017.
\newblock ISSN 0162-8828.
\newblock \doi{10.1109/TPAMI.2016.2568185}.

\bibitem[Xiang et~al.(2011)Xiang, Xu, and Ramadge]{Xiang2011}
Zhen~James Xiang, Hao Xu, and Peter~J Ramadge.
\newblock Learning sparse representations of high dimensional data on large
  scale dictionaries.
\newblock In \emph{Advances in Neural Information Processing Systems (NIPS)},
  volume~24, pages 900--908, Granada, Spain, Dec 2011.

\bibitem[{Yanez} and {Bach}(2017)]{Yanez2017}
F.~{Yanez} and F.~{Bach}.
\newblock Primal-dual algorithms for non-negative matrix factorization with the
  {Kullback-Leibler} divergence.
\newblock In \emph{2017 IEEE International Conference on Acoustics, Speech and
  Signal Processing (ICASSP)}, pages 2257--2261, 2017.

\bibitem[Yoshida et~al.(2018)Yoshida, Takeuchi, and Karasuyama]{Yoshida2018}
Tomoki Yoshida, Ichiro Takeuchi, and Masayuki Karasuyama.
\newblock Safe triplet screening for distance metric learning.
\newblock In \emph{Proceedings of the 24th ACM SIGKDD International Conference
  on Knowledge Discovery \& Data Mining}, KDD '18, page 2653–2662, New York,
  NY, USA, 2018. Association for Computing Machinery.
\newblock ISBN 9781450355520.
\newblock \doi{10.1145/3219819.3220037}.
\newblock URL \url{https://doi.org/10.1145/3219819.3220037}.

\bibitem[Yuan et~al.(2010)Yuan, Chang, Hsieh, and Lin]{Yuan2010}
Guo-Xun Yuan, Kai-Wei Chang, Cho-Jui Hsieh, and Chih-Jen Lin.
\newblock A comparison of optimization methods and software for large-scale
  l1-regularized linear classification.
\newblock \emph{Journal of Machine Learning Research}, 11\penalty0
  (105):\penalty0 3183--3234, 2010.
\newblock URL \url{http://jmlr.org/papers/v11/yuan10c.html}.

\bibitem[Zhou and Zhao(2015)]{Zhou2015}
Qiang Zhou and Qi~Zhao.
\newblock Safe subspace screening for nuclear norm regularized least squares
  problems.
\newblock In Francis Bach and David Blei, editors, \emph{Proceedings of the
  32nd International Conference on Machine Learning}, volume~37 of
  \emph{Proceedings of Machine Learning Research}, pages 1103--1112, Lille,
  France, 07--09 Jul 2015. PMLR.
\newblock URL \url{http://proceedings.mlr.press/v37/zhoua15.html}.

\bibitem[Zimmert et~al.(2015)Zimmert, de~Witt, Kerg, and Kloft]{Zimmert2015}
Julian Zimmert, Christian~Schroeder de~Witt, Giancarlo Kerg, and Marius Kloft.
\newblock Safe screening for support vector machines.
\newblock In \emph{NIPS 2015 Workshop on Optimization in Machine Learning
  (OPT)}, 2015.

\end{thebibliography}

\end{document}